\documentclass[anon,12pt]{colt2022} 

\title[]{Sharp Asymptotics of Kernel Ridge Regression\\ Beyond the Linear Regime}
\usepackage{times}
\usepackage{comment}
\usepackage{graphicx}
\usepackage{epstopdf}

\coltauthor{\Name{Hong Hu} \Email{honghu@g.harvard.edu}\and
\Name{Yue M. Lu} \Email{yuelu@seas.harvard.edu}\\
\addr John A. Paulson School of Engineering and Applied Sciences, Harvard University}

\newtheorem{thm}{Theorem}
\newtheorem{prop}{Proposition}
\newtheorem{lem}{Lemma}

\newtheorem{rem}{Remark}

\begin{document}
\global\long\def\vct#1{\boldsymbol{#1}}%
\global\long\def\mat#1{\boldsymbol{#1}}%
\global\long\def\opvec{\text{vec}}%
\global\long\def\optr{\mbox{tr}}%
\global\long\def\opmat{\mbox{mat}}%
\global\long\def\opdiag{\mbox{diag}}%

\global\long\def\t#1{\widetilde{#1}}%
\global\long\def\h#1{\widehat{#1}}%
\global\long\def\abs#1{\left\lvert #1\right\rvert }%
\global\long\def\norm#1{\lVert#1\rVert}%

\global\long\def\inprod#1{\langle#1\rangle}%
\global\long\def\set#1{\left\{  #1\right\}  }%
\global\long\def\bydef{\overset{\text{def}}{=}}%

\global\long\def\teq#1{\overset{#1}{=}}%
\global\long\def\tleq#1{\overset{#1}{\leq}}%
\global\long\def\tgeq#1{\overset{#1}{\geq}}%
\global\long\def\tapprox#1{\overset{#1}{\approx}}%
\global\long\def\tle#1{\overset{#1}{<}}%

\global\long\def\EE{\mathbb{E}\,}%
\global\long\def\EEk{\mathbb{E}_{k}\,}%

\global\long\def\R{\mathbb{R}}%
\global\long\def\E{\mathbb{E}}%
\global\long\def\P{\mathbb{P}}%
\global\long\def\Q{\mathbb{Q}}%
\global\long\def\I{\mathbb{I}}%
\global\long\def\C{\mathbb{C}}%

\global\long\def\va{\boldsymbol{a}}%
\global\long\def\vb{\boldsymbol{b}}%
\global\long\def\vc{\boldsymbol{c}}%
\global\long\def\vd{\boldsymbol{d}}%
\global\long\def\ve{\boldsymbol{e}}%
\global\long\def\vf{\boldsymbol{f}}%
\global\long\def\vg{\boldsymbol{g}}%

\global\long\def\vh{\boldsymbol{h}}%
\global\long\def\vi{\boldsymbol{i}}%
\global\long\def\vj{\boldsymbol{j}}%
\global\long\def\vk{\boldsymbol{k}}%
\global\long\def\vl{\boldsymbol{l}}%

\global\long\def\vm{\boldsymbol{m}}%
\global\long\def\vn{\boldsymbol{n}}%
\global\long\def\vo{\boldsymbol{o}}%
\global\long\def\vp{\boldsymbol{p}}%
\global\long\def\vq{\boldsymbol{q}}%
\global\long\def\vr{\boldsymbol{r}}%

\global\long\def\vs{\boldsymbol{s}}%
\global\long\def\vt{\boldsymbol{t}}%
\global\long\def\vu{\boldsymbol{u}}%
\global\long\def\vv{\boldsymbol{v}}%
\global\long\def\vw{\boldsymbol{w}}%
\global\long\def\vx{\boldsymbol{x}}%
\global\long\def\vy{\boldsymbol{y}}%
\global\long\def\vz{\boldsymbol{z}}%

\global\long\def\vtheta{\boldsymbol{\theta}}%
\global\long\def\vxi{\boldsymbol{\xi}}%
\global\long\def\vdelta{\boldsymbol{\delta}}%
\global\long\def\veta{\vct{\eta}}%
\global\long\def\vlambda{\boldsymbol{\lambda}}%
\global\long\def\vbeta{\boldsymbol{\beta}}%
\global\long\def\vmu{\boldsymbol{\mu}}%

\global\long\def\vphi{\boldsymbol{\phi}}%
\global\long\def\vgamma{\boldsymbol{\gamma}}%
\global\long\def\vpsi{\boldsymbol{\psi}}%
\global\long\def\vrho{\boldsymbol{\rho}}%

\global\long\def\mA{\boldsymbol{A}}%
\global\long\def\mB{\boldsymbol{B}}%
\global\long\def\mC{\boldsymbol{C}}%
\global\long\def\mD{\boldsymbol{D}}%
\global\long\def\mE{\boldsymbol{E}}%
\global\long\def\mF{\boldsymbol{F}}%
\global\long\def\mG{\boldsymbol{G}}%

\global\long\def\mH{\boldsymbol{H}}%
\global\long\def\mI{\boldsymbol{I}}%
\global\long\def\mJ{\boldsymbol{J}}%
\global\long\def\mK{\boldsymbol{K}}%
\global\long\def\mL{\boldsymbol{L}}%
\global\long\def\mM{\boldsymbol{M}}%
\global\long\def\mN{\boldsymbol{N}}%

\global\long\def\mO{\boldsymbol{O}}%
\global\long\def\mP{\boldsymbol{P}}%
\global\long\def\mQ{\boldsymbol{Q}}%
\global\long\def\mR{\boldsymbol{R}}%
\global\long\def\mS{\boldsymbol{S}}%
\global\long\def\mT{\boldsymbol{T}}%
\global\long\def\mU{\boldsymbol{U}}%

\global\long\def\mV{\boldsymbol{V}}%
\global\long\def\mW{\boldsymbol{W}}%
\global\long\def\mX{\boldsymbol{X}}%
\global\long\def\mY{\boldsymbol{Y}}%
\global\long\def\mZ{\boldsymbol{Z}}%

\global\long\def\mLa{\boldsymbol{\Lambda}}%
\global\long\def\mOm{\boldsymbol{\Omega}}%
\global\long\def\mSig{\boldsymbol{\Sigma}}%
\global\long\def\mDelta{\boldsymbol{\Delta}}%
\global\long\def\mPsi{\boldsymbol{\Psi}}%
\global\long\def\mGam{\boldsymbol{\Gamma}}%

\global\long\def\calS{\mathcal{S}}%
\global\long\def\calN{\mathcal{N}}%
\global\long\def\calL{\mathcal{L}}%
\global\long\def\calD{\mathcal{D}}%
\global\long\def\calV{\mathcal{V}}%
\global\long\def\calW{\mathcal{W}}%

\global\long\def\a{\alpha}%
\global\long\def\b{\beta}%
\global\long\def\m{\mu}%
\global\long\def\n{\nu}%
\global\long\def\g{\gamma}%
\global\long\def\s{\sigma}%
\global\long\def\e{\epsilon}%
\global\long\def\w{\omega}%
\global\long\def\veps{\varepsilon}%

\global\long\def\T{\intercal}%
\global\long\def\d{\text{d}}%

\global\long\def\nt{\left\lfloor nt\right\rfloor }%
\global\long\def\ns{\left\lfloor ns\right\rfloor }%
\global\long\def\textif{\text{if }}%
\global\long\def\otherwise{\text{otherwise}}%
\global\long\def\st{\text{s.t. }}%

\global\long\def\lmax{\lambda_{\max}}%
\global\long\def\lmin{\lambda_{\min}}%

\global\long\def\tvy{\t{\boldsymbol{y}}}%
\global\long\def\tc{\t c}%
\global\long\def\ttau{\t{\tau}}%
\global\long\def\tf{\t f}%
\global\long\def\th{\t h}%
\global\long\def\tq{\t q}%
\global\long\def\tz{\t z}%

\global\long\def\tva{\t{\boldsymbol{a}}}%
\global\long\def\tvw{\t{\boldsymbol{w}}}%
\global\long\def\tw{\t w}%

\global\long\def\argmin#1{\underset{#1}{\text{argmin }}}%
\global\long\def\argmax#1{\underset{#1}{\text{argmax }}}%
\global\long\def\targmin#1{\text{argmin}_{#1}}%
\global\long\def\targmax#1{\text{argmax}_{#1}}%

\global\long\def\asconv{\overset{a.s.}{\rightarrow}}%
\global\long\def\pconv{\overset{\mathbb{P}}{\rightarrow}}%
\global\long\def\wconv{\overset{\calL}{\rightarrow}}%

\global\long\def\iid{\overset{i.i.d.}{\sim}}%

\global\long\def\dmtx{\mA}%
\global\long\def\dmtxi{A}%
\global\long\def\kmtx{\mK}%
\global\long\def\inputmtx{\mX}%
\global\long\def\shmtx{\mY}%
\global\long\def\rsmtx{\mR}%
\global\long\def\mdmtx{\mD}%

\global\long\def\sgl{\vxi}%
\global\long\def\noise{\boldsymbol{z}}%
\global\long\def\sgli{\xi}%
\global\long\def\est{\vb}%
\global\long\def\esti{b}%
\global\long\def\res{\vr}%
\global\long\def\err{\vv}%
\global\long\def\obser{\boldsymbol{y}}%
\global\long\def\invec{\vx}%
\global\long\def\tevec{\vg}%

\global\long\def\noisei{z}%
\global\long\def\sol{\widehat{\vw}}%
\global\long\def\soli{\hat{w}}%
\global\long\def\Soli{w}%

\global\long\def\kcoeff#1{\mu_{#1}}%
\global\long\def\tkcoeff#1{\tilde{\mu}_{#1}}%
\global\long\def\tcoeff#1{\alpha_{#1}}%
\global\long\def\usdim#1{N_{#1}}%
\global\long\def\spmeasure#1{\tau_{#1}}%

\global\long\def\kernalfn{f}%
\global\long\def\Kernalfn{K}%
\global\long\def\teacherfn{g}%
\global\long\def\usp#1{q_{#1}}%
\global\long\def\tusp#1{\tilde{q}_{#1}}%
\global\long\def\dusp#1#2{q_{#1,#2}}%

\global\long\def\regressfn{h}%
\global\long\def\hermitefn{H}%
\global\long\def\varnoise{\sigma_{\noisei}}%
\global\long\def\equisamp{\theta}%

\global\long\def\unifdist{\text{Unif}}%
\global\long\def\indicatorfn{\mathbb{I}}%
\global\long\def\sgn{\text{sign}}%
\global\long\def\diag{\mbox{diag}}%
\global\long\def\var{\text{Var}}%
\global\long\def\snr{\text{SNR }}%
\global\long\def\tr{\mbox{Tr }}%
\global\long\def\Tr{\mbox{Tr }}%

\global\long\def\Rpos{\R_{>0}}%
\global\long\def\Rnonneg{\R_{\geq0}}%

\global\long\def\msc{K}%
\global\long\def\stodom{\lesssim}%
\global\long\def\symset{\mathcal{S}}%
\global\long\def\dsphere#1{\mathcal{S}^{#1-1}(\sqrt{#1})}%

\maketitle

\begin{abstract}%
The generalization performance of kernel ridge regression (KRR) exhibits a multi-phased pattern that crucially depends on the scaling relationship between the sample size $n$ and the underlying dimension $d$. This phenomenon is due to the fact that KRR sequentially learns
functions of increasing complexity as the sample size increases;  when $d^{k-1}\ll n\ll d^{k}$, only polynomials with degree less than $k$ are learned. In this paper, we present sharp asymptotic characterization of the performance of KRR at the critical transition regions with $n \asymp d^k$, for $k\in\mathbb{Z}^{+}$. Our asymptotic characterization provides a precise picture of the whole learning process and clarifies the impact of various parameters (including the choice of the kernel function) on the generalization performance. In particular, we show that the learning curves of KRR can have a delicate ``double descent'' behavior due to specific bias-variance trade-offs at different polynomial scaling regimes.
\end{abstract}




\section{Introduction}

Consider kernel ridge regression (KRR), where we seek to learn a function $\regressfn:\R^{d}\mapsto\R$ from a reproducing kernel Hilbert space (RKHS) associated with a positive semi-definite kernel $\Kernalfn(\cdot,\cdot)$ by solving the following
optimization problem:
\begin{equation}
\hat{\regressfn}=\argmin{\regressfn}\sum_{i=1}^{n}[y_{i}-h(\invec_{i})]^{2}+\lambda\|\regressfn\|_{\Kernalfn}^{2}.\label{eq:kernel_1}
\end{equation}
Here, $\{\invec_{i},y_{i}\}_{i=1}^{n}$ is a collection of training samples, $\|\cdot\|_{\Kernalfn}$ is the RKHS norm and $\lambda>0$ is the regularization parameter.
The performance
of KRR can be characterized by the \emph{training error}
\[
\mathcal{E}_{\text{train}}=\frac{1}{n}\Big[\sum_{i=1}^{n}\big(y_{i}-\hat{\regressfn}(\invec_{i})\big)^{2}+\lambda\|\hat{\regressfn}\|_{\Kernalfn}^{2}\Big]
\]
and the \emph{test error}
\[
\mathcal{E}_{\text{test}}=\E_{\text{new}}[\E(y_{\text{new}}\mid\invec_{\text{new}})-\hat{\regressfn}(\invec_{\text{new}})]^{2},
\]
where $(\invec_{\text{new}},y_{\text{new}})\sim\mathcal{P}$ denotes an independent
test sample, and $\E_{\text{new}}$ denotes the expectation with respect to $(\invec_{\text{new}},y_{\text{new}})$ while keeping the training samples $\{\invec_{i},y_{i}\}_{i=1}^{n}$ fixed.

Kernel ridge regression is a classical method for supervised learning \citep{scholkopf2002learning}. Due to its connection
to modern overparameterized neural networks \citep{neal1996priors,williams1996computing,daniely2016toward,jacot2018neural,belkin2018understand,du2019gradient},
there has been a strong resurgence of interest in studying the performance of KRR, especially in various high-dimensional settings. See, \emph{e.g.}, \cite{rakhlin2019consistency,liang2020multiple,liang2020just,bordelon2020spectrum,canatar2021spectral,ghorbani2021linearized,mei2021generalization}.

\begin{figure}[t]
\begin{centering}
\includegraphics[width=0.6\textwidth]{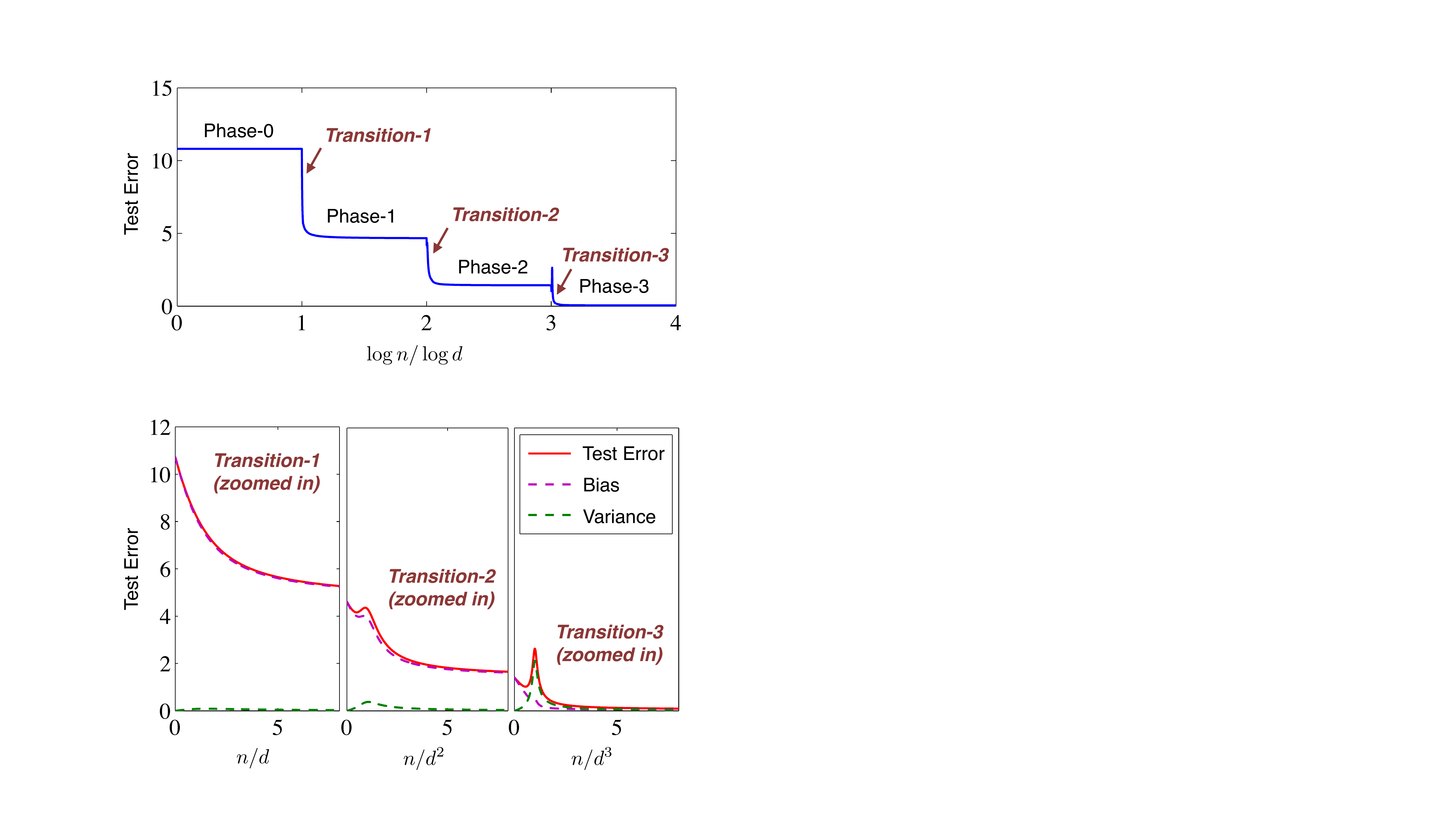}
\par\end{centering}
\caption{\label{fig:Illustration-of-hierarchical}Illustration of the hierarchical
learning process of kernel ridge regression. Top figure: test error as a function of $\log n / \log d$, where $n$ is the sample size and $d$ is the dimension. The test error appears to remain unchanged when $d^{k-1}\ll n\ll d^{k}$, for $k\in\mathbb{Z}^{+}$, while drastic transitions
occur at $n\asymp d^{k}$. Bottom figures: the zoomed-in views of the test error within each transition region, corresponding to $n \asymp d^k$ for $k = 1, 2, 3$. The red curves show the asymptotic predictions obtained in this work. These learning curves exhibit delicate non-monotonic behavior, due to bias-variance trade-offs at different polynomial scaling regimes.}
\end{figure}

One intriguing phenomenon revealed by several recent works \citep{liang2020multiple,ghorbani2021linearized,mei2021generalization} is that the generalization performance of KRR exhibits a hierarchical and multi-phased pattern that crucially depends on the scaling relationship between the sample size $n$ and the underlying dimension $d$. We illustrate this phenomenon in Fig. \ref{fig:Illustration-of-hierarchical} (top part), where we plot the test error $\mathcal{E}_{\text{test}}$ against the ratio $\log n / \log d$. The test error can be clearly partitioned into several consecutive ``stationary'' phases that are separated by more drastic transitions in between. More precisely, $\mathcal{E}_{\text{test}}$
appears to remain unchanged when $d^{k-1}\ll n\ll d^{k}$, for $k\in\mathbb{Z}^{+}$, while transitions
occur at $n\asymp d^{k}$. An explanation of this phenomenon was given in \cite{ghorbani2021linearized}: It is shown that, when $d^{k-1}\ll n\ll d^{k}$, $\mathcal{E}_{\text{test}}$
is approximately equal to the approximation error of the function $h$ by all the polynomials
with degree less than $k$. This means that KRR sequentially learns
functions of increasing complexity as the sample size increases;  when $d^{k-1}\ll n\ll d^{k}$, only polynomials with degree less than $k$ are learned.


What is the performance of KRR near the critical regions, exactly where the transitions happen? This is the focus of the current paper. More precisely, we ``zoom into'' each transition region by assuming $n\asymp d^{k}$, and derive sharp asymptotics of KRR for different values of $k$. Such asymptotic characterization provides a precise picture of the whole learning process and clarifies the impact of various parameters (including the choice of the kernel function) on the generalization performance. As a preview of our results, we plot in the lower part of Fig. \ref{fig:Illustration-of-hierarchical} the theoretical predictions of $\mathcal{E}_{\text{test}}$
in the regimes $n\asymp d^{k}$, for $k=1,2,3$, for a specific choice
of the kernel function. It can be seen that the learning curves of KRR can exhibit delicate non-monotonic behavior due to bias-variance trade-offs: as the sample size $n$ increases, $\mathcal{E}_{\text{test}}$
can first increase and then decrease again after crossing certain
deterministic thresholds. Under the names of  ``double descent'' or ``multiple descent'', such phenomenon
has been observed and analyzed in various other problems and models in learning \citep{mei2019generalization,d2020double,adlam2020neural,nakkiran2021deep}.

Some of the asymptotic predictions given in the paper were first derived in \cite{bordelon2020spectrum,canatar2021spectral} (see also \cite{dietrich1999statistical} for a related earlier work), via non-rigorous statistical physics methods and a ``Gaussian equivalence conjecture'' (see Sec \ref{sec:Equivalence-Conjecture}). One of the technical contributions of this paper is to rigorously establish this conjecture, which allows us to characterize the exact performance KRR in the polynomial scaling regime.

When the current work was under review at COLT '22, we became aware of the recent paper \cite{misiakiewicz2022spectrum} that also studies the exact asymptotics of KRR in the polynomial scaling regime.
The target function considered in that paper is different from ours.
On one hand,  the expansion coefficients (\emph{i.e.}, $\tcoeff{k,d}$ in \eqref{eq:teacher_expansion} below) of low-degree components can be arbitrary, which is more general than ours; on the other hand, they require high-degree coefficients
to be \emph{independent} random variables. In comparison, we consider a target function (as detailed in Sec. \ref{sec:model}) whose coefficients are dependent.
In terms of the distribution of $\{\invec_i\}$, our work focuses on the uniform distribution over $d$-dimensional sphere.
In addition to this case, \cite{misiakiewicz2022spectrum} also considers the uniform distribution over the hypercube $\{-1,1\}^{d}$.



\section{Main Results}

\subsection{Model}
\label{sec:model}
We start by describing the statistical model under which we analyze the performance of KRR. Let $\dsphere d$ denotes the $d$-dimensional sphere with radius $\sqrt{d}$. We assume that the input data vectors  $\invec_{i}\iid\spmeasure{d-1}$, where $\spmeasure{d-1}$
denotes the uniform distribution over $\dsphere d$. The labels $\{y_{i}\}$ are generated from a generalized linear teacher model
\begin{equation}
y_{i}=\teacherfn(\invec_{i}^{\T}\sgl/\sqrt{d})+\noisei_{i},\label{eq:teacher_model}
\end{equation}
where $\teacherfn$ is an unknown teacher function, $\sgl\in\dsphere d$
is the teacher weight vector, and $\noisei_{i}\iid\mathcal{N}(0,\varnoise^{2})$
denotes independent additive noise. We consider
an inner-product kernel on $\dsphere d$, represented as
\begin{equation}
\Kernalfn(\invec,\invec')=\kernalfn_{d}(\invec^{\T}\invec'/\sqrt{d}), \label{eq:inner_product_kernel_def}
\end{equation}
where $\kernalfn_{d}$ is a transformation function that can
depend on $d$. It is known that
any inner-product PSD kernel $\Kernalfn\in L^{2}(\dsphere d\times\dsphere d,\spmeasure{d-1}^{\otimes2})$
can be expanded as \cite[Lemma 4.20]{scholkopf2002learning}:
\begin{equation}
\Kernalfn(\invec,\invec')=\sum_{k=0}^{\infty}\tkcoeff{k,d}\sum_{i=1}^{\usdim k}Y_{ki}(\vx)Y_{ki}(\vx'),\label{eq:kernel_expansion}
\end{equation}
where $\tkcoeff{k,d}\geq0$ is the eigenvalue of the $k$th eigenspace (we will also denote the normalized eigenvalue as $\kcoeff{k,d}:=\usdim k\tkcoeff{k,d}$),
$\{Y_{ki}(\vx)\}_{i=1}^{\usdim k}$ are the associated eigenfunctions,
which is a set of orthonormal degree-$k$ spherical harmonics. We have
\[
\int Y_{ki}(\vx)Y_{\ell j}(\vx)\spmeasure{d-1}(d\vx)=\indicatorfn_{k\ell} \indicatorfn_{ij},
\]
where $\indicatorfn_{ab} = 1$ if $a = b$ and $\indicatorfn_{ab} = 0$ otherwise. Moreover, $\usdim k$ is the corresponding geometric multiplicity of the $k$th
eigenspace, which coincides with the dimension of the subspace spanned
by all degree-$k$ spherical harmonics. (We collect a list of related concepts and properties of spherical
harmonics in Appendix \ref{sec:Spherical-Harmonics-and}.)
Note that both $\usdim k$ and $\{Y_{ki}(\vx)\}_{i=1}^{\usdim k}$
can depend on $d$. However, To lighten the notation, we will omit this dependence when doing so causes no confusion. Finally, we denote by $\spmeasure{d-1,1}$ the distribution of the scalar random variable
$\vx^{\T}\ve_{1}$, where $\vx\sim\spmeasure{d-1}$ and $\ve_1$ is the first standard basis vector.

\subsection{Technical Assumptions}

Our main results are proved under the following assumptions.
\begin{enumerate}
\item[(A.1)]  Both the number of training samples $n$ and the input dimension
$d$ go to infinity, while $\frac{n}{d^{\msc}/\msc!}\to\delta_{\msc}\in(0,\infty)$,
for $\msc\in\mathbb{Z}^{+}$.
\item[(A.2)]  $\kernalfn_{d}(x)=\kernalfn(x/\sqrt{d})$, where 
$\kernalfn(z)$ is well-defined on $[-1,1]$ and satisfies: (1) $\kernalfn(z)\leq C_{f}$, for some $C_{f}>0$ and (2) there exists some $\upsilon\in(0,1]$ such that $\kernalfn(z)\in C^{\infty}[-\upsilon,\upsilon]$.
\item[(A.3)] There exists $C>0$ such that $\kcoeff{k,d}\geq C$, for all $0\leq k\leq\msc$ and large enough $d$.
\item[(A.4)]  There exist constants $C_{g},K_{g}>0$ such that $g(x)\leq C_{g}(1+|x|^{K_{g}})$
for any $x\in\R$.
\end{enumerate}

\begin{rem}
Assumption (A.2) put two constraints on the
kernel functions we consider. The first condition
ensures that the expansion (\ref{eq:kernel_expansion}) is valid and
$
\sum_{k=0}^{\infty}\kcoeff{k,d}<\infty
$
holds uniformly over $d$, which directly follows from the properties
(\ref{eq:fd_under_qk}) and (\ref{eq:diagonal_Y_constant}) given
in Appendix \ref{sec:Spherical-Harmonics-and}. 
The boundedness condition
$
\sum_{k=0}^{\infty}\kcoeff{k,d}<\infty
$
is convenient as it makes all the quantities of interests
such as the training and test errors of size $\mathcal{O}(1)$, when $\lambda\in(0,\infty)$. 
On the other hand, the second condition requires that $\kernalfn(z)$ is smooth in some neighbourhood of $0$. This guarantees that
for any fixed $k$, $\kcoeff{k,d} \to\frac{f^{(k)}(0)}{k!}$, as $d\to\infty$
(see Lemma \ref{fact:kcoeff_convergence} in Appendix \ref{sec:kcoeff_convergece}), where $f^{(k)}(0)$ denotes the $k$th derivative of $f(z)$ at $z = 0$. Below are two concrete examples of kernels that satisfy
Assumption (A.2).

\textbf{Example 1: Polynomial kernel.} In this case, $\kernalfn(z)=(z+b)^{k}$,
where $b\geq0$ is a fixed constant and $k\in\mathbb{Z}^{+}$. It
is a classical type of kernel, widely applied in several machine learning
problems such as support vector machine (SVM).

\textbf{Example 2: Random feature model \citep{rahimi2008random}.}The
random feature model is a computationally efficient random approximation
of kernel function, which is based on the following feature map:
\[
\vx\mapsto\frac{1}{\sqrt{p}}[\sigma(\vr_{1}^{\T}\vx),\sigma(\vr_{2}^{\T}\vx),\cdots,\sigma(\vr_{p}^{\T}\vx)]^{\T},
\]
where $\{\vr_{i}\}_{i=1}^{p}$ is a set of independent random feature
vectors and $\sigma(x)$ is an activation function satisfying $\E\sigma(G)^2<\infty$, with $G\sim\mathcal{N}(0,1)$. The associated random kernel function is as follows:
\[
\Kernalfn_{p}(\vx,\vx')=\frac{1}{p}\sum_{u=1}^{p}\sigma(\vr_{u}^{\T}\vx)\cdot\sigma(\vr_{u}^{\T}\vx').
\]
In particular, when $\vr_{u}\iid\mathcal{N}(\boldsymbol{0},\mI/d)$
and $p\to\infty$, $\Kernalfn_{p}(\vx,\vx')\to\kernalfn(\frac{\invec^{\T}\invec'}{d})$.
Here, $\kernalfn(z)=\E[\sigma(X)\cdot\sigma(Y)]$, with
\[
\begin{pmatrix}X\\
Y
\end{pmatrix}\sim\mathcal{N}(\boldsymbol{0},\mSig),\;\mSig=\begin{bmatrix}1 & z\\
z & 1
\end{bmatrix}.
\]
The smoothness of $f$ near 0 can be directly checked by taking derivatives
of $\kernalfn(z)$. 

In fact, we can also let $\kernalfn(z)$ depend on $d$, as long as: (1) $\kernalfn_{d}\in L^{2}([-\sqrt{d},\sqrt{d}],\spmeasure{d-1,1})$, (2) $\sum_{k=0}^{\infty}\kcoeff{k,d}<\infty$ for large enough $d$, and (3) for any fixed $k$, $\kcoeff{k,d}\to\kcoeff k$ as $d\to\infty$.
\end{rem}
\begin{rem}
Assumption (A.3) guarantees
that the RKHS associated with the kernel $\Kernalfn$ is not degenerate:
the condition that $\mu_k > 0$ for $0 \le k \le K$ is satisfied as long as all the degree-$k$ ($k\leq\msc$) polynomials
are in the RKHS. This is equivalent to requiring $\kernalfn^{(k)}(0)>0$, for all $0\leq k \leq \msc$.
\end{rem}

\begin{rem}
Assumption (A.4) implies that for all $d\in\mathbb{Z}^{+}$, $\teacherfn\in L^{2}([-\sqrt{d},\sqrt{d}],\spmeasure{d-1,1})$.
Therefore, we have the following expansion:
\begin{equation}
\teacherfn(x)=\sum_{k=0}^{\infty}\tcoeff{k,d}\usp k(x),\label{eq:teacher_expansion}
\end{equation}
where $\usp k(x)$ is the degree-$k$ ultraspherical polynomial in
$L^{2}([-\sqrt{d},\sqrt{d}],\spmeasure{d-1})$ (more details about these polynomials can be
found in Appendix \ref{sec:Spherical-Harmonics-and}) and $\{\tcoeff{k,d}\}_{k\geq0}$
satisfies $\sum_{k=0}^{\infty}\tcoeff{k,d}^{2}<\infty$. Note that $\usp{k}(x)$ depends on $d$, but for notational simplicity, we will suppress the dependence.

In fact,
$g(x)\leq C\exp(C|x|)$, for some $C>0$ will suffice to guarantee
the $L^{2}$-integrability of $\teacherfn$. Here, we put a stronger
assumption to ensure that all the moments of $g(x)$ exists, which
helps simplify several parts of the proof.
\end{rem}

\subsection{Main Results and Insights}

We are now ready to state our main theorem.
\begin{thm}
\label{thm:training_test_err}Under Assumption (A.1)-(A.4), we have
the following asymptotic characterization of the training and test errors.

\textbf{\textup{1. Training error:}}
\begin{equation}
\mathcal{E}_{\text{train}}\pconv\lambda\left[\frac{\tcoeff{\msc}^{2}R_{\star}}{1+\kcoeff{\msc}\delta_{\msc}R_{\star}}+\Big(\varnoise^{2}+\sum_{k>\msc}\tcoeff k^{2}\Big)R_{\star}\right],\label{eq:main_results_trainerr}
\end{equation}
where $R_{\star}$ is the unique non-negative solution of:
\begin{equation}
\frac{1}{R}=\tilde{\lambda}+\frac{\mu_{\msc}}{1+\delta_{\msc}\mu_{\msc}R},\label{eq:main_results_Stieltjes}
\end{equation}
with $\kcoeff k=\frac{\kernalfn^{(k)}(0)}{k!}$,
$\tcoeff k=\int_{-\infty}^{\infty}\teacherfn(x)\hermitefn_{k}(x)\tau_G(dx)$ and $\tilde{\lambda}=\lambda+\sum_{k>\msc}^{\infty}\kcoeff k$. 
Here, $\hermitefn_{k}(x)$ is the degree-$k$ Hermite polynomial, $\tau_G(dx)$ is the standard Gaussian measure 
and $R_{\star}=\mathcal{R}(\tilde{\lambda};\mu_{\msc},\delta_{\msc})$,
where
\begin{equation}
\mathcal{R}(\lambda;\mu,\delta)=\frac{-(\lambda+\mu-\mu\delta)+\sqrt{(\lambda+\mu-\delta\mu)^{2}+4\lambda\mu\delta}}{2\lambda\mu\delta}.\label{eq:main_results_Stieltjes_1}
\end{equation}

\textbf{\textup{2. Test error:}}
\begin{equation}
\mathcal{E}_{\text{test}}\pconv\frac{1}{\equisamp-1}\left[\frac{\equisamp\tcoeff{\msc}^{2}}{(1+\kcoeff{\msc}\delta_{\msc}R_{\star})^{2}}+\equisamp\sum_{k>\msc}\tcoeff k^{2}+\varnoise^{2}\right],\label{eq:main_results_testerr}
\end{equation}
where $\equisamp=\frac{(1+\kcoeff{\msc}\delta_{\msc}R_{\star})^{2}}{\delta_{\msc}\kcoeff{\msc}^{2}R_{\star}^{2}}$.
In both (\ref{eq:main_results_trainerr}) and (\ref{eq:main_results_testerr}), $\pconv$ denotes convergence in probability as $d\to\infty$.
\end{thm}

\begin{rem}
It turns out that $R_{\star}$
coincides with the Stieltjes transform of the Marchenko-Pastur (MP)
law, with aspect ratio parameter $1/\delta_{\msc}$ and variance $\kcoeff{\msc}$.
The MP law corresponds to the limiting empirical
eigenvalue distributions of the Wishart ensemble. 
\end{rem}

\textbf{Bias-Variance Decomposition. }The formulas in (\ref{eq:main_results_trainerr})
and (\ref{eq:main_results_testerr}) can be viewed from the perspective of bias-variance
decomposition. Conditioning on the input data vectors $\set{\vx_i}$, we consider the following
average squared bias and variance of an estimator $h$ \cite[Sec.7.3]{hastie2009elements}:
\begin{align*}
\mathcal{B}_{h} & :=\E_{\invec_{\text{new}}}[\E_{\noise}\regressfn(\invec_{\text{new}})-\E(y_{\text{new}}\mid\invec_{\text{new}})]^{2}\\
\mathcal{V}_{h} & :=\E_{\invec_{\text{new}}}[\var_{\noise}(\regressfn(\invec_{\text{new}}))]^{2}.
\end{align*}
Then $\mathcal{E}_{\text{test}}$
in (\ref{eq:main_results_testerr}) can be decomposed as $\mathcal{E}_{\text{test}}=\mathcal{B}_{\hat{h}}+\mathcal{V}_{\hat{h}}$
and it holds that
\begin{align}
\mathcal{B}_{\hat{h}}\pconv & \frac{\equisamp}{\equisamp-1}\Big[\frac{\tcoeff{\msc}^{2}}{(1+\kcoeff{\msc}\delta_{\msc}R_{\star})^{2}}+\sum_{k>\msc}\tcoeff k^{2}\Big]\label{eq:bias_hhat}\\
\mathcal{V}_{\hat{h}}\pconv & \frac{\varnoise^{2}}{\equisamp-1}.\label{eq:variance_hhat}
\end{align}
[These can be directly verified via the same proof strategy underlying our proof of (\ref{eq:main_results_testerr}).]
From (\ref{eq:variance_hhat}), we can regard $\equisamp$ as
an inflation factor of the variance. By its definition, $\equisamp$ only depends
on $\{\kcoeff k\}_{k\geq0}$ but not on $\{\tcoeff k\}_{k\geq0}$
or $\varnoise^{2}$. Therefore, $\equisamp$ reflects the influence
of the kernel function on
the learning performance. On the other hand, from (\ref{eq:bias_hhat})
we can further decompose the bias term $\mathcal{B}_{\hat{h}}$ as $\mathcal{B}_{\hat{h}}=\sum_{k=0}^{\infty}\mathcal{B}_{k}$, where
\[
\mathcal{B}_{k}=\begin{cases}
0 & k<\msc,\\
\frac{\equisamp}{\equisamp-1}\frac{\tcoeff \msc^{2}}{(1+\kcoeff \msc\delta_{\msc}R_{\star})^{2}}+\frac{\sum_{\ell>\msc}\tcoeff \ell^{2}}{\equisamp-1} & k=\msc,\\
\tcoeff k^{2} & k>\msc.
\end{cases}
\]
It can be seen that $\mathcal{B}_{k}$ is equal to the average squared
bias when $\tcoeff k$ is the only non-zero coefficient in (\ref{eq:teacher_expansion}).
In this sense, each $\mathcal{B}_{k}$ can be understood as the contribution
of degree-$k$ components to the total squared bias. Moreover, the
contributions from different components are linear. Similar results
also hold for the training error (\ref{eq:main_results_trainerr}).

\textbf{Hierarchical learning process. }Based on the characterizations
(\ref{eq:main_results_testerr}), we can now have a better understanding
of the hierarchical learning process illustrated in Fig. \ref{fig:Illustration-of-hierarchical}.
Consider the limit: $\delta_{\msc}\to0$ and $\delta_{\msc}\to\infty$.
By (\ref{eq:main_results_Stieltjes_1}), we can get for any fixed $\tilde{\lambda}>0$
and $\mu_{\msc}>0$,
\[
R_{\star}\to\begin{cases}
(\tilde{\lambda}+\mu_{\msc})^{-1}& \delta_{\msc}\to0\\
\tilde{\lambda}^{-1}& \delta_{\msc}\to\infty
\end{cases}
\]
and $\equisamp\to\infty$ under both limits. Therefore, $\lim_{\delta_{\msc}\to0}\mathcal{B}_{\msc}=\tcoeff \msc^{2}$
and $\lim_{\delta_{\msc}\to\infty}\mathcal{B}_{\msc}=0$. Recall that
$\tcoeff k^{2}$ is the energy of the degree-$k$ component of the
teacher function, so this justifies that KRR learns
the degree-$k$ components in phase-$k$. Moreover we can also identify the roles
played by the components of different degrees. From (\ref{eq:main_results_testerr}),
we can see that at phase-$k$: (i) all the low-degree ($\ell<k$) components
of the kernel function and the teacher function do not exert any influence;
(ii) high-degree ($\ell>k$) components of the kernel function act as
an additional regularization term, as manifested in $\tilde{\lambda}$; and
(iii) high-degree components of the teacher function act as
additive noise.%

\textbf{Non-monotonicity of $\mathcal{E}_{\text{test}}$.} In Fig. \ref{fig:Illustration-of-hierarchical}, we show that $\mathcal{E}_{\text{test}}$ can be non-monotonic with respect to the sample complexity. Based on our asymptotic characterization, this phenomenon can be explained as follows. In Fig. 1, we choose $f(z)=\frac{z^3}{30}+\frac{z^2}{2}+z+1$, $g(x)=\frac{x^4}{20}+\frac{x^3}{2}+x^2+x$, $\varnoise=0.5$ and $\lambda=10^{-4}\approx 0$. This corresponds to the ridgeless limit of KRR. Let us take $\lambda\to 0$ in the asymptotic characterization and suppose $\kcoeff k=0$ for all $k>\msc$.
In this case, we can show from \eqref{eq:bias_hhat} and \eqref{eq:variance_hhat} that $\mathcal{B}_{\hat{h}}\pconv\tcoeff{\msc}^{2}\max\{1-\delta_{\msc},0\}$
and $\mathcal{V}_{\hat{h}}\pconv\frac{\varnoise^{2}}{\delta_{\msc}-1}\min\{\delta_{\msc},1\}$. We can find that $\mathcal{B}_{\hat{h}}$ is monotonically decreasing and the non-monotonicity of $\mathcal{E}_{\text{test}}$ stems from the non-monotonicity of $\mathcal{V}_{\hat{h}}$. The peak at the interpolation threshold $\delta_{\msc}=1$ due to the explosion of the variance. In Fig. 1, this corresponds to the peak in the Transition-3.
On the other hand, we can find that the peak at Transition-2 is less notable and there is no peak at Transition-1. The reason is that when $K=1,2$, higher-degree components of kernel function do not vanish and they act as regularization terms (manifested by $\tilde{\lambda}=\lambda+\sum_{k>\msc}^{\infty}\kcoeff k$), which reduces the variance.

\subsection{\label{sec:Equivalence-Conjecture}Gaussian Equivalence Conjecture}

A key technical insight behind our proof of Theorem \ref{thm:training_test_err} is the rigorous establishment of a so-called Gaussian equivalence conjecture. This conjecture was implicitly used (without proof) in several earlier works \citep{dietrich1999statistical,bordelon2020spectrum,canatar2021spectral} that study the generalization performance of kernel methods using non-rigorous statistical physics methods.

For simplicity, let us assume that the kernel and teacher functions are
both band-limited: there exists $L>0$ such that $\kcoeff k=\tcoeff k=0$,
for all $k>L$. First, based on the expansions (\ref{eq:kernel_expansion})
and (\ref{eq:teacher_expansion}), we can obtain an equivalent formulation
of (\ref{eq:kernel_1}) as a linear regression in the feature space
with the following feature map:
\[
\invec_{a}\mapsto\vgamma_{a}=[\underbrace{\widetilde{Y}_{01}(\invec_{a})}_{\usdim 0},\underbrace{\widetilde{Y}_{11}(\invec_{a}),\cdots,\widetilde{Y}_{1\usdim 1}(\invec_{a})}_{\usdim 1},\cdots\cdots\underbrace{\widetilde{Y}_{L1}(\invec_{a}),\cdots\widetilde{Y}_{L\usdim L}(\invec_{a})}_{\usdim L}]^{\T}
\]
where $\widetilde{Y}_{ki}(\cdot):=\sqrt{\tkcoeff{k,d}}Y_{ki}(\cdot)$.
Then (\ref{eq:kernel_1}) is equivalent to:
\begin{equation}
\hat{\vtheta}=\argmin{\vtheta}\sum_{a=1}^{n}[y_{a}-\vtheta^{\T}\vgamma_{a}]^{2}+\lambda\|\vtheta\|^{2}\label{eq:kernel_1_feature}
\end{equation}
where $y_{a}=\vbeta^{\T}\mLa\vgamma_{a}+\noisei_{a}$, with %
\begin{align*}
\vbeta & =[Y_{01}(\sgl),Y_{11}(\sgl),\cdots,Y_{1\usdim 1}(\sgl),\cdots\cdots,Y_{L1}(\sgl),\cdots Y_{L\usdim L}(\sgl)]^{\T}
\end{align*}
and
\[
\mLa=\diag\Big\{\tfrac{\tcoeff{0,d}}{\sqrt{\kcoeff{0,d}}},\tfrac{\tcoeff{1,d}}{\sqrt{\kcoeff{1,d}}},\cdots,\tfrac{\tcoeff{1,d}}{\sqrt{\kcoeff{1,d}}},\cdots\cdots,\tfrac{\tcoeff{L,d}}{\sqrt{\kcoeff{L,d}}},\cdots\tfrac{\tcoeff{L,d}}{\sqrt{\kcoeff{L,d}}}\Big\}.
\]
If the entries of $\vgamma_{a}$ are independent Gaussian random variables, then we reach
the setting that has been analyzed in several recent papers \citep{dicker2016ridge,dobriban2018high,hastie2019surprises,wu2020optimal,richards2021asymptotics}.
The main challenge here is that different entries of $\vgamma_{a}$
are not independent and that there is no linear transformation that can
decouple this dependence. On the other hand, however, different entries $\vgamma_{a}$ are still \emph{uncorrelated}: recall that
$\E\widetilde{Y}_{ki}(\invec)\widetilde{Y}_{\ell j}(\invec)=\tkcoeff{k,d} \indicatorfn_{k\ell}\indicatorfn_{ij}$,
since $\{Y_{ki}(\invec)\}_{k,i}$ are orthonormal with respect to
$\spmeasure{d-1}$. Thus, the so-called Gaussian
equivalence conjecture states that the learning performance of the original
KRR problem will remain asymptotically unchanged if we replace each $\vgamma_{a}$
by a Gaussian vector $\vg_{a}$ with the same mean and covariance
matrix.

\subsection{Limitations of the Current Work}
\label{sec:limitations}

Finally, we point out several important limitations of our results. First, we have assumed that the input vectors $\invec$
are uniformly distributed over $\dsphere d$. Although this is a convenient model for theoretical analysis, the spherical symmetry of the model might impose too strong an assumption on the input data. It will be desirable to explore other more general data distributions such as those considered in \cite{liang2020just,liang2020multiple,mei2021generalization}.
Second, we have assumed that the labels $\set{y_i}$ in the training set are generated by a specific teacher-student model, which is essentially a generalized linear model. On the contrary, in \cite{liang2020just,liang2020multiple,mei2021generalization},
there is no such constraint and a generic non-parametric model for the labels $\set{y_i}$ is
considered. 
Since current proof crucially hinges on the fact that the distribution of $\invec_i$ is isotropic and that $h(\invec)$ only depends on the projection $\invec^{\T}\sgl$, handling more general function classes may require substantial changes to our current proof technique.
Finally, we only analyze the inner-product kernels here. 
It will be interesting to consider other types of kernels, \emph{e.g.}, radial kernel $\Kernalfn(\invec,\invec')=k(\|\invec-\invec'\|/\sqrt{d})$ and translation invariant kernel $\Kernalfn(\invec,\invec')=k(\invec-\invec')$. In the current setting, since $\invec_{i}\iid\spmeasure{d-1}$, it is easy to see that radial kernels can be viewed as inner-product kernels. For more general settings, we conjecture that some versions of
Gaussian equivalence may still hold. 
The extensions to the above more general cases are left as
interesting future work. 

\section{Proof of Main Results}

\subsection{Notations}

Before delving into the formal proof, let us first list some notations
that will be used throughout our proof in the following sections.

For $n\in\mathbb{Z}^{+},$we denote by $[n]$ the set $\{1,2,\cdots,n\}$.
For a vector $\vx\in\R^{n}$, we use $\|\vx\|$ to denote its $\ell_{2}$
norm and for a matrix $\mX\in\R^{m\times n}$, we use $\|\mX\|$ to
denote its operator norm and $\|\mX\|_{F}$ as its Frobenius norm.

For convenience of stating some results regarding deterministic or
probabilistic upper bounds, we will adopt the following notations.
$f(d)\lesssim g(d)$ means that there exists $C>0$ such that $|f(d)|\leq C|g(d)|$
and $|f(d)|\gtrsim g(d)$ means that there exists $c>0$ such that
$|f(d)|\geq c|g(d)|$. Also for two non-negative random variables,
$X\stodom Y$ means that for any $\tau>0$ and $\veps>0$, $\P(X\leq d^{\tau}Y)\leq\veps$
for all large enough $d$.

Our proof will frequently utilize the expansions of $\Kernalfn(\cdot,\cdot)$
and $\teacherfn(\cdot)$ under spherical harmonics $\{Y_{ki}(\vx)\}_{k,i}$
and ultraspherical polynomials $\{\usp k\}_{k}$. For any vector $\va\in\R^{d}$,
$$\shmtx_{k}(\va):=[Y_{k1}(\va),Y_{k2}(\va),\cdots,Y_{k\usdim k}(\va)]^{\T}$$
and for any matrix $\mA=[\va_{1},\va_{2},\cdots,\va_{n}]^{\T}\in\R^{n\times d}$,
$\shmtx_{k}(\mA):=[\shmtx_{k}(\va_{1}),\shmtx_{k}(\va_{2}),\cdots,\shmtx_{k}(\va_{n})]^{\T}$.
In particular, for the input matrix $\inputmtx=[\invec_{1},\invec_{2},\cdots,\invec_{n}]^{\T}\in\R^{n\times d}$,
we denote $\shmtx_{k}:=\shmtx_{k}(\inputmtx)$. In the kernel expansion,
the degree-$k$ component will be denoted as $\kmtx_{k}:=\tkcoeff {k,d}\shmtx_{k}\shmtx_{k}^{\T}$
and likewise for the teacher model, we have $\tevec_{k}:=\tcoeff{k,d}\usp k\big(\frac{\inputmtx\sgl}{\sqrt{d}}\big)=\frac{\tcoeff{k,d}}{\sqrt{\usdim k}}\shmtx_{k}\shmtx_{k}(\sgl)$ and we write $\widetilde{\shmtx}_{k}(\sgl)=\frac{\tcoeff k}{\sqrt{\usdim k}}\shmtx_{k}(\sgl)$.
We also denote $\delta_{k}=n/\usdim k$ as the sampling ratio with
respect to the degree-$k$ component and $\mD_k = \sqrt{\kcoeff{k,d}\delta_k} \mI_{\usdim{k}}$.

We use the following short-hand notations for the partial sum:
$\kmtx_{\leq k}=\sum_{\ell=0}^{k}\kmtx_{k}$,
$\usdim{\leq k}=\sum_{\ell=0}^{k}\usdim{\ell}$,
$\tevec_{\leq k}=\sum_{\ell=0}^{k}\tevec_{\ell}$ and
$\obser_{\leq k}:=\tevec_{\leq k}+\noise$
and block matrix:
$\shmtx_{\leq k}=[\shmtx_{0},\cdots,\shmtx_{k}]$, ${\shmtx}_{\leq k}(\sgl)=[{\shmtx}_{0}(\sgl)^{\T},\cdots,{\shmtx}_{k}(\sgl)^{\T}]^{\T}$
and
$\mD_{\leq k} = \diag\{\mD_{0},\cdots,\mD_{k}\}$. The  quantities like $\kmtx_{>k}$ or $\shmtx_{>k}$ are defined in the same way.

Also since we are focusing on asymptotic results and under our main
assumptions, $\kcoeff{k,d}\to\kcoeff k$ and $\tcoeff{k,d}\to\tcoeff k$ as $d\to\infty$, we will drop the dependence of $\kcoeff{k,d}$ and $\tcoeff{k,d}$
on $d$ in our proof, when it is clear from the context.

\subsection{Proof for the Asymptotic Formula of the Training Error}\label{sec:proof_trainingerr}

We first study the asymptotics of the training error. It can be proved
\cite[Theorem 4.2]{scholkopf2002learning} that the optimal solution
of (\ref{eq:kernel_1}) is:
\begin{equation}
\hat{\regressfn}(\vx)=\sum_{i=1}^{n}\Kernalfn(\invec,\invec_{i})\soli_{i}.\label{eq:representer_theorem}
\end{equation}
Here, $\sol=(\lambda\mI+\kmtx)^{-1}\obser$ is the optimal
solution of
\begin{equation}
\min_{\vw}(\vy-\mK\vw)^{2}+\lambda\vw^{\T}\mK\vw,\label{eq:kernel_2}
\end{equation}
where $\kmtx\in\R^{n\times n}$ is the kernel matrix, with $[\kmtx]_{ij}=\Kernalfn(\invec_{i},\invec_{j})$.
Therefore, $\mathcal{E}_{\text{train}}$ has an explicit form:
\begin{align}
\mathcal{E}_{\text{train}} & =\frac{\lambda}{n}\vy^{\T}\rsmtx\vy,\label{eq:training_error}
\end{align}
where $\rsmtx=(\lambda\mI+\mK)^{-1}$ is the resolvent matrix of $\mK$.

\subsubsection{\label{subsec:training_err_special_case_proof}A Special Case}

We first present the proof for a special case. Consider the following
kernel function $\Kernalfn(\cdot,\cdot)$:
\begin{equation}
\Kernalfn(\vx,\vx')=\frac{\kcoeff{\msc,d}}{\usdim{\msc}}\sum_{i=1}^{\usdim{\msc}}Y_{\msc i}(\vx)Y_{\msc i}(\vx')\label{eq:kernel_expansion_1}
\end{equation}
and teacher function $\teacherfn(\cdot)$:
\begin{equation}
\teacherfn(x)=\tcoeff{\msc,d}\usp{\msc}(x),\label{eq:teacher_expansion_1}
\end{equation}
where $\msc\in\mathbb{Z}^{+}$ is defined as in Assumption (A.1).
Comparing (\ref{eq:kernel_expansion_1}) and (\ref{eq:teacher_expansion_1})
with (\ref{eq:kernel_expansion}) and (\ref{eq:teacher_expansion}),
we can find that (\ref{eq:kernel_expansion_1}) and (\ref{eq:teacher_expansion_1})
correspond to a special model that only retains the degree-$\msc$
component, while discarding all the low-degree and high-degree parts.
Although this may appear to be a substantial simplification of the
original model, it turns out that this simplified setting already
captures some main technical ingredients in the general proof.%

We can make some simplifications utilizing the rotational invariance
of the input vectors $\{\invec_{i}\}_{i=1}^{n}$. Since $\invec_{i}\sim\spmeasure{d-1}$,
we have the following representation:
\begin{equation}
\invec_{i}=(\eta_{i},[(d-\eta_{i}^{2})/(d-1)]^{1/2}\vv_{i}^{\T})^{\T},\label{eq:xi_representation}
\end{equation}
where $\eta_{i}\sim\spmeasure{d-1,1}$, $\vv_{i}\sim\spmeasure{d-2}$
and they are independent. Also by rotational invariance, we can assume
without loss of generality that $\sgl=\sqrt{d}\ve_{1}$. Then substituting
$\sgl=\sqrt{d}\ve_{1}$, (\ref{eq:teacher_expansion_1}) and (\ref{eq:xi_representation})
into (\ref{eq:teacher_model}), we get
\begin{equation}
\obser=\tcoeff{\msc,d}\usp{\msc}(\veta)+\noise,\label{eq:teacher_model_1}
\end{equation}
where $\usp{\msc}(\cdot)$ is applied pointwise on $\veta$. On the
other hand, the kernel function can be written compactly as:
\begin{align}
\Kernalfn(\vx_{i},\vx_{j}) & =\tfrac{\kcoeff{\msc,d}}{\sqrt{\usdim{\msc}}}\usp{\msc}(\vx_{i}^{\T}\vx_{j}/\sqrt{d}),\nonumber \\
 & =\tfrac{\kcoeff{\msc,d}}{\sqrt{\usdim{\msc}}}\usp{\msc}\bigg(\tfrac{\eta_{i}\eta_{j}}{\sqrt{d}}+\tfrac{\vv_{i}^{\T}\vv_{j}}{\sqrt{d-1}}\sqrt{\tfrac{d}{d-1}\Big(1-\tfrac{\eta_{i}^{2}}{d}\Big)\Big(1-\tfrac{\eta_{j}^{2}}{d}\Big)}\bigg)\label{eq:kernel_function_1}
\end{align}
where in the second step, we use (\ref{eq:xi_representation}). Correspondingly,
the kernel matrix $\kmtx$($=\kmtx_{\msc}$) becomes:
\begin{align}
\mK & =\tfrac{\kcoeff{\msc,d}}{\sqrt{\usdim{\msc}}}\usp{\msc}\Big(\tfrac{\veta\veta^{\T}}{\sqrt{d}}+\sqrt{\tfrac{d}{d-1}}\diag\{(1-\eta_{i}^{2}/d)^{1/2}\}\tfrac{\mV\mV^{\T}}{\sqrt{d-1}}\diag\{(1-\eta_{i}^{2}/d)^{1/2}\}\Big)\label{eq:kernel_matrix_1}
\end{align}
where $\veta=(\eta_{1},\eta_{2},\cdots,\eta_{n})^{\T}$ and $\mV=(\vv_{1},\vv_{2},\cdots,\vv_{n})^{\T}$.

From (\ref{eq:teacher_model_1}) we know $\vy$ is a (noisy) function of $\veta$.
Therefore, to compute $\mathcal{E}_{\text{train}}=\frac{\lambda}{n}\vy^{\T}\rsmtx\vy,$
we need to handle the  (weak) correlation between $\veta$ and $\kmtx$.
However, the formulation in (\ref{eq:kernel_matrix_1}) is not amenable
for analysis, as $\kmtx$ depends on $\veta$ in a convoluted way. To
this end, we can apply Proposition 1, Lemma 4 and Lemma 7 in \cite{Lu2022Equi} (with
a slightly different scaling) to get
\begin{equation}
\Big|\frac{1}{n}\obser^{\T}(\rsmtx-\widehat{\rsmtx})\obser\Big|\stodom\frac{1}{\sqrt{d}},\label{eq:training_error_2_1}
\end{equation}
where $\widehat{\rsmtx}:=(\lambda\mI+\widehat{\kmtx})^{-1}$, with
\begin{equation}
\widehat{\kmtx}=\frac{\kcoeff{\msc,d}}{\sqrt{\usdim{\msc}}}\tusp{\msc}(\mV\mV^{\T})+\frac{\kcoeff{\msc,d}}{\usdim{\msc}}\vv_{\msc}(\veta)\vv_{\msc}(\veta)^{\T},\label{eq:kernelmtx_approximation_Kl}
\end{equation}
$\vv_{\ell}(\veta):=(\usp{\ell}(\eta_{1}),\cdots,\usp{\ell}(\eta_{n}))^{\T}$ and $\tusp{\msc}(x):=\usp{\msc,d-1}(x)$.
The approximation $\widehat{\kmtx}$
is much easier to handle,
as it depends on $\veta$ only through a rank-1 matrix.
Define
$\widetilde{\rsmtx}:=(\lambda\mI+\widetilde{\kmtx})^{-1}$, where $\widetilde{\kmtx}:=\frac{\kcoeff{\msc,d}}{\sqrt{\usdim{\msc}}}\tusp{\msc}(\mV\mV^{\T})$. By Sherman--Morrison formula, we can get:
$$
\frac{1}{n}\obser^{\T}\widehat{\rsmtx}\obser=\frac{\tcoeff{\msc,d}^2 a + 2\tcoeff{\msc} b - \delta_{\msc}\kcoeff{\msc,d} b^2}{1+\delta_{\msc}\kcoeff{\msc,d}a} + c,
$$
where $a = \frac{\vv_{\msc}(\veta)^{\T}\widetilde{\rsmtx}\vv_{\msc}(\veta)}{n}$, $b = \frac{\vv_{\msc}(\veta)^{\T}\widetilde{\rsmtx}\noise}{n}$ and $c=\frac{\noise^{\T}\widetilde{\rsmtx}\noise}{n}$. Also by Lemma 5 in \cite{Lu2022Equi}, we have $|a-\widetilde{R}_{\msc}|\stodom\frac{1}{\sqrt{d}}$, $|b|\stodom\frac{1}{\sqrt{d}}$ and $|c-\sigma_{\noisei}^2\widetilde{R}_{\msc}|\stodom\frac{1}{\sqrt{d}}$, where $\widetilde{R}_{\msc}=\frac{1}{n}\Tr\widetilde{\rsmtx}$. Therefore,
\begin{align}
\bigg|\frac{1}{n}\obser^{\T}\widehat{\rsmtx}\obser-\Big(\frac{\tcoeff{\msc,d}^{2}\widetilde{R}_{\msc}}{1+\delta_{\msc}\kcoeff{\msc,d}\widetilde{R}_{\msc}}+\varnoise^{2}\widetilde{R}_{\msc}\Big)\bigg|
\stodom\frac{1}{\sqrt{d}}.
\label{eq:training_error_2_2}
\end{align}
Then after
combining (\ref{eq:training_error_2_1}), (\ref{eq:training_error_2_2})
and Lemma \ref{lem:obser_truncation}, we obtain
\begin{equation}
\frac{1}{n}\obser^{\T}\rsmtx\obser\pconv\frac{\tcoeff k^{2}\widetilde{R}_{\msc}}{1+\delta_{\msc}\kcoeff{\msc}\widetilde{R}_{\msc}}+\varnoise^{2}\widetilde{R}_{\msc}.\label{eq:training_error_2}
\end{equation}
Note that $\widetilde{R}_{\msc}$ is the Stieltjes transform of $\widetilde{\kmtx}$. Then following the same proof of Theorem 1 in \cite{Lu2022Equi}, we can show that $\widetilde{R}_{\msc}\pconv R_{\star,\msc}$,
where $R_{\star,\msc}$ is the unique non-negative solution of $\frac{1}{R}=\lambda+\frac{\mu_{\msc}}{1+\delta_{\msc}\mu_{\msc}R}$.
This concludes the proof.

\subsubsection{General Case}

To extend the proof to the general setting {[}c.f. (\ref{eq:kernel_expansion})
and (\ref{eq:teacher_model}){]}, we need to take into account all
the terms in the expansion of $\Kernalfn(\cdot,\cdot)$ and $\teacherfn(\cdot)$,
not just the degree-$\msc$ component as in (\ref{eq:kernel_expansion_1})
and (\ref{eq:teacher_expansion_1}). The bridge is the following result,
which shows that: (1) the low-degree parts ($k<\msc$) can be truncated
and (2) the high-degree components ($k>\msc$) of kernel function
act as a regularization term.
\begin{prop}
\label{prop:train_error_truncation}There exists $\tau>0$ such that
\begin{equation}
\frac{1}{n}\Big|\obser^{\T}\rsmtx\obser-\obser_{\geq\msc}^{\T}\widetilde{\rsmtx}_{\msc}\obser_{\geq\msc}\Big|\stodom\frac{1}{d^{\tau}}\label{eq:training_err_approx_target}
\end{equation}
where $\widetilde{\rsmtx}_{\msc}=(\tilde{\lambda}\mI+\mK_{\msc})^{-1}$ and $\tilde{\lambda}:=\lambda+\sum_{k>\msc}\kcoeff k$ is an
equivalent regularization parameter.
\end{prop}

\begin{proof}
We have the following decomposition for $\frac{1}{n}\obser^{\T}\rsmtx\obser$:
\begin{align*}
\frac{1}{n}\obser^{\T}\rsmtx\obser= & \left[\frac{1}{n}\obser^{\T}(\rsmtx-\widetilde{\rsmtx}_{\leq\msc})\obser\right]+\left[\frac{1}{n}\obser^{\T}\widetilde{\rsmtx}_{\leq\msc}\obser-\frac{1}{n}\obser_{\geq\msc}^{\T}\widetilde{\rsmtx}_{\leq\msc}\obser_{\geq\msc}\right]\\
 & +\frac{1}{n}\obser_{\geq\msc}^{\T}(\widetilde{\rsmtx}_{\leq\msc}-\widetilde{\rsmtx}_{\msc})\obser_{\geq\msc}+\frac{1}{n}\obser_{\geq\msc}^{\T}\widetilde{\rsmtx}_{\msc}\obser_{\geq\msc},
\end{align*}
where $\widetilde{\rsmtx}_{\leq\msc}=(\tilde{\lambda}\mI+\mK_{\leq\msc})^{-1}$.
The first three terms in the above display are approximation errors.
In Lemma \ref{lem:training_err_approx_step1}, Lemma \ref{lem:training_err_approx_step2}
and Lemma \ref{lem:training_err_approx_step3}, we show that they
all decay to zero as $d\to\infty$, with the desired rate. This completes
the proof.
\end{proof}
Proposition \ref{prop:train_error_truncation} brings us closer to
the special case analyzed previously in Sec. \ref{subsec:training_err_special_case_proof}.
In particular, it implies that all the low-degree components in $\Kernalfn(\cdot,\cdot)$
and $\teacherfn(\cdot)$ can be dropped, without causing any non-vanishing
error and all the higher-degree components of $\Kernalfn(\cdot,\cdot)$
can be equivalently treated as a regularization term. The remaining thing is to handle the higher-degree components contained
in $\obser_{\geq\msc}$. Recall that in the simplified setting (\ref{eq:teacher_expansion_1}),
only the $\msc$th degree component is involved.%

To proceed, we first apply a truncation over $\obser_{\geq\msc}$.
Let $\hat{\obser}_{\geq\msc}=\sum_{k=\msc}^{L}\tcoeff{k}\usp k(\inputmtx\sgl/\sqrt{d})+\noise$,
for some $L\geq\msc$ to be chosen. Then we can follow the same strategy
in Sec. \ref{subsec:training_err_special_case_proof} to obtain
\begin{equation}
\bigg|\frac{1}{n}\hat{\obser}_{\geq\msc}^{\T}\widetilde{\rsmtx}_{\msc}\hat{\obser}_{\geq\msc}-\Big[\frac{\tcoeff{\msc}^{2}R_{\star}}{1+\delta_{\msc}\kcoeff{\msc}R_{\star}}+\Big(\varnoise^{2}+\sum_{k=\msc+1}^{L}\tcoeff{k}^{2}\Big)R_{\star}\Big]\bigg|\pconv0.\label{eq:training_error_3}
\end{equation}

To complete the proof, we just need to show the above approximation
by $\hat{\obser}_{\geq\msc}$ can be made arbitrarily precise. In
particular, by Lemma \ref{lem:obser_truncation}, for any $\veps>0$,
we can always find an $L\in\mathbb{Z}^{+}$ such that for all large
enough $d$, $\sum_{k=L+1}^{\infty}\tcoeff{k,d}^{2}<\frac{\veps}{2}$
and $\P(\frac{1}{n}\|\obser_{\geq\msc}-\hat{\obser}_{\geq\msc}\|^{2}\geq\veps)\leq\frac{C}{n\veps^{2}}$,
where $C>0$ is some constant. These two bounds together with (\ref{eq:training_error_3})
imply that for any $\veps>0$,
\[
\P\bigg\{\bigg|\frac{1}{n}\obser_{\geq\msc}^{\T}\widetilde{\rsmtx}_{\msc}\obser_{\geq\msc}-\Big[\frac{\tcoeff{\msc}^{2}R_{\star}}{1+\delta_{\msc}\kcoeff{\msc}R_{\star}}+\Big(\varnoise^{2}+\sum_{k=\msc+1}^{\infty}\tcoeff{k}^{2}\Big)R_{\star}\Big]\bigg|>\veps\bigg\}\to0,
\]
as $d\to\infty$ and the proof is finished.

\subsection{Proof for the Asymptotic Formula of the Test Error}

Next we analyze the asymptotics of the test error. We can decompose $\mathcal{E}_{\text{test}}$
as:
\begin{align}
\mathcal{E}_{\text{test}} =\underbrace{\E_{\text{new}}\E(y_{\text{new}}\mid\invec_{\text{new}})^{2}}_{\text{I}}-\underbrace{2\E_{\text{new}}[\E(y_{\text{new}}\mid\invec_{\text{new}})\hat{\regressfn}(\invec_{\text{new}})]}_{\text{II}}+\underbrace{\E_{\text{new}}\hat{\regressfn}(\invec_{\text{new}})^{2}}_{\text{III}}.\label{eq:testerr_decomposition}
\end{align}
In the following, we deal with them individually.

\subsubsection{Part I}

It can be directly calculated from (\ref{eq:teacher_model}) and (\ref{eq:teacher_expansion})
that $\E_{\text{new}}\E(y_{\text{new}}\mid\invec_{\text{new}})^{2}=\sum_{k=0}^{\infty}\tcoeff{k,d}^{2}$.
Then by Lemma \ref{lem:obser_truncation}, as $d\to\infty$
\begin{equation}
\E_{\text{new}}\E(y_{\text{new}}\mid\invec_{\text{new}})^{2}\to\sum_{k=0}^{\infty}\tcoeff k^{2}.\label{eq:testerr_sglnormsqterm_finalform}
\end{equation}

\subsubsection{Part II}
Following the similar strategy as in Sec. \ref{sec:proof_trainingerr}, we can get 
\begin{equation}
\E_{\text{new}}[y_{\text{new}}\hat{\regressfn}(\invec_{\text{new}})]=\frac{1}{n}\tilde{\tevec}^{\T}\rsmtx\obser\pconv\sum_{k<\msc}\tcoeff k^{2}+\tfrac{\kcoeff{\msc}\delta_{\msc}\tcoeff{\msc}^{2}R_{\star}}{1+\kcoeff{\msc}\delta_{\msc}R_{\star}}\label{eq:testerr_crossterm_final_form}
\end{equation}
where $\tilde{\tevec}=\sum_{k=0}^{\infty}\kcoeff k\delta_{k}\tevec_{k}$. The proof is deferred to Appendix \ref{sec:proof_of_cross}.

\subsubsection{Part III}

This part of proof follows the same strategy as Part I and Part II.
In particular, we can get
\begin{equation}
\E_{\text{new}}\hat{\regressfn}(\invec_{\text{new}})^{2}\pconv\sum_{k<\msc}\tcoeff k^{2}+\frac{(\delta_{\msc}\kcoeff{\msc}\tcoeff{\msc}R_{\star})^{2}}{(1+\delta_{\msc}\kcoeff{\msc}R_{\star})^{2}}+\frac{\tcoeff{\msc}^{2}}{(\theta-1)(1+\delta_{\msc}\kcoeff{\msc}R_{\star})^{2}}+\frac{\varnoise^{2}+\sum_{k>\msc}\tcoeff k^{2}}{\theta-1}.\label{eq:testerr_normsqterm_finalform}
\end{equation}
The details are relegated to Appendix \ref{sec:Proof-of-hsq}.

\subsection{Finish the Proof}

The final step is to substitute (\ref{eq:testerr_sglnormsqterm_finalform}),
(\ref{eq:testerr_crossterm_final_form}) and (\ref{eq:testerr_normsqterm_finalform})
back to (\ref{eq:testerr_decomposition}) and get (\ref{eq:main_results_testerr}).

\bibliography{refs}

\begin{thebibliography}{32}
\providecommand{\natexlab}[1]{#1}
\providecommand{\url}[1]{\texttt{#1}}
\expandafter\ifx\csname urlstyle\endcsname\relax
  \providecommand{\doi}[1]{doi: #1}\else
  \providecommand{\doi}{doi: \begingroup \urlstyle{rm}\Url}\fi

\bibitem[Adlam and Pennington(2020)]{adlam2020neural}
Ben Adlam and Jeffrey Pennington.
\newblock The neural tangent kernel in high dimensions: Triple descent and a
  multi-scale theory of generalization.
\newblock In \emph{International Conference on Machine Learning}, pages 74--84.
  PMLR, 2020.

\bibitem[Belkin et~al.(2018)Belkin, Ma, and Mandal]{belkin2018understand}
Mikhail Belkin, Siyuan Ma, and Soumik Mandal.
\newblock To understand deep learning we need to understand kernel learning.
\newblock \emph{arXiv preprint arXiv:1802.01396}, 2018.

\bibitem[Billingsley(2008)]{billingsley2008probability}
Patrick Billingsley.
\newblock \emph{Probability and measure}.
\newblock John Wiley \& Sons, 2008.

\bibitem[Bordelon et~al.(2020)Bordelon, Canatar, and
  Pehlevan]{bordelon2020spectrum}
Blake Bordelon, Abdulkadir Canatar, and Cengiz Pehlevan.
\newblock Spectrum dependent learning curves in kernel regression and wide
  neural networks.
\newblock In \emph{International Conference on Machine Learning}, pages
  1024--1034. PMLR, 2020.

\bibitem[Canatar et~al.(2021)Canatar, Bordelon, and
  Pehlevan]{canatar2021spectral}
Abdulkadir Canatar, Blake Bordelon, and Cengiz Pehlevan.
\newblock Spectral bias and task-model alignment explain generalization in
  kernel regression and infinitely wide neural networks.
\newblock \emph{Nature communications}, 12\penalty0 (1):\penalty0 1--12, 2021.

\bibitem[Cheng and Singer(2013)]{cheng2013spectrum}
Xiuyuan Cheng and Amit Singer.
\newblock The spectrum of random inner-product kernel matrices.
\newblock \emph{Random Matrices: Theory and Applications}, 2\penalty0
  (04):\penalty0 1350010, 2013.

\bibitem[Dai and Xu(2013)]{dai2013approximation}
Feng Dai and Yuan Xu.
\newblock \emph{Approximation theory and harmonic analysis on spheres and
  balls}, volume~23.
\newblock Springer, 2013.

\bibitem[Daniely et~al.(2016)Daniely, Frostig, and Singer]{daniely2016toward}
Amit Daniely, Roy Frostig, and Yoram Singer.
\newblock Toward deeper understanding of neural networks: The power of
  initialization and a dual view on expressivity.
\newblock In \emph{Advances In Neural Information Processing Systems}, pages
  2253--2261, 2016.

\bibitem[d'Ascoli et~al.(2020)d'Ascoli, Refinetti, Biroli, and
  Krzakala]{d2020double}
St{\'e}phane d'Ascoli, Maria Refinetti, Giulio Biroli, and Florent Krzakala.
\newblock Double trouble in double descent: Bias and variance (s) in the lazy
  regime.
\newblock In \emph{International Conference on Machine Learning}, pages
  2280--2290. PMLR, 2020.

\bibitem[Dicker(2016)]{dicker2016ridge}
Lee~H Dicker.
\newblock Ridge regression and asymptotic minimax estimation over spheres of
  growing dimension.
\newblock \emph{Bernoulli}, 22\penalty0 (1):\penalty0 1--37, 2016.

\bibitem[Dietrich et~al.(1999)Dietrich, Opper, and
  Sompolinsky]{dietrich1999statistical}
Rainer Dietrich, Manfred Opper, and Haim Sompolinsky.
\newblock Statistical mechanics of support vector networks.
\newblock \emph{Physical review letters}, 82\penalty0 (14):\penalty0 2975,
  1999.

\bibitem[Dobriban and Wager(2018)]{dobriban2018high}
Edgar Dobriban and Stefan Wager.
\newblock High-dimensional asymptotics of prediction: Ridge regression and
  classification.
\newblock \emph{The Annals of Statistics}, 46\penalty0 (1):\penalty0 247--279,
  2018.

\bibitem[Du et~al.(2019)Du, Lee, Li, Wang, and Zhai]{du2019gradient}
Simon Du, Jason Lee, Haochuan Li, Liwei Wang, and Xiyu Zhai.
\newblock Gradient descent finds global minima of deep neural networks.
\newblock In \emph{International conference on machine learning}, pages
  1675--1685. PMLR, 2019.

\bibitem[Ghorbani et~al.(2021)Ghorbani, Mei, Misiakiewicz, and
  Montanari]{ghorbani2021linearized}
Behrooz Ghorbani, Song Mei, Theodor Misiakiewicz, and Andrea Montanari.
\newblock Linearized two-layers neural networks in high dimension.
\newblock \emph{The Annals of Statistics}, 49\penalty0 (2):\penalty0
  1029--1054, 2021.

\bibitem[Hastie et~al.(2009)Hastie, Tibshirani, and
  Friedman]{hastie2009elements}
Trevor Hastie, Robert Tibshirani, and Jerome~H Friedman.
\newblock \emph{The elements of statistical learning: data mining, inference,
  and prediction}, volume~2.
\newblock Springer, 2009.

\bibitem[Hastie et~al.(2019)Hastie, Montanari, Rosset, and
  Tibshirani]{hastie2019surprises}
Trevor Hastie, Andrea Montanari, Saharon Rosset, and Ryan~J Tibshirani.
\newblock Surprises in high-dimensional ridgeless least squares interpolation.
\newblock \emph{arXiv preprint arXiv:1903.08560}, 2019.

\bibitem[Jacot et~al.(2018)Jacot, Gabriel, and Hongler]{jacot2018neural}
Arthur Jacot, Franck Gabriel, and Cl{\'e}ment Hongler.
\newblock Neural tangent kernel: Convergence and generalization in neural
  networks.
\newblock In \emph{Advances in neural information processing systems}, pages
  8571--8580, 2018.

\bibitem[Liang et~al.(2020{\natexlab{a}})Liang, Rakhlin, and
  Zhai]{liang2020multiple}
Tengyuan Liang, Alexander Rakhlin, and Xiyu Zhai.
\newblock On the multiple descent of minimum-norm interpolants and restricted
  lower isometry of kernels.
\newblock In \emph{Conference on Learning Theory}, pages 2683--2711. PMLR,
  2020{\natexlab{a}}.

\bibitem[Liang et~al.(2020{\natexlab{b}})Liang, Rakhlin, et~al.]{liang2020just}
Tengyuan Liang, Alexander Rakhlin, et~al.
\newblock Just interpolate: Kernel ``ridgeless'' regression can generalize.
\newblock \emph{Annals of Statistics}, 48\penalty0 (3):\penalty0 1329--1347,
  2020{\natexlab{b}}.

\bibitem[Lu and Yau(2022)]{Lu2022Equi}
Yue~M. Lu and Horng-Tzer Yau.
\newblock An equivalence principle for the spectrum of random inner-product
  kernel matrices.
\newblock \emph{arXiv preprint}, 2022.

\bibitem[Mei and Montanari(2019)]{mei2019generalization}
Song Mei and Andrea Montanari.
\newblock The generalization error of random features regression: Precise
  asymptotics and double descent curve.
\newblock \emph{arXiv preprint arXiv:1908.05355}, 2019.

\bibitem[Mei et~al.(2021)Mei, Misiakiewicz, and
  Montanari]{mei2021generalization}
Song Mei, Theodor Misiakiewicz, and Andrea Montanari.
\newblock Generalization error of random feature and kernel methods:
  hypercontractivity and kernel matrix concentration.
\newblock \emph{Applied and Computational Harmonic Analysis}, 2021.

\bibitem[Misiakiewicz(2022)]{misiakiewicz2022spectrum}
Theodor Misiakiewicz.
\newblock Spectrum of inner-product kernel matrices in the polynomial regime
  and multiple descent phenomenon in kernel ridge regression.
\newblock \emph{arXiv preprint arXiv:2204.10425}, 2022.

\bibitem[Nakkiran et~al.(2021)Nakkiran, Kaplun, Bansal, Yang, Barak, and
  Sutskever]{nakkiran2021deep}
Preetum Nakkiran, Gal Kaplun, Yamini Bansal, Tristan Yang, Boaz Barak, and Ilya
  Sutskever.
\newblock Deep double descent: Where bigger models and more data hurt.
\newblock \emph{Journal of Statistical Mechanics: Theory and Experiment},
  2021\penalty0 (12):\penalty0 124003, 2021.

\bibitem[Neal(1996)]{neal1996priors}
Radford~M Neal.
\newblock Priors for infinite networks.
\newblock In \emph{Bayesian Learning for Neural Networks}, pages 29--53.
  Springer, 1996.

\bibitem[Rahimi and Recht(2008)]{rahimi2008random}
Ali Rahimi and Benjamin Recht.
\newblock Random features for large-scale kernel machines.
\newblock In \emph{Advances in neural information processing systems}, pages
  1177--1184, 2008.

\bibitem[Rakhlin and Zhai(2019)]{rakhlin2019consistency}
Alexander Rakhlin and Xiyu Zhai.
\newblock Consistency of interpolation with laplace kernels is a
  high-dimensional phenomenon.
\newblock In \emph{Conference on Learning Theory}, pages 2595--2623. PMLR,
  2019.

\bibitem[Richards et~al.(2021)Richards, Mourtada, and
  Rosasco]{richards2021asymptotics}
Dominic Richards, Jaouad Mourtada, and Lorenzo Rosasco.
\newblock Asymptotics of ridge (less) regression under general source
  condition.
\newblock In \emph{International Conference on Artificial Intelligence and
  Statistics}, pages 3889--3897. PMLR, 2021.

\bibitem[Sch{\"o}lkopf et~al.(2002)Sch{\"o}lkopf, Smola, Bach,
  et~al.]{scholkopf2002learning}
Bernhard Sch{\"o}lkopf, Alexander~J Smola, Francis Bach, et~al.
\newblock \emph{Learning with kernels: support vector machines, regularization,
  optimization, and beyond}.
\newblock MIT press, 2002.

\bibitem[Vershynin(2018)]{vershynin2018high}
Roman Vershynin.
\newblock \emph{High-dimensional Probability: {An} Introduction with
  Applications in Data Science}.
\newblock Cambridge University Press, 2018.

\bibitem[Williams(1996)]{williams1996computing}
Christopher Williams.
\newblock Computing with infinite networks.
\newblock \emph{Advances in neural information processing systems}, 9, 1996.

\bibitem[Wu and Xu(2020)]{wu2020optimal}
Denny Wu and Ji~Xu.
\newblock On the optimal weighted l2 regularization in overparameterized linear
  regression.
\newblock \emph{Advances in Neural Information Processing Systems},
  33:\penalty0 10112--10123, 2020.

\end{thebibliography}

\appendix

\section{\label{sec:Spherical-Harmonics-and}Spherical Harmonics and Ultraspherical
Polynomials}

In this section, we give a brief overview of some properties of the
spherical harmonics and ultraspherical polynomials that are frequently
used in our analysis. A detailed coverage of these topics can be found
in \cite{dai2013approximation}.

Spherical harmonics are homogeneous harmonic polynomials restricted
on $\dsphere d$. In other words, a polynomial $P(\vx)$, $\vx\in\dsphere d$
is a spherical harmonic if and only if (1) $P(t\vx)=t^{d}P(t\vx)$,
for any $t\in\R$, (2) $\Delta P=0$, where $\Delta$ is the Laplace
operator. Let $\mathcal{H}_{k,d}$ be the space of all degree-$k$
spherical harmonics in $d$ dimension. Then $\dim(\mathcal{H}_{k,d})=\usdim k$,
where
\begin{equation}
\usdim k=\begin{cases}
1 & k=0\\
d & k=1\\
\frac{d+2k-2}{k}{d+k-3 \choose k-1} & k\geq2
\end{cases}\label{eq:usdim_def}
\end{equation}
From (\ref{eq:usdim_def}), we can find that $|\usdim k/{d \choose k}-1|\leq\frac{C_{k}}{d}$,
where $C_{k}>0$ is some constant that only depends on $k$, \emph{i.e.},
$\usdim k$ is approximately equal to the combinatorial number up
to an $\mathcal{O}(\frac{1}{d})$ correction. Any orthonormal basis of $\mathcal{H}_{k,d}$, $k\geq0$ is denoted
by $\{Y_{ki}(\vx)\}_{i=1}^{\usdim k}$ and they satisfy
\begin{equation}
\int_{\dsphere d}Y_{ki}(\vx)Y_{kj}(\vx)\spmeasure{d-1}(d\vx)=\indicatorfn_{i=j}.\label{eq:orthogonal_Y}
\end{equation}
where $\spmeasure{d-1}$ is the uniform distribution over $\dsphere d$.

Ultraspherical polynomials $\{\usp k\}_{k=0}^{\infty}$ are orthonormal polynomials on $L^{2}([-\sqrt{d},\sqrt{d}],\spmeasure{d-1,1})$,
\emph{i.e.},
\[
\int_{-\sqrt{d}}^{\sqrt{d}}\usp k(x)\usp{\ell}(x)\spmeasure{d-1,1}(dx)=\indicatorfn_{k=\ell}
\]
where $\spmeasure{d-1,1}$ is the distribution of $\vx^{\T}\ve_{1}$,
with $\vx\sim\spmeasure{d-1}$ and it also has the explicit form:
$\spmeasure{d-1,1}(dx)=\frac{\omega_{d-2}}{\sqrt{d}\omega_{d-1}}(1-x^{2})^{\frac{d-3}{2}}dx$,
where $\omega_{d-1}=\frac{2\pi^{d/2}}{\Gamma(d/2)}$ is the surface
area of $\mathcal{S}^{d-1}$. The moments of measure $\spmeasure{d-1,1}$
equal to
\begin{equation}
\E_{\spmeasure{d-1,1}}X^{m}=\begin{cases}
0 & m=2k+1\\
\frac{(2k-1)!!}{\prod_{0\leq i<k}(1+2i/d)} & m=2k
\end{cases}\label{eq:moments_of_usp}
\end{equation}
Based on (\ref{eq:moments_of_usp}), we can explicitly write out the
first three $\usp k(x)$ via the Gram-Schmidt procedure:
\begin{align*}
\usp 0(x) & =1, & \usp 1(x) & =x, & \usp 2(x) & =\frac{1}{\sqrt{2}}\sqrt{\frac{d+2}{d-1}}(x^{2}-1).
\end{align*}

In particular, for any $k\geq0$, $\usp k(x)$ and $\{Y_{ki}(\vx)\}_{i=1}^{\usdim k}$
have the following correspondence: for any $\vx,\vx'\in\dsphere d$,
\begin{equation}
\usp k(\invec^{\T}\invec'/\sqrt{d})=\frac{1}{\sqrt{\usdim k}}\sum_{i=1}^{\usdim k}Y_{ki}(\vx)Y_{ki}(\vx'),\label{eq:q_Y_relation}
\end{equation}
which is also known as the addition theorem. As a result, the expansion
(\ref{eq:kernel_expansion}) can also be written as:%
\begin{equation}
\kernalfn_{d}(\invec^{\T}\invec'/\sqrt{d})=\sum_{k=0}^{\infty}\sqrt{\usdim k}\tkcoeff{k,d}\usp k(\invec^{\T}\invec'/\sqrt{d}).\label{eq:fd_under_qk}
\end{equation}
Also we can deduce that for any $\vx\in\dsphere d$,
\begin{equation}
\frac{1}{\sqrt{\usdim k}}\sum_{i=1}^{\usdim k}Y_{ki}(\vx)^{2}=\usp k(\sqrt{d})=\sqrt{\usdim k}.\label{eq:diagonal_Y_constant}
\end{equation}
Indeed, for any $\vx\in\dsphere d$,
\begin{align*}
\usp k(\sqrt{d}) & =\usp k(\|\vx\|^{2}/\sqrt{d})\\
 & =\int_{\dsphere d}\usp k(\|\vx\|^{2}/\sqrt{d})\spmeasure{d-1}(d\vx)\\
 & \teq{\text{(a)}}\frac{1}{\sqrt{\usdim k}}\sum_{i=1}^{\usdim k}\int_{\dsphere d}Y_{ki}(\vx)^{2}\spmeasure{d-1}(d\vx)\\
 & \teq{\text{(b)}}\sqrt{\usdim k},
\end{align*}
where (a) follows from (\ref{eq:q_Y_relation}) and (b) follows from
(\ref{eq:orthogonal_Y}).

The ultraspherical polynomials are closely related with the Hermite polynomials  $\{\hermitefn_k\}_{k=0}^{\infty}$, which are the orthonormal polynomials on $L^{2}(\R, \tau_G)$, where $\tau_G$ denotes the standard Gaussian measure. From (\ref{eq:moments_of_usp}), we can see that for any fixed $k$, $\lim_{d\to\infty}\E_{\spmeasure{d-1,1}}X^{m}=\E_{\tau_G}X^{m}$,
where $\tau_G$ denotes the standard Gaussian measure. By Theorem
30.1 in \citep{billingsley2008probability}, $\spmeasure{d-1,1}$ converges weakly to $ \tau_G$. Therefore, as $d\to\infty$, we can get $\usp{k,d}(x)\to \hermitefn_k(x)$ pointwise on $\R$.

\section{\label{sec:kcoeff_convergece}Convergence of Expansion Coefficients}
\begin{lem}
\label{fact:kcoeff_convergence}Suppose $\kernalfn_{d}(x)=\kernalfn(\frac{x}{\sqrt{d}})$,
where $\kernalfn(z)$ is defined on $[-1,1]$ satisfying $\kernalfn(z)\leq C_{f}$ and $\kernalfn(z)\in C^{\infty}[-\upsilon,\upsilon]$,
for some $C_{f}>0$ and $\upsilon\in(0,1]$. Then for any fixed $k\geq 0$, \[\kcoeff{k,d}\to\kcoeff{k}=\frac{\kernalfn^{(k)}(0)}{k!},
\]
as $d\to\infty$.
\end{lem}

\begin{proof}
First, we analyze the special case $\upsilon=1$, \emph{i.e.}, $f(z)\in C^{\infty}[-1,1]$.
From (\ref{eq:fd_under_qk}), we have
\begin{equation}
\kcoeff{k,d}=\sqrt{\usdim k}\int_{-\sqrt{d}}^{\sqrt{d}}\kernalfn(\tfrac{x}{\sqrt{d}})\usp k(x)\spmeasure{d-1,1}(dx),\label{eq:mu_d_exact_1}
\end{equation}
where $\spmeasure{d-1,1}(dx)=\frac{\omega_{d-2}}{\sqrt{d}\omega_{d-1}}(1-x^{2}/d)^{\frac{d-3}{2}}dx$
,with $\omega_{d-1}=\frac{2\pi^{d/2}}{\Gamma(d/2)}$. Then we can
utilize the Rodrigues' formula for $\usp k(x)$:
\[
\usp k(x)=\sqrt{\usdim k}C_{k,d}\sqrt{d}^{k}\Big(1-\frac{x^{2}}{d}\Big)^{\frac{3-d}{2}}\Big(\frac{d}{dx}\Big)^{k}\Big(1-\frac{x^{2}}{d}\Big)^{k+\frac{d-3}{2}}
\]
where $C_{k,d}=\Big(-\frac{1}{2}\Big)^{k}\frac{\Gamma(\frac{d-1}{2})}{\Gamma(k+\frac{d-1}{2})}$.
Then
\begin{align}
 & \int_{-\sqrt{d}}^{\sqrt{d}}\kernalfn(\tfrac{x}{\sqrt{d}})\usp k(x)\spmeasure{d-1,1}(dx)\nonumber \\
= & \sqrt{\usdim k}C_{k,d}\sqrt{d}^{k}\int_{-\sqrt{d}}^{\sqrt{d}}\kernalfn(\tfrac{x}{\sqrt{d}})\cdot\Big(1-\frac{x^{2}}{d}\Big)^{\frac{3-d}{2}}\Big(\frac{d}{dx}\Big)^{k}\Big(1-\frac{x^{2}}{d}\Big)^{k+\frac{d-3}{2}}\cdot\frac{\omega_{d-2}}{\sqrt{d}\omega_{d-1}}\Big(1-\frac{x^{2}}{d}\Big)^{\frac{d-3}{2}}dx\nonumber \\
\teq{\text{(a)}} & \sqrt{\usdim k}C_{k,d}\frac{\omega_{d-2}}{\omega_{d-1}}\int_{-1}^{1}\kernalfn(z)\cdot\Big(\frac{d}{dz}\Big)^{k}(1-z^{2})^{k+\frac{d-3}{2}}dz\nonumber \\
\teq{\text{(b)}} & \sqrt{\usdim k}C_{k,d}(-1)^{k}\frac{\omega_{d-2}}{\omega_{d-1}}\int_{-1}^{1}(1-z^{2})^{k+\frac{d-3}{2}}\kernalfn^{(k)}(z)dz\nonumber \\
= & \frac{\sqrt{\usdim k}}{2^{k}\prod_{i=0}^{k-1}(\frac{d-1}{2}+i)}\int_{-1}^{1}\frac{\omega_{d-2}}{\omega_{d-1}}(1-z^{2})^{k+\frac{d-3}{2}}\kernalfn^{(k)}(z)dz\label{eq:mu_d_exact_2}
\end{align}
where in (a) we make a change of variable: $\frac{x}{\sqrt{d}}\to z$
and in (b) we use integration by parts for $k$ times. After combining
(\ref{eq:mu_d_exact_1}) and (\ref{eq:mu_d_exact_2}), we get
\[
\kcoeff{k,d}=\frac{\usdim k}{\prod_{i=0}^{k-1}(d-1+2i)}\int_{-1}^{1}\frac{\omega_{d-2}}{\omega_{d-1}}(1-z^{2})^{k+\frac{3-d}{2}}\kernalfn^{(k)}(z)dz.
\]
As $d\to\infty$, $\frac{\usdim k}{\prod_{i=0}^{k-1}(d-1+2i)}\to\frac{1}{k!}$.
Note that $\frac{\omega_{d-2}}{\omega_{d-1}}(1-z^{2})^{\frac{d-3}{2}}$
is the density of distribution of $\frac{\vx^{\T}\ve_{1}}{\sqrt{d}},$where
$\vx\sim\spmeasure{d-1}$, so
\begin{align}
\int_{-1}^{1}\frac{\omega_{d-2}}{\omega_{d-1}}(1-z^{2})^{k+\frac{3-d}{2}}\kernalfn^{(k)}(z)dz & =\E\Big[\kernalfn^{(k)}\Big(\frac{\vx^{\T}\ve_{1}}{\sqrt{d}}\Big)\cdot\Big(1-\frac{(\vx^{\T}\ve_{1})^{2}}{d}\Big)^{k}\Big]\nonumber \\
 & \to\kernalfn^{(k)}(0)\label{eq:mu_d_converge_1}
\end{align}
where the last step follows from dominated convergence theorem, due
to the assumption that $\kernalfn(z)\in C^{\infty}[-1,1]$. Therefore,
we get $\kcoeff{k,d}\to\frac{\kernalfn^{(k)}(0)}{k!}$.

Then we analyze the general case: $\kernalfn(z)\in C^{\infty}[-\upsilon,\upsilon]$, for some
$\upsilon\in(0,1]$. To utilize the results for $\upsilon=1$, we
need the following truncation argument. For any $\veps\in(0,1]$ and
$k$, it holds that
\begin{align}
\sqrt{\usdim k}\int_{-\sqrt{d}}^{\sqrt{d}}\big|\kernalfn(\tfrac{x}{\sqrt{d}})\usp k(x)\indicatorfn_{|x|/\sqrt{d}\geq\veps}\big|\spmeasure{d-1,1}(dx)= & \sqrt{\usdim k}\E_{Z}\big|\kernalfn(Z)\usp k(\sqrt{d}Z)\indicatorfn_{|Z|\geq\veps}\big|\nonumber \\
\tleq{\text{(a)}} & \sqrt{\usdim k C_{f}^{2}\P(|Z|\geq\veps)}\nonumber \\
\lesssim & \sqrt{d^{k}(1-\veps^{2})^{\frac{d-3}{2}}}\label{eq:mu_d_bd_1}
\end{align}
where $Z\sim\frac{\vx^{\T}\ve_{1}}{\sqrt{d}}$, $\vx\sim\spmeasure{d-1}$
and in (a), we use Cauchy-Schwarz inequality, with $\E\usp k(\sqrt{d}Z)^{2}=1$
and $\kernalfn(z)\leq C_{f}$. Clearly, in (\ref{eq:mu_d_bd_1}),
$\sqrt{d^{k}(1-\veps^{2})^{\frac{d-3}{2}}}\to0$ as $d\to\infty$,
so this implies that for any $\veps\in(0,1]$ and $k$, as $d\to\infty$,
\begin{equation}
\sqrt{\usdim k}\int_{-\sqrt{d}}^{\sqrt{d}}\big|\kernalfn(\tfrac{x}{\sqrt{d}})-\hat{\kernalfn}(\tfrac{x}{\sqrt{d}})\big|\cdot\big|\usp k(x)\big|\cdot\spmeasure{d-1,1}(dx)\to0,\label{eq:mu_d_bd_2}
\end{equation}
where $\hat{\kernalfn}(z)=\kernalfn(z)\indicatorfn_{|z|\leq\veps}$.
From (\ref{eq:mu_d_bd_2}) we conclude that for any bounded function
$r(z)$ on $[-1,1]$, if for some $\veps\in(0,1]$, $r(z)=\kernalfn(z)$
on $[-\veps,\veps]$, then as $d\to\infty$
\begin{equation}
|\kcoeff{k,d}(r)-\kcoeff{k,d}|\to0,\label{eq:mu_d_converge_2}
\end{equation}
where $\kcoeff{k,d}(r):=\sqrt{\usdim k}\int_{-\sqrt{d}}^{\sqrt{d}}r(\tfrac{x}{\sqrt{d}})\usp k(x)\spmeasure{d-1,1}(dx)$.

In light of (\ref{eq:mu_d_converge_2}) and the results we have established
for the smooth functions {[}c.f. (\ref{eq:mu_d_converge_1}){]}, it remains
to choose a proper smoothed approximation of $\kernalfn(z)$ on $[-1,1]$,
which agrees with $\kernalfn(z)$ in some neighborhood of 0. One such choice is:
\[
\tilde{\kernalfn}(z):=\int\kernalfn(z-t)\indicatorfn_{|z-t|\leq\veps}\cdot m_{\varsigma}(t)dt,
\]
where $\varsigma=\frac{1}{3}\min\{\veps,1-\veps\}$ and $m_{\varsigma}(t)=\frac{1}{\varsigma}m(\frac{t}{\varsigma})$,
with
\[
m(s)=\begin{cases}
ce^{-\frac{1}{1-s^{2}}} & |s|<1\\
0 & |s|\geq1
\end{cases}
\]
and $c$ is a normalizing constant such that $\int m(s)ds=1$. It
can be directly verified that $\tilde{\kernalfn}(z)\in C^{\infty}[-1,1]$
and $\tilde{\kernalfn}(z)=\kernalfn(z)$, when $|z|\leq\veps-\varsigma$.
Then applying (\ref{eq:mu_d_converge_2}) and (\ref{eq:mu_d_converge_1}),
we obtain that
\[
\kcoeff{k,d}\to\kcoeff{k,d}(\tilde{\kernalfn}):=\sqrt{\usdim k}\int_{-\sqrt{d}}^{\sqrt{d}}\tilde{\kernalfn}(\tfrac{x}{\sqrt{d}})\usp k(x)\spmeasure{d-1,1}(dx)\to\frac{\tilde{\kernalfn}^{(k)}(0)}{k!}=\frac{\kernalfn^{(k)}(0)}{k!}.
\]
\end{proof}

\begin{lem}
\label{lem:obser_truncation}For any $k\geq0$, as $d\to\infty$
\[
\tcoeff{k,d}\to\tcoeff k=\int\teacherfn(x)\hermitefn_{k}(x)\tau_G(dx)
\]
where $\hermitefn_{k}(x)$ is the degree-$k$ Hermite polynomial and $\tau_G$
denotes the standard Gaussian measure. Also for any $\veps>0$, there
exists $L\in\mathbb{Z}^{+}$ and $C>0$, such that for any
large enough $d$, 
\[
\E g_{>L,i}^{2}=\sum_{k=L+1}^{\infty}\tcoeff{k,d}^{2}<{\veps}
\]
and 
\[
\P(\frac{1}{n}\|\tevec_{>L}\|^{2}>\veps)\leq\frac{C}{n\veps^{2}},
\]
where $g_{>L,i}$ is the $i$th coordinate of $\tevec_{>L}$, $i\in[n]$.
\end{lem}

\begin{proof}
From (\ref{eq:moments_of_usp}), we get for any fixed $j\in\mathbb{Z}^{+}$,
there exists $C_{j}>0$ such that for any $d\in\mathbb{Z}^{+}$,
\begin{equation}
|\E_{\spmeasure{d-1,1}}X^{j}-\E_{\tau_G}X^{j}|\leq\frac{C_{j}}{d}.\label{eq:kmoments_sp_gauss_diff}
\end{equation}
Also by Assumption
(A.4), $g(x)\leq C_{g}(1+|x|^{K_{g}})$, so there exists $C>0$ such
that
\begin{align*}
\lim_{R\to\infty}\sup_{d\in\mathbb{Z}^{+}}\E_{\spmeasure{d-1,1}}[g(X)^{2}\indicatorfn_{|X|\geq R}] & \leq\lim_{R\to\infty}\sup_{d\in\mathbb{Z}^{+}}\E_{\spmeasure{d-1,1}}C_{g}^{2}(1+|X|^{K_{g}})^{2}\indicatorfn_{|X|\geq R}\\
 & \leq\lim_{R\to\infty}\sup_{d\in\mathbb{Z}^{+}}\sqrt{\E_{\spmeasure{d-1,1}}C_{g}^{4}(1+|X|^{K_{g}})^{4}\cdot\E_{\spmeasure{d-1,1}}\indicatorfn_{|X|\geq R}}\\
 & \leq\lim_{R\to\infty}\sup_{d\in\mathbb{Z}^{+}}\sqrt{\big[\E_{\tau_G}C_{g}^{4}(1+|X|^{K_{g}})^{4}+\tfrac{C}{d}\big]\cdot\E_{\spmeasure{d-1,1}}\indicatorfn_{|X|\geq R}}\\
 & =\lim_{R\to\infty}\sup_{d\geq R^{2}}\sqrt{\big[\E_{\tau_G}C_{g}^{4}(1+|X|^{K_{g}})^{4}+\tfrac{C}{d}\big]\cdot\E_{\spmeasure{d-1,1}}\indicatorfn_{|X|\geq R}}\\
 & =0
\end{align*}
where in the last step, we use the fact that $\spmeasure{d-1,1}$
converges weakly to $\tau_G$, as $d\to\infty$. 
By Lemma C.5 in \cite{cheng2013spectrum}, we have
\[
\int \teacherfn(x)^2|\spmeasure{d-1}(x)-\tau_G(x)|dx \to 0.
\]
Then by Lemma C.1 in \cite{cheng2013spectrum},
we get for any fixed $k\geq0$, $\tcoeff{k,d}\to\tcoeff k$ and by Lemma C.2 in \cite{cheng2013spectrum}, for any $\veps>0$,
there exists a fixed $L\in\mathbb{Z}^{+}$ such that for any large
enough $d$ and $i\in[n]$, $\E g_{>L,i}^{2}=\sum_{k=L+1}^{\infty}\tcoeff{k,d}^{2}<\frac{\veps}{2}$.
Then 
\begin{align*}
\P(\frac{1}{n}\|\tevec_{>L}\|^{2}>\veps) & \leq\P(\frac{1}{n}(\|\tevec_{>L}\|^{2}-\E\|\tevec_{>L}\|^{2})>\frac{\veps}{2})\\
 & \leq\frac{4\var(g_{>L,i}^{2})}{n\veps^{2}}\\
 & \leq C_{1}\frac{\E\teacherfn_{i}^{4}+\E\teacherfn_{\leq L,i}^{4}}{n\veps^{2}}\\
 & \leq\frac{C_{2}}{n\veps^{2}},
\end{align*}
where $C_{1},C_{2}$ are two constants and in the last step, we use
$\E\teacherfn_{i}^{4},\E\teacherfn_{\leq L,i}^{4}<\infty$, which
follows from Assumption (A.4) and (\ref{eq:kmoments_sp_gauss_diff}).
\end{proof}

\section{Spectral Norm Bounds}

The following is a concentration result for the spectral norm of kernel
matrix.
\begin{lem}
\label{lem:samplecov_spectral_norm}Let $\shmtx=[\vy_{1},\vy_{2},\ldots,\vy_{n}]^{\T}\in\R^{n\times N}$
be a matrix with $n$ independent rows $\{\vy_{i}\}_{i=1}^{n}$ satisfying
$\E\vy_{i}=\boldsymbol{0}$
and $\E\vy_{i}\vy_{i}^{\T}=\mI_{N}$ and
$\|\vy_{i}\|^{2}=N$.
For any $t>0$, we have
\begin{equation}
\P(\|\frac{1}{n}\shmtx\shmtx^{\T}\|>1+t)\leq2N\exp\big(-\tfrac{\delta}{8}\min\{t^{2},t\}\big)\label{eq:lem:spectral_norm_1_ubd}
\end{equation}
where $\delta=n/N$. %
Also there exists $c>0$ such that for $\delta\geq(\log N)^{2}$
\begin{equation}
\P\big(\lambda_{\min}(\frac{1}{n}\shmtx^{\T}\shmtx)\leq1/2\big)\leq c\exp(-(\log N)^{2}/c).\label{eq:lem:spectral_norm_1_domilbd}
\end{equation}
\end{lem}

\begin{proof}
Since $\|\frac{1}{n}\shmtx\shmtx^{\T}\|=\|\frac{1}{n}\shmtx^{\T}\shmtx\|$,
it is equivalent to prove all the results for the sample covariance
matrix $\frac{1}{n}\shmtx^{\T}\shmtx$. Denote $\mX_{i}=\frac{1}{n}(\vy_{i}\vy_{i}^{\T}-\mI)$.
We have
\begin{align*}
\|\mX_{i}\| & \leq\frac{1}{n}(\|\vy_{i}\vy_{i}^{\T}\|+\|\mI\|)\\
 & =\frac{N+1}{n}
\end{align*}
and
\begin{align*}
\|\sum_{i=1}^{n}\E\mX_{i}^{2}\| & =\frac{1}{n^{2}}\|\sum_{i=1}^{n}\E(\vy_{i}\vy_{i}^{\T}-\mI)^{2}\|\\
 & =\frac{1}{n^{2}}\|n( N-1)\mI\|\\
 & =\frac{N-1}{n}.
\end{align*}
By matrix Bernstein's inequality \cite[Theorem 5.4.1]{vershynin2018high},
for any $t>0$
\begin{align}
\P\Big(\|\frac{1}{n}\shmtx^{\T}\shmtx-\mI\|\geq t\Big) & =\P\Big(\|\sum_{i=1}^{n}\mX_{i}\|\geq t\Big)\nonumber \\
 & \leq2N\exp\Big(-\tfrac{t^{2}/4}{n^{-1} N(1+t/3)}\Big)\nonumber \\
 & \leq2N\exp\big(-\tfrac{\delta}{8}\min\{t^{2},t\}\big).\label{eq:lem:spectral_norm_1_origin}
\end{align}
Therefore, for any $t>0$,
\begin{align*}
\P\Big(\|\frac{1}{n}\shmtx^{\T}\shmtx\|\geq1+t\Big) & \leq\P\Big(\|\frac{1}{n}\shmtx^{\T}\shmtx-\mI\|\geq t\Big)\\
 & \leq2N\exp\big(-\tfrac{\delta}{8}\min\{t^{2},t\}\big)
\end{align*}
which proves (\ref{eq:lem:spectral_norm_1_ubd}). %

Finally, since $\lambda_{\min}(\frac{1}{n}\shmtx^{\T}\shmtx)\geq1-\|\frac{1}{n}\shmtx^{\T}\shmtx-\mI\|$,
\begin{align*}
\P\big(\lambda_{\min}(\frac{1}{n}\shmtx^{\T}\shmtx)\leq1/2\big) & \leq\P\big(\|\frac{1}{n}\shmtx^{\T}\shmtx-\mI\|\geq1/2\big)\\
 & \leq2N\exp\big(-\tfrac{\delta}{32}\big)
\end{align*}
where we use (\ref{eq:lem:spectral_norm_1_origin}). Therefore, when
$\delta\geq(\log N)^{2}$, $\P\big(\lambda_{\min}(\frac{1}{n}\shmtx^{\T}\shmtx)\leq1/2\big)\leq c\exp(-(\log N)^{2}/c)$,
for some $c>0$.
\end{proof}
\begin{lem}
\label{lem:training_err_approx_step1}There exists $\tau>0$ such
that 
\begin{equation}
\|\rsmtx-\widetilde{\rsmtx}_{\leq\msc}\|\stodom\frac{1}{d^{\tau}}\label{eq:training_err_approx_step1}
\end{equation}
where $\widetilde{\rsmtx}_{\leq\msc}=(\tilde{\lambda}\mI+\mK_{\leq\msc})^{-1}$
and $\tilde{\lambda}:=\lambda+\sum_{k>\msc}\kcoeff k$.
\end{lem}

\begin{proof}
By (72) in \cite{ghorbani2021linearized}, for any fixed $k\in\mathbb{Z}^{+}$,
there exists $C>0$ such that for all large enough $d$ , 
\begin{equation}
\E\|\kmtx_{k}-\kcoeff k\mI\|\leq C\Big[p^{3/2}n^{1/2p}\sqrt{\frac{n}{d^{k}}}+(\frac{n}{d^{k}})^{1/p}\Big]\label{eq:higher_order_norm_bd_Mpapaer}
\end{equation}
where $p\in\mathbb{Z}^{+}$ need to satisfy $2p\leq-\log\left(\frac{Cnp^{k+1}}{d^{k}}\right)$.
Here we choose $p=2\msc$. Since $n\asymp d^{\msc}$ by Assumption
(A.1), when $k>\msc$, $2p\leq-\log\left(\frac{Cnp^{k+1}}{d^{k}}\right)$
is satisfied for all large $d$. Then substituting $p=2\msc$ in (\ref{eq:higher_order_norm_bd_Mpapaer}),
we get there exists $C>0$, such that for any fixed $k>K$ and large
enough $d$,
\begin{equation}
\E\|\kmtx_{k}-\kcoeff k\mI\|\leq C\Big[d^{-\frac{1}{2}(k-\msc)+\frac{1}{4}}+d^{-\frac{k-\msc}{2\msc}}\Big].\label{eq:higher_order_norm_bd_Mpapaer_1}
\end{equation}
On the other hand, following the steps leading to (55) in \cite{ghorbani2021linearized},
we can get for any $L\geq2\msc+3$, 
\begin{equation}
\E\|\kmtx_{\geq L}-\sum_{k\geq L}\kcoeff k\mI\|^{2}\lesssim\frac{1}{d}.\label{eq:higher_order_norm_bd_Mpapaer_2}
\end{equation}
Combining (\ref{eq:higher_order_norm_bd_Mpapaer_1}) and (\ref{eq:higher_order_norm_bd_Mpapaer_2}),
we can get there exists $\tau>0$, such that 
\begin{equation}
\|\mK_{>\msc}-\sum_{k>\msc}\kcoeff k\mI\|\lesssim\frac{1}{d^{\tau}}.\label{eq:higher_order_norm_bd_Mpapaer_3}
\end{equation}
Therefore, combine (\ref{eq:higher_order_norm_bd_Mpapaer_3}) with
the fact $\|\widetilde{\rsmtx}_{\leq\msc}\|,\|\rsmtx\|\leq\frac{1}{\lambda}$,
we have
\begin{align*}
\|\rsmtx-\widetilde{\rsmtx}_{\leq\msc}\| & =\|\widetilde{\rsmtx}_{\leq\msc}(\mK_{>\msc}-\sum_{k>\msc}\kcoeff k\mI)\rsmtx\|\\
 & \stodom\frac{1}{d^{\tau}}.
\end{align*}
\end{proof}
\begin{lem}
\label{lem:spectral_norm_1}For any positive semi-definite matrix
$\mM\in\R^{n\times n}$ satisfying 
\begin{equation}
1\lesssim\lambda_{\min}(\mM)\leq\lambda_{\max}(\mM)\stodom1,\label{eq:spectral_lbd_ubd}
\end{equation}
we have
\begin{equation}
\|\frac{1}{\sqrt{n}}\mM\shmtx_{<\msc}\|\lesssim1\label{eq:spectral_RK_YleK_bd}
\end{equation}
and 
\begin{equation}
\|(\mD_{<\msc}^{-2}+\frac{1}{n}\shmtx_{<\msc}^{\T}\mM\shmtx_{<\msc})^{-1}\|\stodom1.\label{eq:spectral_Dk-2_bd}
\end{equation}
Here, $\lambda_{\min}(\cdot)$ and $\lambda_{\max}(\cdot)$ are the
smallest and largest eigenvalue of $\mM$ and $$\mD_{<\msc} := \diag\{\mD_0,\mD_1,\cdots,\mD_{\msc-1}\}.$$ On the other hand, $(\lambda\mI+\kmtx_{\msc})^{-1}$
and $(\lambda\mI+\kmtx_{>\msc})^{-1}$ both satisfy (\ref{eq:spectral_lbd_ubd}),
for any $\lambda>0$.
\end{lem}

\begin{proof}
Denote $\rsmtx_{\msc}=(\lambda\mI+\kmtx_{\msc})^{-1}$ and $\rsmtx_{>\msc}=(\lambda\mI+\kmtx_{>\msc})^{-1}$.
By Lemma \ref{lem:samplecov_spectral_norm}, we have $\|\frac{1}{n}\shmtx_{\msc}\shmtx_{\msc}^{\T}\|\stodom1$,
so
\begin{align*}
\|\frac{1}{\sqrt{n}}\mM\shmtx_{<\msc}\| & \leq\|\mM\|\cdot\|\frac{1}{\sqrt{n}}\shmtx_{<\msc}\|\\
 & \lesssim1
\end{align*}
and we can also get $1\lesssim\lambda_{\min}(\rsmtx_{\msc})$. On
the other hand, by (\ref{eq:lem:spectral_norm_1_domilbd}) we also
have $1\lesssim\lambda_{\min}(\frac{1}{n}\shmtx_{<\msc}^{\T}\shmtx_{<\msc})$,
since $n/\usdim{<\msc}\gtrsim\log(\usdim{<\msc})^{2}$. As a result,
\begin{align*}
\|(\mD_{<\msc}^{-2}+\frac{1}{n}\shmtx_{<\msc}^{\T}\rsmtx_{\msc}\shmtx_{<\msc})^{-1}\| & \leq\|(\frac{1}{n}\shmtx_{<\msc}^{\T}\rsmtx_{\msc}\shmtx_{<\msc})^{-1}\|\\
 & \leq[\lambda_{\min}(\rsmtx_{\msc})\cdot\lambda_{\min}(\frac{1}{n}\shmtx_{<\msc}^{\T}\shmtx_{<\msc})]^{-1}\\
 & \stodom1,
\end{align*}
where the last step follows from $1\lesssim\lambda_{\min}(\rsmtx_{\msc})$
and $1\lesssim\lambda_{\min}(\frac{1}{n}\shmtx_{<\msc}^{\T}\shmtx_{<\msc})$.

Finally, we verify that both $\rsmtx_{\msc}$ and $\rsmtx_{>\msc}$
satisfy (\ref{eq:spectral_lbd_ubd}). Clearly, $\lambda_{\max}(\rsmtx_{\msc}),\lambda_{\max}(\rsmtx_{>\msc})\leq\frac{1}{\lambda}\lesssim1$.
On the other hand, since $\|\frac{1}{n}\shmtx_{\msc}\shmtx_{\msc}^{\T}\|\stodom1$
by Lemma \ref{lem:samplecov_spectral_norm} and $\delta_{\msc}\lesssim1$,
we get $1\lesssim\lambda_{\min}(\rsmtx_{\msc})$. For $\lambda_{\min}(\rsmtx_{>\msc})$,
we can apply (\ref{eq:higher_order_norm_bd_Mpapaer_3}) and obtain
that $\|\mK_{>\msc}\|\stodom2\sum_{k>\msc}\kcoeff k\lesssim1$, which
implies that $1\lesssim\lambda_{\min}(\rsmtx_{>\msc})$.
\end{proof}

%


\section{Auxiliary Results for Analyzing Training Error}
\begin{lem}
\label{lem:y_norm_bd}For any $\mathcal{I}\subseteq\{0,1,2,\cdots\}$,
we have $\|\sum_{k\in\mathcal{I}}\tevec_{k}\|\stodom\sqrt{n}$. Also
$\|\noise\|\stodom\sqrt{n}$.
\end{lem}

\begin{proof}
Recall that $\tevec_{k}=\frac{\tcoeff k}{\sqrt{\usdim k}}\shmtx_{k}\shmtx_{k}(\vxi)$.
By rotational invariance, we can assume that $\vxi\sim\spmeasure{d-1}$.
Therefore,
\begin{align*}
\E\|\sum_{k\in\mathcal{I}}\tevec_{k}\|^{2} & =\E\sum_{k\in\mathcal{I}}\frac{\tcoeff k^{2}}{\usdim k}\tr(\shmtx_{k}\shmtx_{k}^{\T})\\
 & =n\sum_{k\in\mathcal{I}}\tcoeff k^{2}.
\end{align*}
Since $\sum_{k=0}^{\infty}\tcoeff k^{2}\lesssim1$, we get by Chebyshev's
inequality, $\|\sum_{k\in\mathcal{I}}\tevec_{k}\|\stodom\sqrt{n}$.
In a similar way, we can get $\|\noise\|\stodom\sqrt{n}$.
\end{proof}

\begin{lem}
\label{lem:muted_lowerdegree}For any positive semidefinite matrix
$\mM\in\R^{n\times n}$, we have%
\begin{equation}
(\lambda\mI+\kmtx_{<\msc}+\mM)^{-1}\shmtx_{<\msc}=\rsmtx_{M}\shmtx_{<\msc}(\mD_{<\msc}^{-2}+\frac{1}{n}\shmtx_{<\msc}^{\T}\rsmtx_{M}\shmtx_{<\msc})^{-1}\cdot\mD_{<\msc}^{-2}\label{eq:muted_lowerdegree}
\end{equation}
where $\rsmtx_{M}=(\lambda\mI+\mM)^{-1}$ and 
$\mD_{<\msc} := \diag\{\mD_0,\mD_1,\cdots,\mD_{\msc-1}\}.$
\end{lem}

\begin{proof}
By the matrix inversion formula, we have
\begin{align*}
&~~~(\lambda\mI+\kmtx_{<\msc}+\mM)^{-1} \\
& =(\lambda\mI+\frac{1}{n}\shmtx_{<\msc}\mD_{<\msc}^{2}\shmtx_{<\msc}^{\T}+\mM)^{-1}\\
 & =\rsmtx_{M}-\frac{1}{n}\rsmtx_{M}\shmtx_{<\msc}\mD_{<\msc}(\mI+\frac{1}{n}\mD_{<\msc}\shmtx_{<\msc}^{\T}\rsmtx_{M}\shmtx_{<\msc}\mD_{<\msc})^{-1}\mD_{<\msc}\shmtx_{<\msc}^{\T}\rsmtx_{M}\\
 & =\rsmtx_{M}-\frac{1}{n}\rsmtx_{M}\shmtx_{<\msc}(\mD_{<\msc}^{-2}+\frac{1}{n}\shmtx_{<\msc}^{\T}\rsmtx_{M}\shmtx_{<\msc})^{-1}\shmtx_{<\msc}^{\T}\rsmtx_{M}.
\end{align*}
Note that due to Assumption (A.3), $\mD_{<\msc}^{-2}$ is well-defined.
Then%
\begin{align*}
(\lambda\mI+\kmtx_{<\msc}+\mM)^{-1}\shmtx_{<\msc} & =\rsmtx_{M}\shmtx_{<\msc}[\mI-(\mD_{<\msc}^{-2}+\frac{1}{n}\shmtx_{<\msc}^{\T}\rsmtx_{M}\shmtx_{<\msc})^{-1}\cdot(\frac{1}{n}\shmtx_{<\msc}^{\T}\rsmtx_{M}\shmtx_{<\msc})]\\
 & =\rsmtx_{M}\shmtx_{<\msc}(\mD_{<\msc}^{-2}+\frac{1}{n}\shmtx_{<\msc}^{\T}\rsmtx_{M}\shmtx_{<\msc})^{-1}\cdot\mD_{<\msc}^{-2}.
\end{align*}
\end{proof}
\begin{lem}
\label{lem:training_err_approx_step2}It holds that
\begin{equation}
\frac{1}{n}|\obser^{\T}\widetilde{\rsmtx}_{\leq\msc}\obser-(\tevec_{\geq\msc}^{\T}+\vz^{\T})\widetilde{\rsmtx}_{\leq\msc}(\tevec_{\geq\msc}+\vz)|\stodom\frac{1}{d}.\label{eq:training_err_approx_step2}
\end{equation}
where $\widetilde{\rsmtx}_{\leq\msc}=(\tilde{\lambda}\mI+\mK_{\leq\msc})^{-1}$
and $\tilde{\lambda}:=\lambda+\sum_{k>\msc}\kcoeff k$.
\end{lem}

\begin{proof}
We first show 
\begin{equation}
\Big|\frac{1}{\sqrt{n}}\widetilde{\rsmtx}_{\leq\msc}\tevec_{<\msc}\Big|\stodom\frac{1}{d}.\label{eq:training_err_approx_step2_1}
\end{equation}
By (\ref{eq:muted_lowerdegree}), we have 
\begin{align}
\frac{1}{\sqrt{n}}\widetilde{\rsmtx}_{\leq\msc}\tevec_{<\msc} & =\frac{1}{\sqrt{n}}\widetilde{\rsmtx}_{\leq\msc}\shmtx_{<\msc}\widetilde{\shmtx}_{<\msc}(\sgl)\nonumber \\
 & =\frac{1}{\sqrt{n}}\widetilde{\rsmtx}_{\msc}\shmtx_{<\msc}(\mD_{<\msc}^{-2}+\frac{1}{n}\shmtx_{<\msc}^{\T}\widetilde{\rsmtx}_{\msc}\shmtx_{<\msc})^{-1}\cdot\mD_{<\msc}^{-2}\widetilde{\shmtx}_{<\msc}(\sgl),\label{eq:ht_R_yleK}
\end{align}
where $\widetilde{\rsmtx}_{\msc}=(\tilde{\lambda}\mI+\mK_{\msc})^{-1}$ and $\widetilde{\shmtx}_{<\msc}(\sgl)=[\widetilde{\shmtx}_{0}(\sgl)^{\T},\cdots,\widetilde{\shmtx}_{\msc-1}(\sgl)^{\T}]^{\T}$.
Substituting (\ref{eq:spectral_RK_YleK_bd}) and (\ref{eq:spectral_Dk-2_bd})
in Lemma \ref{lem:spectral_norm_1} into (\ref{eq:ht_R_yleK}), we
have
\begin{align*}
|\frac{1}{\sqrt{n}}\widetilde{\rsmtx}_{\leq\msc}\tevec_{<\msc}|\leq & \|\frac{1}{\sqrt{n}}\widetilde{\rsmtx}_{\msc}\shmtx_{<\msc}\|\cdot\|(\mD_{<\msc}^{-2}+\frac{1}{n}\shmtx_{<\msc}^{\T}\widetilde{\rsmtx}_{\msc}\shmtx_{<\msc})^{-1}\|\cdot\|\mD_{<\msc}^{-2}\|\cdot\|\widetilde{\shmtx}_{<\msc}(\sgl)\|\\
\lesssim & \frac{1}{d}
\end{align*}
which is (\ref{eq:training_err_approx_step2_1}). Since
\begin{align*}
 & \frac{1}{n}|\obser^{\T}\widetilde{\rsmtx}_{\leq\msc}\obser-(\tevec_{\geq\msc}^{\T}+\vz^{\T})\widetilde{\rsmtx}_{\leq\msc}(\tevec_{\geq\msc}+\vz)|\\
= & 2|\frac{1}{n}\obser^{\T}\widetilde{\rsmtx}_{\leq\msc}\tevec_{<\msc}|+|\frac{1}{n}\tevec_{<\msc}^{\T}\widetilde{\rsmtx}_{\leq\msc}\tevec_{<\msc}|
\end{align*}
and $\frac{1}{\sqrt{n}}\|\obser\|\lesssim1$, $\frac{1}{\sqrt{n}}\|\tevec_{<\msc}\|\lesssim1$
by Lemma \ref{lem:y_norm_bd}, we can obtain (\ref{eq:training_err_approx_step2}),
after applying (\ref{eq:training_err_approx_step2_1}) with Cauchy-Schwarz
inequality.
\end{proof}
\begin{lem}
\label{lem:training_err_approx_step3}It holds that
\begin{equation}
\frac{1}{n}|(\tevec_{\geq\msc}^{\T}+\vz^{\T})\cdot(\widetilde{\rsmtx}_{\msc}-\widetilde{\rsmtx}_{\leq\msc})\cdot(\tevec_{\geq\msc}+\vz)|\stodom\frac{1}{\sqrt{d}}\label{eq:training_err_approx_step3}
\end{equation}
where $\widetilde{\rsmtx}_{\msc}=(\tilde{\lambda}\mI+\mK_{\msc})^{-1}$
and $\widetilde{\rsmtx}_{\leq\msc}=(\tilde{\lambda}\mI+\mK_{\leq\msc})^{-1}$,
with $\tilde{\lambda}:=\lambda+\sum_{k>\msc}\kcoeff k$.
\end{lem}

\begin{proof}
By the matrix inversion formula, 
\begin{align*}
\widetilde{\rsmtx}_{\msc}-\widetilde{\rsmtx}_{\leq\msc}= & \frac{1}{n}\widetilde{\rsmtx}_{\msc}\shmtx_{<\msc}(\mD_{<\msc}^{-2}+\frac{1}{n}\shmtx_{<\msc}^{\T}\widetilde{\rsmtx}_{\msc}\shmtx_{<\msc})^{-1}\shmtx_{<\msc}^{\T}\widetilde{\rsmtx}_{\msc},
\end{align*}
so
\begin{align}
 & \frac{1}{n}(\tevec_{\geq\msc}^{\T}+\vz^{\T})\cdot(\widetilde{\rsmtx}_{\msc}-\widetilde{\rsmtx}_{\leq\msc})\cdot(\tevec_{\geq\msc}+\vz)\nonumber \\
= & \frac{1}{n^{2}}(\tevec_{\geq\msc}^{\T}+\vz^{\T})\widetilde{\rsmtx}_{\msc}\shmtx_{<\msc}(\mD_{<\msc}^{-2}+\frac{1}{n}\shmtx_{<\msc}^{\T}\widetilde{\rsmtx}_{\msc}\shmtx_{<\msc})^{-1}\shmtx_{<\msc}^{\T}\widetilde{\rsmtx}_{\msc}(\tevec_{\geq\msc}+\vz).\label{eq:training_err_approx_step3_1}
\end{align}
It remains to control $\|\frac{1}{n}\shmtx_{<\msc}^{\T}\widetilde{\rsmtx}_{\msc}\tevec_{\geq\msc}\|,$
$\|\frac{1}{n}\shmtx_{<\msc}^{\T}\widetilde{\rsmtx}_{\msc}\vz\|$
and $\|(\mD_{<\msc}^{-2}+\frac{1}{n}\shmtx_{<\msc}^{\T}\widetilde{\rsmtx}_{\msc}\shmtx_{<\msc})^{-1}\|$
in (\ref{eq:training_err_approx_step3_1}). First, by Lemma \ref{lem:spectral_norm_1}
we have $\|(\mD_{<\msc}^{-2}+\frac{1}{n}\shmtx_{<\msc}^{\T}\widetilde{\rsmtx}_{\msc}\shmtx_{<\msc})^{-1}\|\lesssim1$.
By the independence between $\frac{1}{n}\shmtx_{<\msc}^{\T}\widetilde{\rsmtx}_{\msc}$
and $\noise$, we have $\E_{\noise}\|\frac{1}{n}\shmtx_{<\msc}^{\T}\widetilde{\rsmtx}_{\msc}\vz\|^{2}=\frac{1}{n^{2}}\tr(\shmtx_{<\msc}^{\T}\widetilde{\rsmtx}_{\msc}^{2}\shmtx_{<\msc})$.
Meanwhile, using Lemma \ref{lem:spectral_norm_1}, we can get $\frac{1}{n^{2}}\tr(\shmtx_{<\msc}^{\T}\widetilde{\rsmtx}_{\msc}^{2}\shmtx_{<\msc})\stodom\frac{\usdim{<\msc}}{n}\lesssim\frac{1}{d}.$
Hence, $\|\frac{1}{n}\shmtx_{<\msc}^{\T}\widetilde{\rsmtx}_{\msc}\vz\|\stodom\frac{1}{\sqrt{d}}$.

Finally, we show $\|\frac{1}{n}\shmtx_{<\msc}^{\T}\widetilde{\rsmtx}_{\msc}\tevec_{\geq\msc}\|\stodom d^{-1/4}$.
Recall that 
\[
\tevec_{\geq\msc}=\sum_{k\geq\msc}\frac{\tcoeff k}{\sqrt{\usdim k}}\shmtx_{k}\shmtx_{k}(\vxi)
\]
and by rotational invariance, we can assume that $\vxi\sim\spmeasure{d-1}$. Therefore, conditioning on $\inputmtx$ and taking expectation over $\sgl\sim\spmeasure{d-1}$,
we have
\begin{align}
\E_{\sgl}\|\frac{1}{n}\shmtx_{<\msc}^{\T}\widetilde{\rsmtx}_{\msc}\tevec_{\geq\msc}\|^{2} & =\frac{1}{n^{2}}\tr\Big[(\widetilde{\rsmtx}_{\msc}\shmtx_{<\msc}\shmtx_{<\msc}^{\T}\widetilde{\rsmtx}_{\msc})\cdot\Big(\sum_{k\geq\msc}\frac{\tcoeff k^{2}}{\usdim k}\shmtx_{k}\shmtx_{k}^{\T}\Big)\Big]\nonumber \\
 & \leq\frac{1}{n^{2}}\|\widetilde{\rsmtx}_{\msc}\shmtx_{<\msc}\shmtx_{<\msc}^{\T}\widetilde{\rsmtx}_{\msc}\|_{F}\cdot\|\sum_{k\geq\msc}\frac{\tcoeff k^{2}}{\usdim k}\shmtx_{k}\shmtx_{k}^{\T}\|_{F}\nonumber \\
 & =\frac{1}{n^{2}}\sqrt{\tr(\shmtx_{<\msc}^{\T}\widetilde{\rsmtx}_{\msc}^{2}\shmtx_{<\msc})^{2}}\cdot\|\sum_{k\geq\msc}\frac{\tcoeff k^{2}}{\usdim k}\shmtx_{k}\shmtx_{k}^{\T}\|_{F}\nonumber \\
 & \leq\frac{1}{n^{2}}\sqrt{\usdim{<K}}\|\widetilde{\rsmtx}_{\msc}\shmtx_{<\msc}\|^{2}\cdot\|\sum_{k\geq\msc}\frac{\tcoeff k^{2}}{\usdim k}\shmtx_{k}\shmtx_{k}^{\T}\|_{F}\nonumber \\
 & \leq\frac{1}{n\lambda^{2}}\sqrt{\usdim{<K}}\|\frac{1}{n}\shmtx_{<\msc}\shmtx_{<\msc}^{\T}\|\cdot\|\sum_{k\geq\msc}\frac{\tcoeff k^{2}}{\usdim k}\shmtx_{k}\shmtx_{k}^{\T}\|_{F},\label{eq:E_xi_YleK_RK_ygeK}
\end{align}
where $\usdim{<K}:=\sum_{k=0}^{\msc-1}\usdim k$. The right-hand side
of (\ref{eq:E_xi_YleK_RK_ygeK}) can be controlled as follows. From
(\ref{eq:lem:spectral_norm_1_ubd}), we can get 
\[
\P(\|\frac{1}{n}\shmtx_{<\msc}\shmtx_{<\msc}^{\T}\|\geq2)\leq2\usdim{<\msc}\exp\big(-\tfrac{\delta_{<\msc}}{8}\big),
\]
where $\delta_{<\msc}=\frac{n}{\usdim{<K}}$. Also $\|\frac{1}{n}\shmtx_{<\msc}\shmtx_{<\msc}^{\T}\|\leq\tr(\frac{1}{n}\shmtx_{<\msc}\shmtx_{<\msc}^{\T})=\usdim{<\msc}$.
Therefore,
\begin{align}
\E\|\frac{1}{n}\shmtx_{<\msc}\shmtx_{<\msc}^{\T}\|^{2} & =\E\|\frac{1}{n}\shmtx_{<\msc}\shmtx_{<\msc}^{\T}\|^{2}\indicatorfn_{\|\frac{1}{n}\shmtx_{<\msc}\shmtx_{<\msc}^{\T}\|\geq2}+\E\|\frac{1}{n}\shmtx_{<\msc}\shmtx_{<\msc}^{\T}\|^{2}\indicatorfn_{\|\frac{1}{n}\shmtx_{<\msc}\shmtx_{<\msc}^{\T}\|<2}\nonumber \\
 & \leq\usdim{<\msc}^{2}\exp\big(-\tfrac{\delta_{<\msc}}{8}\big)+4\nonumber \\
 & \lesssim1,\label{eq:E_xi_YleK_RK_ygeK_1}
\end{align}
where in the second step, we use the deterministic bound $\|\frac{1}{n}\shmtx_{<\msc}\shmtx_{<\msc}^{\T}\|\leq\frac{1}{n}\tr(\shmtx_{<\msc}\shmtx_{<\msc}^{\T})=\usdim{<\msc}$
and in the last step we use the fact that $\usdim{<\msc}\lesssim d^{\msc-1}$
and $\delta_{<\msc}\gtrsim d$. On the other hand, $\|\sum_{k\geq\msc}\frac{\tcoeff k^{2}}{\usdim k}\shmtx_{k}\shmtx_{k}^{\T}\|_{F}$
can be controlled as:
\begin{align}
\E\|\sum_{k\geq\msc}\frac{\tcoeff k^{2}}{\usdim k}\shmtx_{k}\shmtx_{k}^{\T}\|_{F}^{2} & =\sum_{k,l\geq\msc}\frac{\tcoeff k^{2}\tcoeff{\ell}^{2}}{\usdim k\usdim{\ell}}\tr\E[\shmtx_{\ell}^{\T}\shmtx_{k}\shmtx_{k}^{\T}\shmtx_{\ell}]\nonumber \\
 & =\sum_{k,l\geq\msc}\frac{\tcoeff k^{2}\tcoeff{\ell}^{2}}{\usdim k\usdim{\ell}}[\usdim kn+n(n-1)\delta_{k\ell}]\usdim{\ell}\nonumber \\
 & =n\sum_{\substack{k,\ell\geq K}
}\tcoeff k^{2}\tcoeff{\ell}^{2}+n(n-1)\sum_{k\geq\msc}\frac{\tcoeff k^{4}}{\usdim k}\nonumber \\
 & \leq C_{\msc}n\cdot(\sum_{k\geq\msc}\tcoeff k^{2})^{2}\nonumber \\
 & \lesssim n,\label{eq:E_xi_YleK_RK_ygeK_2}
\end{align}
where $C_{\msc}$ is a constant that only depends on $\msc$, in the
second step, we use Lemma 12 of \cite{ghorbani2021linearized} and
the second to last step follows from $\usdim k\gg n$, when $k>\msc$.
Combining (\ref{eq:E_xi_YleK_RK_ygeK}), (\ref{eq:E_xi_YleK_RK_ygeK_1})
and (\ref{eq:E_xi_YleK_RK_ygeK_2}), we can get 
\begin{align*}
\E\|\frac{1}{n}\shmtx_{<\msc}^{\T}\widetilde{\rsmtx}_{\msc}\tevec_{\geq\msc}\|^{2} & =\E_{\inputmtx}\E_{\sgl}\|\frac{1}{n}\shmtx_{<\msc}^{\T}\widetilde{\rsmtx}_{\msc}\tevec_{\geq\msc}\|^{2}\\
 & \leq\frac{1}{n\lambda^{2}}\sqrt{\usdim{<K}\E\|\frac{1}{n}\shmtx_{<\msc}\shmtx_{<\msc}^{\T}\|^{2}\cdot\E\|\sum_{k\geq\msc}\frac{\tcoeff k^{2}}{\usdim k}\shmtx_{k}\shmtx_{k}^{\T}\|_{F}^{2}}\\
 & \lesssim\sqrt{\frac{\usdim{<K}}{n}}\\
 & \lesssim d^{-1/2}.
\end{align*}
Then by Chebyshev's inequality, we get $\|\frac{1}{n}\shmtx_{<\msc}^{\T}\widetilde{\rsmtx}_{\msc}\tevec_{\geq\msc}\|\stodom d^{-1/4}.$

Substituting $\|\frac{1}{n}\shmtx_{<\msc}^{\T}\widetilde{\rsmtx}_{\msc}\tevec_{\geq\msc}\|\stodom d^{-1/4}$,
$\|\frac{1}{n}\shmtx_{<\msc}^{\T}\widetilde{\rsmtx}_{\msc}\vz\|\stodom d^{-1/2}$
and
\[
\|(\mD_{<\msc}^{-2}+\frac{1}{n}\shmtx_{<\msc}^{\T}\widetilde{\rsmtx}_{\msc}\shmtx_{<\msc})^{-1}\|\stodom1
\]
back to (\ref{eq:training_err_approx_step3_1}), we reach at the desired
result.
\end{proof}

\section{\label{sec:proof_of_cross} Proof of \eqref{eq:testerr_crossterm_final_form}}
By (\ref{eq:representer_theorem}) and (\ref{eq:inner_product_kernel_def}),
$\hat{\regressfn}(\invec_{\text{new}})$ can be written as
\begin{equation}
\hat{\regressfn}(\invec_{\text{new}})=\kernalfn\Big(\tfrac{\invec_{\text{new}}^{\T}\inputmtx^{\T}}{\sqrt{d}}\Big)\rsmtx\vy,\label{eq:hath_x_new_express_1}
\end{equation}
so we can get
\begin{align}
\E_{\text{new}}[\E(y_{\text{new}}\mid\invec_{\text{new}})\hat{\regressfn}(\invec_{\text{new}})] & =\E_{\text{new}}\Big[\E(y_{\text{new}}\mid\invec_{\text{new}})\kernalfn\Big(\tfrac{\invec_{\text{new}}^{\T}\inputmtx^{\T}}{\sqrt{d}}\Big)\rsmtx\vy\Big]\nonumber \\
 & =\E_{\text{new}}\Big[\big(\sum_{k=0}^{\infty}\tfrac{\tcoeff k}{\sqrt{\usdim k}}\shmtx_{k}(\sgl)^{\T}\shmtx_{k}(\invec_{\text{new}})\big)\cdot\big(\sum_{k=0}^{\infty}\tfrac{\kcoeff k}{\usdim k}\shmtx_{k}(\invec_{\text{new}})^{\T}\shmtx_{k}^{\T}\big)\rsmtx\vy\Big]\nonumber \\
 & =\frac{1}{n}\tilde{\tevec}^{\T}\rsmtx\vy,\label{eq:testerr_crossterm_1}
\end{align}
where $\tilde{\tevec}=\sum_{k=0}^{\infty}\kcoeff k\delta_{k}\tevec_{k}$
and in the second step we use the expansion of $\teacherfn(\cdot)$
and $\kernalfn(\cdot)$ under spherical harmonics. Comparing (\ref{eq:testerr_crossterm_1})
and (\ref{eq:training_error}), we can see that $\E_{\text{new}}[\E(y_{\text{new}}\mid\invec_{\text{new}})\hat{\regressfn}(\invec_{\text{new}})]$
takes a similar form as $\mathcal{E}_{\text{train}}$. We will apply
a similar proof strategy here.

First, consider the truncation of $\obser$: $\hat{\obser}=\sum_{k=0}^{L}\tevec_{k}+\noise$,
where $L\geq\msc$ is some constant to be chosen. The next result
shows that $\frac{1}{n}\tilde{\tevec}^{\T}\rsmtx\hat{\obser}$ can
be approximated by dropping the low-degree and high-degree terms in
the expansions of $\Kernalfn(\cdot,\cdot)$ and $\teacherfn(\cdot)$.
\begin{prop}
\label{prop:test_error_cross_truncation}There exists $\tau>0$ such
that
\[
\bigg|\frac{1}{n}\tilde{\tevec}^{\T}\rsmtx\hat{\vy}-\big(\sum_{k<\msc}\tcoeff k^{2}+\frac{1}{n}\tilde{\tevec}_{\msc}^{\T}\widetilde{\rsmtx}_{\msc}\hat{\vy}_{\geq\msc}\big)\bigg|\stodom\frac{1}{d^{\tau}}
\]
where $\tilde{\tevec}_{k}=\kcoeff k\delta_{k}\tevec_{k}$, $\hat{\vy}_{\geq\msc}=\sum_{k=\msc}^{L}\tevec_{k}+\noise$
and $\widetilde{\rsmtx}_{\msc}=(\tilde{\lambda}\mI+\kmtx_{\msc})^{-1}$,
with $\tilde{\lambda}=\lambda+\sum_{k>\msc}\kcoeff k$.
\end{prop}

\begin{proof}
We make the following decomposition of $\frac{1}{n}\tilde{\tevec}^{\T}\rsmtx\hat{\vy}$:%
\begin{align*}
\frac{1}{n}\tilde{\tevec}^{\T}\rsmtx\hat{\vy}= & \sum_{k<\msc}\tcoeff k^{2}+\frac{1}{n}\tilde{\tevec}_{\msc}^{\T}\widetilde{\rsmtx}_{\msc}\hat{\obser}_{\geq\msc}\\
 & +\frac{1}{n}\tilde{\tevec}_{>\msc}^{\T}\rsmtx\hat{\vy}+(\frac{1}{n}\tilde{\tevec}_{<\msc}^{\T}\rsmtx{\tevec}_{<\msc}-\sum_{k<\msc}\tcoeff k^{2})\\
 & +(\frac{1}{n}\tilde{\tevec}_{<\msc}^{\T}\rsmtx\hat{\obser}_{\geq\msc}+\frac{1}{n}\tilde{\tevec}_{\msc}^{\T}\rsmtx\tevec_{<\msc})+\frac{1}{n}\tilde{\tevec}_{\msc}^{\T}(\rsmtx-\widetilde{\rsmtx}_{\msc})\hat{\obser}_{\geq\msc}.
\end{align*}
The last four terms correspond to the approximation error. In Lemma
\ref{lem:test_err_cross_approx_step2}-Lemma \ref{lem:test_err_cross_approx_step4},
we show they all decay to 0 with rate no slower than $\frac{1}{d^{\tau}}$
for some $\tau>0$, so we prove the desired result.
\end{proof}
By Proposition \ref{prop:test_error_cross_truncation}, it boils down
to compute $\frac{1}{n}\tilde{\tevec}_{\msc}^{\T}\widetilde{\rsmtx}_{\msc}\hat{\vy}_{\geq\msc}$.
To this end, we are back to the setting in Sec. \ref{subsec:training_err_special_case_proof}.
Similar as obtaining (\ref{eq:training_error_2}), we can get
\begin{align*}
\frac{1}{n}\tilde{\tevec}_{\msc}^{\T}\widetilde{\rsmtx}_{\msc}\hat{\vy}_{\geq\msc} & =\frac{1}{n}\kcoeff{\msc}\delta_{\msc}\tevec_{\msc}^{\T}\widetilde{\rsmtx}_{\msc}(\sum_{k=\msc}^{L}\tevec_{k}+\noise)\\
 & \pconv\tfrac{\kcoeff{\msc}\delta_{\msc}\tcoeff{\msc}^{2}\widehat{R}_{\msc}}{1+\kcoeff{\msc}\delta_{\msc}\widehat{R}_{\msc}}
\end{align*}
where $\widehat{R}_{\msc}=\frac{\tr[\tilde{\lambda}\mI+\kcoeff{\msc}\widetilde{\shmtx}_{\msc}(\mV)\widetilde{\shmtx}_{\msc}(\mV)^{\T}]^{-1}}{n}$ is defined in a completely analogous way
as in (\ref{eq:training_error_2}) and the only difference is that
$\lambda$ is replaced by $\tilde{\lambda}$ here. Finally, by the same proof of Theorem 1 in \cite{Lu2022Equi}, we can get $\widehat{R}_{\msc}\pconv R_{\star}$.
Therefore, for any $L\geq\msc$,
\begin{equation}
\frac{1}{n}\tilde{\tevec}^{\T}\rsmtx\hat{\vy}\pconv\sum_{k<\msc}\tcoeff k^{2}+\tfrac{\kcoeff{\msc}\delta_{\msc}\tcoeff{\msc}^{2}R_{\star}}{1+\kcoeff{\msc}\delta_{\msc}R_{\star}}.\label{eq:testerr_crossterm_2_pconv}
\end{equation}

The final step is to show the approximation error $|\frac{1}{n}\tilde{\tevec}^{\T}\rsmtx(\obser-\hat{\vy})|$
can be made arbitrarily small. Indeed, by Lemma \ref{lem:obser_truncation},
for any $\veps>0$, there exists $L\geq\msc$ and $C>0$ such that for
all large enough $d$, $\P(\frac{1}{n}\|\obser-\hat{\obser}\|^{2}>\veps)\leq\frac{C}{n\veps^{2}}$.
Also from (\ref{eq:test_err_cross_approx_step3_2}) in the proof of
Lemma \ref{lem:test_err_cross_approx_step3}, we have $\frac{1}{\sqrt{n}}\|\rsmtx\tilde{\tevec}_{<\msc}\|\stodom1$
and from Lemma \ref{lem:y_norm_bd} with the fact that $\delta_{k}\lesssim1$,
for $k\geq\msc$, we can get $\frac{1}{\sqrt{n}}\|\rsmtx\tilde{\tevec}_{\geq\msc}\|\stodom1.$
These together imply that for any $\veps>0$, there exists $L\geq\msc$
and $C>0$ such that for all large enough $d$,
\begin{equation}
\P(|\frac{1}{n}\tilde{\tevec}^{\T}\rsmtx(\obser-\hat{\vy})|>\veps)\leq\frac{C}{n\veps^{2}}.\label{eq:testerr_crossterm_y_yhat_diffbd}
\end{equation}
Combining (\ref{eq:testerr_crossterm_y_yhat_diffbd}), (\ref{eq:testerr_crossterm_2_pconv})
and (\ref{eq:testerr_crossterm_1}), we get
\begin{equation}
\E_{\text{new}}[y_{\text{new}}\hat{\regressfn}(\invec_{\text{new}})]=\frac{1}{n}\tilde{\tevec}^{\T}\rsmtx\obser\pconv\sum_{k<\msc}\tcoeff k^{2}+\tfrac{\kcoeff{\msc}\delta_{\msc}\tcoeff{\msc}^{2}R_{\star}}{1+\kcoeff{\msc}\delta_{\msc}R_{\star}}.
\end{equation}

\section{\label{sec:Proof-of-hsq}Proof of (\ref{eq:testerr_normsqterm_finalform})}

Recall that $\hat{\regressfn}(\invec_{\text{new}})=f\Big(\tfrac{\invec_{\text{new}}^{\T}\inputmtx^{\T}}{\sqrt{d}}\Big)\rsmtx\vy$
and for any $i,j\in[n]$,
\begin{align*}
\E_{\text{new}}\Big[\kernalfn\big(\tfrac{\vx_{i}^{\T}\vx_{\text{new}}}{\sqrt{d}}\big)\kernalfn\big(\tfrac{\vx_{\text{new}}^{\T}\vx_{j}}{\sqrt{d}}\big)\Big] & =\E_{\text{new}}[\sum_{k=0}^{\infty}\frac{\kcoeff k}{\usdim k}Y_{k}(\vx_{i}^{\T})Y_{k}(\vx_{\text{new}})\times\sum_{k=0}^{\infty}\frac{\kcoeff k}{\usdim k}Y_{k}(\vx_{\text{new}}^{\T})Y_{k}(\vx_{j})]\\
 & =\sum_{k=0}^{\infty}\tfrac{\kcoeff k^{2}}{\usdim k^{2}}Y_{k}(\vx_{i})^{\T}Y_{k}(\vx_{j}).
\end{align*}
Therefore, we can get
\begin{align}
\E_{\text{new}}\hat{\regressfn}(\invec_{\text{new}})^{2} & =\E_{\text{new}}\bigg[\vy^{\T}\rsmtx\kernalfn\Big(\tfrac{\inputmtx\invec_{\text{new}}}{\sqrt{d}}\Big)\kernalfn\Big(\tfrac{\invec_{\text{new}}^{\T}\inputmtx^{\T}}{\sqrt{d}}\Big)\rsmtx\vy\bigg]\nonumber \\
 & =\frac{1}{n}\vy^{\T}\rsmtx\Big(\sum_{k=0}^{\infty}\kcoeff k\delta_{k}\kmtx_{k}\Big)\rsmtx\vy,\label{eq:testerr_normsqterm_1}
\end{align}
where $\kmtx_{k}=\tfrac{\kcoeff k}{\usdim k}\shmtx_{k}\shmtx_{k}^{\T}$
and $\delta_{k}=n/\usdim k$.

Similar as before, the first step is to apply a truncation argument
to simplify the right hand-side of (\ref{eq:testerr_normsqterm_1}).
Specifically, by Lemma \ref{lem:testerr_normsqterm_lowdegree_truncation},
we have
\begin{equation}
\Big|\frac{1}{n}\vy^{\T}\rsmtx\Big(\sum_{k=0}^{\infty}\kcoeff k\delta_{k}\kmtx_{k}\Big)\rsmtx\vy-[\sum_{k=0}^{\msc-1}\tcoeff k^{2}+\frac{1}{n}\vy^{\T}\rsmtx\big(\kcoeff{\msc}\delta_{\msc}\kmtx_{\msc})\rsmtx\vy]\Big|\stodom\frac{1}{\sqrt{d}}.\label{eq:testerr_normsqterm_1_approx_1_1}
\end{equation}
This shows that all the high-degree components of kernel function
can be neglected. while the low-degree components can be equivalent
to a single constant $\sum_{k=0}^{\msc-1}\tcoeff k^{2}$. Based on
(\ref{eq:testerr_normsqterm_1_approx_1_1}), the remaining task is
to compute $\frac{1}{n}\vy^{\T}\rsmtx(\kcoeff{\msc}\delta_{\msc}\kmtx_{\msc})\rsmtx\vy$.
To this end, we introduce the following function:
\begin{equation}
\Gamma_{d}(\epsilon)=\frac{1}{n}\obser^{\T}\big(\lambda\mI+\mK+\epsilon\kcoeff{\msc}\delta_{\msc}\kmtx_{\msc}\big)^{-1}\vy,\label{eq:testerr_normsqterm_1_Gamma_d}
\end{equation}
where $\epsilon\in[-\epsilon_{0},\epsilon_{0}]$, with $\epsilon_{0}=\frac{1}{2}(\kcoeff{\msc}\delta_{\msc})^{-1}$.
It can be directly checked that
\begin{equation}
\Gamma_{d}'(0)=-\frac{1}{n}\vy^{\T}\rsmtx(\kcoeff{\msc}\delta_{\msc}\kmtx_{\msc})\rsmtx\vy,\label{eq:testerr_normsqterm_1_Gammad_deri0}
\end{equation}
so in order to compute $\lim_{d\to\infty}\frac{1}{n}\vy^{\T}\rsmtx(\kcoeff{\msc}\delta_{\msc}\kmtx_{\msc})\rsmtx\vy$,
it suffices to compute $\lim_{d\to\infty}\Gamma_{d}'(0)$. Also $\Gamma_{d}(\epsilon)$
is well-defined for $\epsilon\in[-\epsilon_{0},\epsilon_{0}]$, because
$\mK+\epsilon\kcoeff{\msc}\delta_{\msc}\kmtx_{\msc}\succeq\boldsymbol{0}$
for any $\epsilon\in[-\epsilon_{0},\epsilon_{0}]$. In addition, since
$\kcoeff{\msc}\in(0,\infty)$ and $\delta_{\msc}\in(0,\infty)$ under our
main assumptions, we have $0<\epsilon_{0}\leq C_{0}$, for some $C_{0}$
not depending on $d$.

From (\ref{eq:testerr_normsqterm_1_Gamma_d}) we can find that
$\Gamma_{d}(\epsilon)$ takes the same form as $\mathcal{E}_{\text{train}}$
{[}c.f. (\ref{eq:training_error}){]}. Therefore, repeating the same
procedure in proving (\ref{eq:main_results_trainerr}), we can obtain
\begin{equation}
\Gamma_{d}(\epsilon)\pconv\frac{\tcoeff{\msc}^{2}R(\tilde{\lambda},\epsilon)}{1+\delta_{\msc}\kcoeff{\msc}(1+\epsilon\kcoeff{\msc}\delta_{\msc})R(\tilde{\lambda},\epsilon)}+\Big(\varnoise^{2}+\sum_{k>\msc}\tcoeff k^{2}\Big)R(\tilde{\lambda},\epsilon):=\Gamma(\epsilon)\label{eq:testerr_normsqterm_1_Gamma_lim}
\end{equation}
where $R(\tilde{\lambda},\epsilon)$ is the unique non-negative solution
of
\begin{equation}
\frac{1}{R}=\tilde{\lambda}+\frac{\mu_{\msc}(1+\epsilon\kcoeff{\msc}\delta_{\msc})}{1+\delta_{\msc}\mu_{\msc}(1+\epsilon\kcoeff{\msc}\delta_{\msc})R}.\label{eq:testerr_normsqterm_1_quadratic_equation}
\end{equation}
Note that (\ref{eq:testerr_normsqterm_1_quadratic_equation}) is a
quadratic equation. For any fixed $\tilde{\lambda}>0$, it allows
for an explicit solution: $R(\tilde{\lambda},\epsilon)=\mathcal{R}(\tilde{\lambda};\mu_{\msc}(1+\epsilon\kcoeff{\msc}\delta_{\msc}),\delta_{\msc})$,
where $\mathcal{R}(\lambda;\mu,\delta)$ is defined in (\ref{eq:main_results_Stieltjes_1}).
It can be directly verified that $R(\tilde{\lambda},\epsilon)$ is
smooth with respect to $\epsilon$. In particular, we can compute
its partial derivative with respect to $\epsilon$ at $0$:
\begin{equation}
R'_{\epsilon}(\tilde{\lambda},0)=-\frac{1}{\equisamp-1},\label{eq:testerr_normsqterm_1_R_deri}
\end{equation}
where $\equisamp=\frac{(1+\kcoeff{\msc}\delta_{\msc}R_{\star})^{2}}{\delta_{\msc}\kcoeff{\msc}^{2}R_{\star}^{2}}$. For notational
simplicity, in the following, we denote by $R'_{\epsilon=0}:=R'_{\epsilon}(\tilde{\lambda},0)$.
On the other hand, since $R(\tilde{\lambda},\epsilon)$ is smooth,
$\Gamma(\epsilon)$ in (\ref{eq:testerr_normsqterm_1_Gamma_lim})
is also smooth with respect to $\epsilon$ and
\begin{equation}
\Gamma'(0)=\frac{\tcoeff{\msc}^{2}R'_{\epsilon=0}}{[1+\delta_{\msc}\kcoeff{\msc}R_{\star}]^{2}}-\frac{(\delta_{\msc}\kcoeff{\msc}\tcoeff{\msc}R_{\star})^{2}}{(1+\delta_{\msc}\kcoeff{\msc}R_{\star})^{2}}+\Big(\varnoise^{2}+\sum_{k>\msc}\tcoeff k^{2}\Big)R'_{\epsilon=0}.\label{eq:testerr_normsqterm_1_Gammaderi}
\end{equation}

The next step is to show that $\Gamma_{d}'(0)\pconv\Gamma'(0)$. Denote
$\rsmtx(\epsilon):=[\lambda\mI+\mK+\epsilon\big(\kcoeff \msc\delta_{\msc}\kmtx_{\msc}\big)]^{-1}$.
It can be directly verified that
\[
\Gamma_{d}''(\epsilon)=\frac{2}{n}\vy^{\T}\rsmtx(\epsilon)\cdot(\kcoeff{\msc}\delta_{\msc}\kmtx_{\msc})\cdot\rsmtx(\epsilon)\cdot(\kcoeff{\msc}\delta_{\msc}\kmtx_{\msc})\cdot\rsmtx(\epsilon)\vy\geq0,
\]
so $\Gamma_{d}(\epsilon)$ is a convex function. Therefore, for any
$\epsilon\in(-\epsilon_{0},\epsilon_{0})$ and $\varsigma\in[0,(\epsilon_{0}-|\epsilon|)/2]$,
\begin{equation}
\frac{\Gamma_{d}(\epsilon-\varsigma)-\Gamma_{d}(\epsilon)}{-\varsigma}-\Gamma'(\epsilon)\leq\Gamma_{d}'(\epsilon)-\Gamma'(\epsilon)\leq\frac{\Gamma_{d}(\epsilon+\varsigma)-\Gamma_{d}(\epsilon)}{\varsigma}-\Gamma'(\epsilon).\label{eq:testerr_normsqterm_1_convex_2}
\end{equation}
On the other hand, since $\Gamma(\epsilon)$ is smooth, for any $\varrho>0$,
there exists $c>0$ such that for any $\varsigma\in[-c,c]$, it holds
that $|\epsilon|+|\varsigma|<\epsilon_{0}$ and
\begin{equation}
\Big|\Gamma'(\epsilon)-\frac{\Gamma(\epsilon+\varsigma)-\Gamma(\epsilon)}{\varsigma}\Big|\leq\frac{\varrho}{2}.\label{eq:testerr_normsqterm_1_smooth_2}
\end{equation}
Combining (\ref{eq:testerr_normsqterm_1_convex_2}) and (\ref{eq:testerr_normsqterm_1_smooth_2}),
we have for any $\varsigma\in[-c,c]$,
\begin{equation}
\begin{aligned}
\frac{[\Gamma_{d}(\epsilon-\varsigma)-\Gamma(\epsilon-\varsigma)]-[\Gamma_{d}(\epsilon)-\Gamma(\epsilon)]}{-\varsigma}-\frac{\varrho}{2}&\leq\Gamma_{d}'(\epsilon)-\Gamma'(\epsilon)\\
&\leq\frac{[\Gamma_{d}(\epsilon+\varsigma)-\Gamma(\epsilon+\varsigma)]-[\Gamma_{d}(\epsilon)-\Gamma(\epsilon)]}{\varsigma}+\frac{\varrho}{2}.
\end{aligned}\label{eq:testerr_normsqterm_1_convex_3}
\end{equation}
Since $\Gamma_{d}(\epsilon)\to\Gamma(\epsilon)$ for any $\epsilon\in(-\epsilon_{0},\epsilon_{0})$,
taking $d\to\infty$ in (\ref{eq:testerr_normsqterm_1_convex_3}),
we get
\[
\P(|\Gamma_{d}'(\epsilon)-\Gamma'(\epsilon)|>\varrho)\to0.
\]
Moreover, here $\varrho>0$ can be made arbitrarily small, so setting $\epsilon=0$, we prove
\begin{equation}
\Gamma_{d}'(0)\pconv\Gamma'(0).\label{eq:testerr_normsqterm_1_pconv_Gammadderi}
\end{equation}

Finally, putting (\ref{eq:testerr_normsqterm_1}), (\ref{eq:testerr_normsqterm_1_approx_1_1}),
(\ref{eq:testerr_normsqterm_1_Gammad_deri0}), (\ref{eq:testerr_normsqterm_1_Gammaderi})
and (\ref{eq:testerr_normsqterm_1_pconv_Gammadderi}) together, we
reach at:
\[
\E_{\text{new}}\hat{\regressfn}(\invec_{\text{new}})^{2}\pconv\sum_{k<\msc}\tcoeff k^{2}+\frac{(\delta_{\msc}\kcoeff{\msc}\tcoeff{\msc}R_{\star})^{2}}{(1+\delta_{\msc}\kcoeff{\msc}R_{\star})^{2}}-\frac{\tcoeff{\msc}^{2}R'_{\epsilon=0}}{(1+\delta_{\msc}\kcoeff{\msc}R_{\star})^{2}}-\Big(\varnoise^{2}+\sum_{k>\msc}\tcoeff k^{2}\Big)R'_{\epsilon=0},
\]
where $R'_{\epsilon=0}=-\frac{1}{\equisamp-1}$, which is defined
in (\ref{eq:testerr_normsqterm_1_R_deri}).

\section{Auxiliary Results for Analyzing Test Error}

\begin{lem}
\label{lem:test_err_cross_approx_step2}It holds that
\end{lem}

\begin{equation}
\left\Vert \frac{1}{\sqrt{n}}\tilde{\tevec}_{>\msc}\right\Vert \stodom\frac{1}{d}\label{eq:test_err_cross_approx_step2_finalform}
\end{equation}
where $\tilde{\tevec}_{>\msc}:=\sum_{k>\msc}\tilde{\tevec}_{k}$,
with $\tilde{\tevec}_{k}=\kcoeff k\delta_{k}\tevec_{k}$.
\begin{proof}
Since $\E\tevec_{k}^{\T}\tevec_{\ell}=\indicatorfn_{k=\ell}\tcoeff k^{2}$,
we have
\begin{align*}
\E\|\tilde{\tevec}_{>\msc}\|^{2} & =\E\|\sum_{k>\msc}\kcoeff k\delta_{k}\tevec_{k}\|^{2}\\
 & =n\sum_{k>\msc}\kcoeff k^{2}\delta_{k}^{2}\tcoeff k^{2}\\
 & \leq n\delta_{\msc+1}^{2}\sum_{k>\msc}\kcoeff k^{2}\tcoeff k^{2}\\
 & \lesssim\frac{n}{d^{2}},
\end{align*}
where the last step follows from $\sum_{k=0}^{\infty}\kcoeff k<\infty,$
$\sum_{k=0}^{\infty}\tcoeff k^{2}<\infty$. Therefore, we have $\|\tilde{\tevec}_{>\msc}\|\stodom\frac{\sqrt{n}}{d}$,
which is the desired result.
\end{proof}
\begin{lem}
\label{lem:test_err_cross_approx_step0_finalform}It holds that
\begin{equation}
\Big|\frac{1}{n}\tilde{\tevec}_{<\msc}^{\T}\rsmtx\tevec_{<\msc}-\sum_{k<\msc}\tcoeff k^{2}\Big|\stodom\frac{1}{d}\label{eq:test_err_cross_approx_step0_finalform}
\end{equation}
where $\tilde{\tevec}_{<\msc}=\sum_{k<\msc}\kcoeff k\delta_{k}\tevec_{k}$
.
\end{lem}

\begin{proof}
From Lemma \ref{lem:muted_lowerdegree},
\begin{align*}
\frac{1}{n}\mD_{<\msc}^{2}\shmtx_{<\msc}^{\T}(\lambda\mI+\kmtx)^{-1}\shmtx_{<\msc} & =(\mD_{<\msc}^{-2}+\frac{1}{n}\shmtx_{<\msc}^{\T}\rsmtx_{\geq\msc}\shmtx_{<\msc})^{-1}(\frac{1}{n}\shmtx_{<\msc}^{\T}\rsmtx_{\geq\msc}\shmtx_{<\msc})\\
 & =\mI-(\mD_{<\msc}^{-2}+\frac{1}{n}\shmtx_{<\msc}^{\T}\rsmtx_{\geq\msc}\shmtx_{<\msc})^{-1}\mD_{<\msc}^{-2}.
\end{align*}
Therefore,
\begin{align}
\frac{1}{n}\tilde{\tevec}_{<\msc}^{\T}\rsmtx\tevec_{<\msc}= & \frac{1}{n}\widetilde{\shmtx}_{<\msc}(\vxi^{\T})\mD_{<\msc}^{2}\shmtx_{<\msc}^{\T}(\lambda\mI+\kmtx)^{-1}\shmtx_{<\msc}\widetilde{\shmtx}_{<\msc}(\vxi)\nonumber \\
= & \sum_{k<\msc}\tcoeff k^{2}-\widetilde{\shmtx}_{<\msc}(\vxi^{\T})(\mD_{<\msc}^{-2}+\frac{1}{n}\shmtx_{<\msc}^{\T}\rsmtx_{\geq\msc}\shmtx_{<\msc})^{-1}\mD_{<\msc}^{-2}\widetilde{\shmtx}_{<\msc}(\vxi).\label{eq:test_err_cross_approx_step0_1}
\end{align}
By Lemma \ref{lem:spectral_norm_1}, we have $\|(\mD_{<\msc}^{-2}+\frac{1}{n}\shmtx_{<\msc}^{\T}\rsmtx_{\geq\msc}\shmtx_{<\msc})^{-1}\|\stodom1$.
Also we have $\|\mD_{<\msc}^{-2}\|\lesssim\frac{1}{d}$. Therefore,
\begin{equation}
|\widetilde{\shmtx}_{<\msc}(\vxi^{\T})(\mD_{<\msc}^{-2}+\frac{1}{n}\shmtx_{<\msc}^{\T}\rsmtx_{\geq\msc}\shmtx_{<\msc})^{-1}\mD_{<\msc}^{-2}\widetilde{\shmtx}_{<\msc}(\vxi)|\stodom\frac{1}{d}\sum_{k<\msc}\tcoeff k^{2}.\label{eq:test_err_cross_approx_step0_2}
\end{equation}
Combining (\ref{eq:test_err_cross_approx_step0_1}) and (\ref{eq:test_err_cross_approx_step0_2})
together with $\sum_{k<\msc}\tcoeff k^{2}\lesssim1$ (which follows
from Assumption (A.4)), we reach at (\ref{eq:test_err_cross_approx_step0_finalform}).
\end{proof}
%

\begin{lem}
\label{lem:test_err_cross_approx_step3}For any $L\geq\msc$, it holds
that
\begin{equation}
\frac{1}{n}|\tilde{\tevec}_{<\msc}^{\T}\rsmtx\hat{\obser}_{\geq\msc}|\stodom\frac{1}{\sqrt{d}}\label{eq:test_err_cross_approx_step3_finalform_1}
\end{equation}
and
\begin{equation}
\frac{1}{n}|\tilde{\tevec}_{\msc}^{\T}\rsmtx\tevec_{<\msc}|\stodom\frac{1}{\sqrt{d}}\label{eq:test_err_cross_approx_step3_finalform_2}
\end{equation}
where $\tilde{\tevec}_{k}=\kcoeff k\delta_{k}\tevec_{k}$, $\hat{\obser}_{\geq\msc}=\sum_{k=\msc}^{L}\tevec_{k}+\noise$
and $\rsmtx=(\lambda\mI+\kmtx)^{-1}$.
\end{lem}

\begin{proof}
In the following, we present the proof for (\ref{eq:test_err_cross_approx_step3_finalform_1}).
The proof for (\ref{eq:test_err_cross_approx_step3_finalform_2})
is analogous and will be omitted for brevity.

To prove (\ref{eq:test_err_cross_approx_step3_finalform_1}), it suffices
to show for any fixed $k\geq\msc$,
\begin{equation}
\frac{1}{n}|\tilde{\tevec}_{<\msc}^{\T}\rsmtx\tevec_{k}|\stodom\frac{1}{\sqrt{d}}\label{eq:test_cross_term_bd_1}
\end{equation}
and
\begin{equation}
\frac{1}{n}|\tilde{\tevec}_{<\msc}^{\T}\rsmtx\vz|\stodom\frac{1}{\sqrt{d}}.\label{eq:test_cross_term_bd_2}
\end{equation}
The second bound (\ref{eq:test_cross_term_bd_2}) is easy to obtain.
By the independence between $\noise$ and $\tilde{\tevec}_{<\msc}^{\T}\rsmtx$,
we have
\begin{align}
\E_{\noise}(\tilde{\tevec}_{<\msc}^{\T}\rsmtx\vz)^{2} & =\|\rsmtx\tilde{\tevec}_{<\msc}\|^{2}.\label{eq:test_err_cross_approx_step3_1}
\end{align}
Since $\tilde{\tevec}_{<\msc}=\shmtx_{<\msc}\mD_{<\msc}^{2}\widetilde{\shmtx}_{<\msc}(\vxi)$,
we have
\begin{align}
\frac{1}{\sqrt{n}}\|\rsmtx\tilde{\tevec}_{<\msc}\| & =\frac{1}{\sqrt{n}}\|(\lambda\mI+\kmtx)^{-1}\shmtx_{<\msc}\mD_{<\msc}^{2}\widetilde{\shmtx}_{<\msc}(\vxi)\|\nonumber \\
 & =\frac{1}{\sqrt{n}}\|\rsmtx_{\geq\msc}\shmtx_{<\msc}(\mD_{<\msc}^{-2}+\frac{1}{n}\shmtx_{<\msc}^{\T}\rsmtx_{\geq\msc}\shmtx_{<\msc})^{-1}\widetilde{\shmtx}_{<\msc}(\vxi)\|\nonumber \\
 & \stodom1,\label{eq:test_err_cross_approx_step3_2}
\end{align}
where the second step follows from Lemma \ref{lem:muted_lowerdegree}
and the last step follows from Lemma \ref{lem:spectral_norm_1} and
the fact that $\|\widetilde{\shmtx}_{<\msc}(\vxi)\|\lesssim1$. Then
combining (\ref{eq:test_err_cross_approx_step3_1}) and (\ref{eq:test_err_cross_approx_step3_2}),
we have
\begin{align*}
\frac{1}{n}|\tilde{\tevec}_{<\msc}^{\T}\rsmtx_{\leq\msc}\vz| & \stodom\frac{1}{n}\|\rsmtx\tilde{\tevec}_{<\msc}\|\\
 & \stodom\frac{1}{\sqrt{n}}
\end{align*}
which implies (\ref{eq:test_cross_term_bd_2}), since $d\lesssim n$.

Now we analyze (\ref{eq:test_cross_term_bd_1}). Note that (\ref{eq:test_cross_term_bd_1})
can be equivalently expressed as:
\begin{align}
\frac{1}{n}|\widetilde{\shmtx}_{<\msc}(\sgl)^{\T}\mdmtx_{<\msc}^{2}\shmtx_{<\msc}^{\T}\rsmtx\shmtx_{k}\widetilde{\shmtx}_{k}(\sgl)| & \stodom\frac{1}{\sqrt{d}}.\label{eq:test_cross_term_bd_11}
\end{align}
Here $\widetilde{\shmtx}_{k}(\sgl)$ and $\widetilde{\shmtx}_{<\msc}(\sgl)$
are independent of the sandwiched matrix $\mdmtx_{<\msc}^{2}\shmtx_{<\msc}^{\T}\rsmtx\shmtx_{k}$.
However, the entries of $\widetilde{\shmtx}_{k}(\sgl)$ and $\widetilde{\shmtx}_{<\msc}(\sgl)$
are not independent, although they are uncorrelated. The key to prove
(\ref{eq:test_cross_term_bd_11}) is to handle the weak correlation
among different entries of $\widetilde{\shmtx}_{k}(\sgl)$ and $\widetilde{\shmtx}_{<\msc}(\sgl)$.
As we know, there are different choices of the orthonormal basis $\{\shmtx_{k}(\sgl)\}_{k\geq0}$,
which are all equivalent. To ease our analysis, we will work with
the following specific type of $\shmtx_{k}(\sgl)$:
\begin{equation}
\shmtx_{k}(\sgl)=[\shmtx_{k,\symset}(\sgl)^{\T},\shmtx_{k,\backslash\symset}(\sgl)^{\T}]^{\T},\label{eq:SH_symmetric}
\end{equation}
where $\shmtx_{k,\symset}(\sgl^{\T})$ is a ${d \choose k}$-dimensional
vector and its entries are $\{\upsilon_{k}\sgli_{i_{1}}\sgli_{i_{2}}\cdots\sgli_{i_{k}}\}_{1\leq i_{1}<i_{2}<\cdots<i_{k}\leq d}$.
Here, the normalizing constant $\upsilon_{k}=[\E(\sgli_{1}^{2}\sgli_{2}^{2}\cdots\sgli_{k}^{2})]^{-\frac{1}{2}}$,
with $\sgl\sim\spmeasure{d-1}$. It can be easily checked that when
$\sgl\sim\spmeasure{d-1}$ and $\{i_{1},i_{2},\cdots,i_{k}\}$, $\{j_{1},j_{2},\cdots,j_{k}\}$
are two different sets of indices, it holds that $\E(\sgli_{i_{1}}\sgli_{i_{2}}\cdots\sgli_{i_{k}}\cdot\sgli_{j_{1}}\sgli_{j_{2}}\cdots\sgli_{j_{k}})=0$.
To see this, we first represent $\sgl$ by: $\sgl=\sqrt{d}\frac{\vtheta}{\|\vtheta\|}$,
where $\vtheta\sim\mathcal{N}(\boldsymbol{0},\mI_{d})$. Since $\{i_{1},i_{2},\cdots,i_{k}\}\neq\{j_{1},j_{2},\cdots,j_{k}\}$,
there exists at least one $u\in\{j_{1},j_{2},\cdots,j_{k}\}$ such
that $u\notin\{i_{1},i_{2},\cdots,i_{k}\}$. Then we have:
\begin{align*}
\E[\sgli_{i_{1}}\sgli_{i_{2}}\cdots\sgli_{i_{k}}\cdot\sgli_{j_{1}}\sgli_{j_{2}}\cdots\sgli_{j_{k}}] & =\E\E[\sgli_{i_{1}}\sgli_{i_{2}}\cdots\sgli_{i_{k}}\cdot\sgli_{j_{1}}\sgli_{j_{2}}\cdots\sgli_{j_{k}}\mid\{\theta_{j}\}_{j\neq u}]\\
 & =\E\E\Big[\tfrac{\varphi\big(\{\theta_{j}\}_{j\neq u}\big)\cdot\theta_{u}}{(\sum_{j\neq u}\theta_{j}^{2}+\theta_{u}^{2})^{k}}\mid\{\theta_{j}\}_{j\neq u}\Big]\\
 & =0,
\end{align*}
where $\varphi$ is a function of $\{\theta_{j}\}_{j\neq u}$ and
the last step follows from the fact that $\theta_{u}\sim\mathcal{N}(0,1)$
and $\theta\mapsto\tfrac{\varphi\big(\{\theta_{j}\}_{j\neq u}\big)\cdot\theta}{(\sum_{j\neq u}\theta_{j}^{2}+\theta^{2})^{k}}$
is an odd function, for any fixed $\{\theta_{j}\}_{j\neq u}$. Therefore,
we have $\E[\shmtx_{k,\symset}(\sgl)\shmtx_{k,\symset}(\sgl)^{\T}]=\mI$.

In the following, we make use of the following notations: (i) for
any $\mX$, $\shmtx_{<\msc,\symset}(\mX):=[\shmtx_{0,\symset}(\mX),\cdots,\shmtx_{\msc-1,\symset}(\mX)]$
and $\shmtx_{<\msc,\backslash\symset}(\mX):=[\shmtx_{0,\backslash\symset}(\mX),\cdots,\shmtx_{\msc-1,\backslash\symset}(\mX)]$,
(ii) for any $k$, $\vx$, $\widetilde{\shmtx}_{k,\symset}(\vx):=\frac{\tcoeff k}{\sqrt{\usdim k}}\shmtx_{k,\symset}(\vx)$
and $\widetilde{\shmtx}_{k,\backslash\symset}(\vx):=\frac{\tcoeff k}{\sqrt{\usdim k}}\shmtx_{k,\backslash\symset}(\vx)$.
The proof of (\ref{eq:test_cross_term_bd_11}) will be done in two
steps.

(I) The first step is to show the following,
\begin{equation}
\tfrac{1}{\sqrt{n}}\left|\widetilde{\shmtx}_{<\msc}(\sgl^{\T})\mM\shmtx_{k}\widetilde{\shmtx}_{k}(\sgl)-\widetilde{\shmtx}_{<\msc,\symset}(\sgl^{\T})\mM_{\symset}\shmtx_{k,\symset}\widetilde{\shmtx}_{k,\symset}(\sgl)\right|\stodom\frac{1}{\sqrt{d}},\label{eq:test_cross_term_bd_1_step_1_finalform}
\end{equation}
where $\mM=\frac{1}{\sqrt{n}}\mdmtx_{<\msc}^{2}\shmtx_{<\msc}^{\T}\rsmtx$
and $\mM_{\symset}=\frac{1}{\sqrt{n}}\mdmtx_{<\msc,\symset}^{2}\shmtx_{<\msc,\symset}^{\T}\rsmtx$.
Here, $\mdmtx_{<\msc,\symset}$ is the minor of $\mdmtx_{<\msc}$,
corresponding to $\shmtx_{<\msc,\symset}$ (in other words, if $\mathcal{I}\subseteq\{1,2,\ldots,\sum_{k=0}^{\msc-1}\usdim k\}$
is the set of all row indices of $\shmtx_{<\msc,\symset}$ in $\shmtx_{<\msc}$,
then $\mdmtx_{<\msc,\symset}=(\mdmtx_{<\msc,i,j})_{i,j\in\mathcal{I}}$).
In fact, (\ref{eq:test_cross_term_bd_1_step_1_finalform}) states
that after applying a truncation on $\shmtx_{k}$ and $\widetilde{\shmtx}_{k}(\sgl)$,
the resultant approximation error vanishes as $d\to\infty$. Then
the subsequent analysis can be done over $\frac{1}{\sqrt{n}}\widetilde{\shmtx}_{<\msc,\symset}(\sgl^{\T})\mM_{\symset}\shmtx_{k,\symset}\widetilde{\shmtx}_{k,\symset}(\sgl)$,
which is easier to handle due to the symmetric structure of $\widetilde{\shmtx}_{k,\symset}(\sgl)$.

We now prove (\ref{eq:test_cross_term_bd_1_step_1_finalform}). We
have
\begin{align}
 & \tfrac{1}{\sqrt{n}}|\widetilde{\shmtx}_{<\msc}(\sgl^{\T})\mM\shmtx_{k}\widetilde{\shmtx}_{k}(\sgl)-\widetilde{\shmtx}_{<\msc,\symset}(\sgl^{\T})\mM_{\symset}\shmtx_{k,\symset}\widetilde{\shmtx}_{k,\symset}(\sgl)|\nonumber \\
\leq & \tfrac{1}{\sqrt{n}}|\widetilde{\shmtx}_{<\msc}(\sgl^{\T})\mM\shmtx_{k}\widetilde{\shmtx}_{k}(\sgl)-\widetilde{\shmtx}_{<\msc,\symset}(\sgl^{\T})\mM_{\symset}\shmtx_{k}\widetilde{\shmtx}_{k}(\sgl)|\nonumber \\
 & +\tfrac{1}{\sqrt{n}}|\widetilde{\shmtx}_{<\msc,\symset}(\sgl^{\T})\mM_{\symset}\shmtx_{k}\widetilde{\shmtx}_{k}(\sgl)-\widetilde{\shmtx}_{<\msc,\symset}(\sgl^{\T})\mM_{\symset}\shmtx_{k,\symset}\widetilde{\shmtx}_{k,\symset}(\sgl)|\nonumber \\
= & \Big|\widetilde{\shmtx}_{<\msc,\backslash\symset}(\sgl^{\T})\mM_{\backslash\symset}\tfrac{1}{\sqrt{n}}\shmtx_{k}\widetilde{\shmtx}_{k}(\sgl)\Big|+\Big|\widetilde{\shmtx}_{<\msc,\symset}(\sgl^{\T})\mM_{\symset}\tfrac{1}{\sqrt{n}}\shmtx_{k,\backslash\symset}\widetilde{\shmtx}_{k,\backslash\symset}(\sgl)\Big|.\label{eq:test_cross_term_bd_1_step_1}
\end{align}
Here, $\mM_{\backslash\symset}$ is defined via the same way as $\mM_{\symset}$.
Next we show both terms on the right-hand side of (\ref{eq:test_cross_term_bd_1_step_1})
vanish as $d\to\infty$. First, for any $k$,
\begin{align*}
\E\|\shmtx_{k,\backslash\symset}(\sgl)\|^{2} & =\E\tr[\shmtx_{k,\backslash\symset}(\sgl)\shmtx_{k,\backslash\symset}(\sgl^{\T})]\\
 & =\usdim k-{d \choose k}\\
 & \lesssim\usdim{k-1},
\end{align*}
where we have used the fact that $\usdim k={d \choose k}[1+\mathcal{O}(1/d)]$.
Therefore, by Chebyshev's inequality,
\begin{align}
\|\widetilde{\shmtx}_{k,\backslash\symset}(\sgl)\| & =\frac{\tcoeff k}{\sqrt{\usdim k}}\|\shmtx_{k,\backslash\symset}(\sgl)\|\nonumber \\
 & \lesssim d^{-1/2}.\label{eq:test_cross_term_bd_1_step_1_bd_1}
\end{align}
Similarly, for any $k$, we can show $\|\frac{1}{\sqrt{n}}\shmtx_{k,\backslash\symset}\widetilde{\shmtx}_{k,\backslash\symset}(\sgl)\|\lesssim d^{-1/2}$
as follows. By the independence between $\shmtx_{k,\backslash\symset}$
and $\shmtx_{k,\backslash\symset}(\sgl)$, we have
\begin{align*}
\E\|\shmtx_{k,\backslash\symset}\shmtx_{k,\backslash\symset}(\sgl)\|^{2} & =\E\tr[\shmtx_{k,\backslash\symset}\shmtx_{k,\backslash\symset}(\sgl)\shmtx_{k,\backslash\symset}(\sgl^{\T})\shmtx_{k,\backslash\symset}^{\T}]\\
 & =\E\tr[\shmtx_{k,\backslash\symset}\shmtx_{k,\backslash\symset}^{\T}]\\
 & =n\Big[\usdim k-{d \choose k}\Big].
\end{align*}
Therefore, similar as (\ref{eq:test_cross_term_bd_1_step_1_bd_1})
we can get
\begin{equation}
\|\frac{1}{\sqrt{n}}\shmtx_{k,\backslash\symset}\widetilde{\shmtx}_{k,\backslash\symset}(\sgl)\|\lesssim d^{-1/2}.\label{eq:test_cross_term_bd_1_step_1_bd_2}
\end{equation}
On the other hand, by Lemma \ref{lem:muted_lowerdegree} we have
\begin{align*}
\mM= & \frac{1}{\sqrt{n}}\mdmtx_{<\msc}^{2}\shmtx_{<\msc}^{\T}\rsmtx\\
= & (\mD_{<\msc}^{-2}+\frac{1}{n}\shmtx_{<\msc}^{\T}\rsmtx_{\geq\msc}\shmtx_{<\msc})^{-1}(\frac{1}{\sqrt{n}}\shmtx_{<\msc}^{\T}\rsmtx_{\geq\msc})
\end{align*}
and by Lemma \ref{lem:spectral_norm_1}, we have $\|\mM\|\stodom1$.
Since $\mM_{\backslash\symset}$ and $\mM_{\symset}$ are both sub-matrices
of $\mM$, we get
\begin{equation}
\|\mM_{\symset}\|,\|\mM_{\backslash\symset}\|\leq\|\mM\|\stodom1.\label{eq:test_cross_term_bd_1_step_1_bd_3}
\end{equation}
Now substituting (\ref{eq:test_cross_term_bd_1_step_1_bd_1}), (\ref{eq:test_cross_term_bd_1_step_1_bd_2}),
(\ref{eq:test_cross_term_bd_1_step_1_bd_3}) back to (\ref{eq:test_cross_term_bd_1_step_1})
and using Lemma \ref{lem:y_norm_bd}, we reach at (\ref{eq:test_cross_term_bd_1_step_1_finalform}).

(II) In the second step, we are going to show
\begin{equation}
\frac{1}{\sqrt{n}}|\widetilde{\shmtx}_{<\msc,\symset}(\sgl^{\T})\mM_{\symset}\shmtx_{k,\symset}\widetilde{\shmtx}_{k,\symset}(\sgl)|\stodom\sqrt{\frac{1}{d}}.\label{eq:test_cross_term_bd_1_step_2_finalform}
\end{equation}
By Lemma \ref{lem:samplecov_spectral_norm} and Lemma \ref{lem:spectral_norm_1}
we have
\begin{align*}
\|\frac{1}{\sqrt{n}}\mM\shmtx_{k}\| & =\|(\mD_{<\msc}^{-2}+\frac{1}{n}\shmtx_{<\msc}^{\T}\rsmtx_{\geq\msc}\shmtx_{<\msc})^{-1}(\frac{1}{\sqrt{n}}\shmtx_{<\msc}^{\T}\rsmtx_{\geq\msc}\frac{1}{\sqrt{n}}\shmtx_{k})\|\\
 & \stodom\sqrt{\frac{\usdim k}{n}}.
\end{align*}
Since $\mM_{\symset}\shmtx_{k,\symset}$ is a sub-matrix of $\mM\shmtx_{k}$, we get $\|\frac{1}{\sqrt{n}}\mM_{\symset}\shmtx_{k,\symset}\|\stodom\sqrt{\frac{\usdim k}{n}}.$
Now we have for any $k\geq\msc$,
\begin{align*}
\frac{1}{\sqrt{n}}|\widetilde{\shmtx}_{<\msc,\symset}(\sgl^{\T})\mM_{\symset}\shmtx_{k,\symset}\widetilde{\shmtx}_{k,\symset}(\sgl)|\leq & \sum_{\ell<\msc}\left|\widetilde{\shmtx}_{\ell,\symset}(\sgl^{\T})\frac{1}{\sqrt{n}}\mM_{\symset,\ell}\shmtx_{k,\symset}\widetilde{\shmtx}_{k,\symset}(\sgl)\right|\\
\lesssim & \sum_{\ell<\msc}\frac{1}{\sqrt{\usdim{\ell}\usdim k}}\left|\shmtx_{\ell,\symset}(\sgl^{\T})\frac{1}{\sqrt{n}}\mM_{\symset,\ell}\shmtx_{k,\symset}\shmtx_{k,\symset}(\sgl)\right|\\
\overset{\text{(a)}}{\stodom} & \sum_{\ell<\msc}\frac{1}{\sqrt{\usdim{\ell}\usdim k}}\sqrt{\usdim{\ell}}\|\frac{1}{\sqrt{n}}\mM_{\symset,\ell}\shmtx_{k,\symset}\|_{F}\\
\leq & \sum_{\ell<\msc}\frac{1}{\sqrt{\usdim{\ell}\usdim k}}\sqrt{\usdim{\ell}}\cdot\sqrt{\usdim{\ell}}\|\frac{1}{\sqrt{n}}\mM_{\symset}\shmtx_{k,\symset}\|\\
\overset{\text{(b)}}{\stodom} & \sum_{\ell<\msc}\sqrt{\frac{\usdim{\ell}}{n}}\\
\lesssim & \sqrt{\frac{1}{d}}
\end{align*}
where (a) follows from (\ref{eq:Fkl_sq_bound}) in Lemma \ref{lem:Fkl_sq_bound}
and in (b) we use $\|\frac{1}{\sqrt{n}}\mM_{\symset}\shmtx_{k,\symset}\|\stodom\sqrt{\frac{\usdim k}{n}}$.

Combining (\ref{eq:test_cross_term_bd_1_step_1_finalform}) and (\ref{eq:test_cross_term_bd_1_step_2_finalform}),
we reach at (\ref{eq:test_cross_term_bd_11}). Hence, the proof of
(\ref{eq:test_cross_term_bd_1}) is completed.
\end{proof}
\begin{lem}
\label{lem:test_err_cross_approx_step4}It holds that
\[
\frac{1}{n}|\tilde{\tevec}_{\msc}^{\T}(\rsmtx-\widetilde{\rsmtx}_{\msc})\hat{\obser}_{\geq\msc}|\stodom\frac{1}{d^\tau}
\]
where $\tilde{\tevec}_{\msc}=\kcoeff{\msc}\delta_{\msc}\tevec_{\msc}$,
$\tevec_{\geq\msc}=\sum_{k\geq\msc}\tevec_{k}$, $\rsmtx=(\lambda\mI+\kmtx)^{-1}$
and $\widetilde{\rsmtx}_{\leq\msc}=(\tilde{\lambda}\mI+\kmtx_{\leq\msc})^{-1}$,
with $\tilde{\lambda}=\lambda+\sum_{k>\msc}\kcoeff k$.
\end{lem}

\begin{proof}
We have the following decomposition:
\begin{align*}
\frac{1}{n}|\tilde{\tevec}_{\msc}^{\T}(\rsmtx-\widetilde{\rsmtx}_{\msc})\hat{\obser}_{\geq\msc}|\leq & \frac{1}{n}|\tilde{\tevec}_{\msc}^{\T}(\rsmtx-\widetilde{\rsmtx}_{\leq\msc})\hat{\obser}_{\geq\msc}|+\frac{1}{n}|\tilde{\tevec}_{\msc}^{\T}(\widetilde{\rsmtx}_{\leq\msc}-\widetilde{\rsmtx}_{\msc})\hat{\obser}_{\geq\msc}|.
\end{align*}
Since $\delta_{k}\lesssim1$, when $k\geq\msc$, it follows directly
from Lemma \ref{lem:training_err_approx_step1} and Lemma \ref{lem:y_norm_bd}
that $\frac{1}{n}|\tilde{\tevec}_{\msc}^{\T}(\rsmtx-\widetilde{\rsmtx}_{\leq\msc})\hat{\obser}_{\geq\msc}|\stodom\frac{1}{{d}^\tau}$.
On the other hand, following the same proof of Lemma \ref{lem:training_err_approx_step3},
we get $\frac{1}{n}|\tilde{\tevec}_{\msc}^{\T}(\widetilde{\rsmtx}_{\leq\msc}-\widetilde{\rsmtx}_{\msc})\hat{\obser}_{\geq\msc}|\stodom\frac{1}{\sqrt{d}}$.
Combining these two bounds together, we reach at the desired result.
\end{proof}
\begin{lem}
\label{lem:testerr_normsqterm_lowdegree_truncation}It holds that
\begin{equation}
\Big|\frac{1}{n}\vy^{\T}\rsmtx\big(\sum_{k=0}^{\infty}\kcoeff k\delta_{k}\kmtx_{k}\big)\rsmtx\vy-\big[\sum_{k=0}^{\msc-1}\tcoeff k^{2}+\frac{1}{n}\vy^{\T}\rsmtx\big(\kcoeff{\msc}\delta_{\msc}\kmtx_{\msc})\rsmtx\vy\big]\Big|\stodom\frac{1}{\sqrt{d}}.\label{eq:testerr_normsqterm_1_approx_1}
\end{equation}
\end{lem}

\begin{proof}
We first show some high-degree components can be discarded:
\begin{equation}
\Big|\frac{1}{n}\vy^{\T}\rsmtx\big(\sum_{k>\msc}\kcoeff k\delta_{k}\kmtx_{k}\big)\rsmtx\vy\Big|\stodom\frac{1}{d}.\label{eq:testerr_normsqterm_1_approx_1_bd_1}
\end{equation}
From (\ref{eq:higher_order_norm_bd_Mpapaer_3}) in the proof of Lemma
\ref{lem:training_err_approx_step1}, we have $\|\sum_{k>\msc}\kmtx_{k}\|\stodom1$.
On the other hand, since $\sup_{k\geq0}\kcoeff k\leq\kernalfn(\sqrt{d})\lesssim1$
and $\sup_{k>\msc}\delta_{k}\lesssim\frac{1}{d}$, we get $\|\sum_{k>\msc}\kcoeff k\delta_{k}\kmtx_{k}\|\lesssim\frac{1}{d}$.
Combine this bound with Lemma \ref{lem:y_norm_bd} and the fact that
$\|\rsmtx\|\lesssim1$, we verify (\ref{eq:testerr_normsqterm_1_approx_1_bd_1}).

Next we show
\begin{equation}
\Big|\frac{1}{n}\vy^{\T}\rsmtx\big(\sum_{k=0}^{\msc-1}\kcoeff k\delta_{k}\kmtx_{k}\big)\rsmtx\vy-\sum_{k=0}^{\msc-1}\tcoeff k^{2}\Big|\stodom\frac{1}{\sqrt{d}}.\label{eq:testerr_normsqterm_1_approx_1_bd_2}
\end{equation}
To start, we rewrite $\frac{1}{n}\vy^{\T}\rsmtx(\sum_{k=0}^{\msc-1}\kcoeff k\delta_{k}\kmtx_{k})\rsmtx\vy$
as:
\begin{align}
\frac{1}{n}\vy^{\T}\rsmtx(\sum_{k=0}^{\msc-1}\kcoeff k\delta_{k}\kmtx_{k})\rsmtx\vy= & \Big(\frac{1}{n}\obser^{\T}\rsmtx\mY_{<\msc}\mD_{<\msc}^{2}\Big)\cdot\Big(\underbrace{\frac{1}{n}\mD_{<\msc}^{2}\mY_{<\msc}^{\T}\rsmtx\obser}_{:=\vphi}\Big).\label{eq:testerr_normsqterm_1_approx_2}
\end{align}
The vector $\vphi$ defined on the right-hand side of (\ref{eq:testerr_normsqterm_1_approx_2})
can be decomposed as:
\begin{align}
\vphi & =\frac{1}{n}\mD_{<\msc}^{2}\mY_{<\msc}^{\T}\rsmtx\tevec_{<\msc}+\frac{1}{n}\mD_{<\msc}^{2}\mY_{<\msc}^{\T}\rsmtx\obser_{\geq\msc}\nonumber \\
 & =\widetilde{\shmtx}_{<\msc}(\sgl)-\underbrace{(\mD_{<\msc}^{-2}+\frac{1}{n}\shmtx_{<\msc}^{\T}\rsmtx_{\geq\msc}\shmtx_{<\msc})^{-1}\mD_{<\msc}^{-2}\widetilde{\shmtx}_{<\msc}(\sgl)}_{:=\vphi_{1}}+\underbrace{\frac{1}{n}\mD_{<\msc}^{2}\mY_{<\msc}^{\T}\rsmtx\obser_{\geq\msc}}_{\vphi_{2}},\label{eq:testerr_normsqterm_1_phi}
\end{align}
where $\widetilde{\shmtx}_{k}(\sgl)=\frac{\tcoeff k}{\sqrt{\usdim k}}\shmtx_{k}(\sgl)$,
$\rsmtx_{\geq\msc}=(\lambda\mI+\mK_{\geq\msc})^{-1}$ and in reaching
the second step we use Lemma \ref{lem:muted_lowerdegree}. To proceed,
let us analyze the norms of vectors $\widetilde{\shmtx}_{<\msc}(\sgl)$,
$\vphi_{1}$ and $\vphi_{2}$ on the right-hand side of (\ref{eq:testerr_normsqterm_1_phi}).
First,
\begin{equation}
\|\widetilde{\shmtx}_{<\msc}(\sgl)\|^{2}=\sum_{k=0}^{\msc-1}\tcoeff k^{2}.\label{eq:testerr_normsqterm_1_vecnorm_1}
\end{equation}
Also it is easy to verify $\|\vphi_{1}\|\stodom\frac{1}{d}$ as follows.
From (\ref{eq:higher_order_norm_bd_Mpapaer_3}), we have $\|\kmtx_{>\msc}\|\stodom1$.
Then combine this with Lemma \ref{lem:spectral_norm_1} we have $\|(\mD_{<\msc}^{-2}+\frac{1}{n}\shmtx_{<\msc}^{\T}\rsmtx_{\geq\msc}\shmtx_{<\msc})^{-1}\|\stodom1$.
Also note that $\|\mD_{<\msc}^{-2}\|\lesssim\frac{1}{d}$ and $\|\widetilde{\shmtx}_{<\msc}(\sgl)\|^{2}=\sum_{k=0}^{\msc-1}\tcoeff k^{2}<\infty$,
so
\begin{equation}
\|\vphi_{1}\|\stodom\frac{1}{d}\label{eq:testerr_normsqterm_1_vecnorm_2}
\end{equation}
Lastly, let us show
\begin{equation}
\|\vphi_{2}\|\stodom\frac{1}{\sqrt{d}}\label{eq:testerr_normsqterm_1_vecnorm_3}
\end{equation}
Note that $\vphi_{2}$ can be written as:
\[
\vphi_{2}=\sum_{k\geq\msc}\mPsi\shmtx_{k}\widetilde{\shmtx}_{k}(\sgl)+\mPsi\noise
\]
where $\mPsi=\frac{1}{n}\mD_{<\msc}^{2}\mY_{<\msc}^{\T}\rsmtx$. Then
\begin{equation}
\E_{\sgl,\noise}\|\vphi_{2}\|^{2}=\tr\Big[\mPsi(\sum_{k\geq\msc}\frac{\tcoeff k^{2}}{\usdim k}\shmtx_{k}\shmtx_{k}^{\T})\mPsi^{\T}\Big]+\varnoise^{2}\tr(\mPsi\mPsi^{\T}).\label{eq:testerr_normsqterm_1_Ephi2normsq}
\end{equation}
We first control the spectral norm of $\mPsi$. From Lemma \ref{lem:muted_lowerdegree},
we can get
\[
\mPsi=\frac{1}{n}(\mD_{<\msc}^{-2}+\frac{1}{n}\shmtx_{<\msc}^{\T}\rsmtx_{\geq\msc}\shmtx_{<\msc})^{-1}\shmtx_{<\msc}^{\T}\rsmtx_{\geq\msc}
\]
By Lemma \ref{lem:samplecov_spectral_norm}, we have
\begin{align*}
\P(\|\frac{1}{n}\shmtx_{<\msc}\shmtx_{<\msc}^{\T}\|>2) & \leq2\usdim{<\msc}\exp\big(-\tfrac{n}{8\usdim{<\msc}}\big)\\
 & \leq c\exp(-d/c),
\end{align*}
for some $c>0$ and we have already shown $\|(\mD_{<\msc}^{-2}+\frac{1}{n}\shmtx_{<\msc}^{\T}\rsmtx_{\geq\msc}\shmtx_{<\msc})^{-1}\|\stodom1$,
so we get
\begin{equation}
\|\mPsi\|\stodom\frac{1}{\sqrt{n}}.\label{eq:testerr_normsqterm_1_spectralnormbd_1}
\end{equation}
On the other hand, same as (\ref{eq:higher_order_norm_bd_Mpapaer_3}),
we can show $\|\sum_{k>\msc}\frac{\tcoeff k^{2}}{\usdim k}\shmtx_{k}\shmtx_{k}^{\T}-\sum_{k>\msc}\tcoeff k^{2}\mI\|\stodom\frac{1}{d^{\tau}}$,
for some $\tau>0$, so combined with $\|\frac{\tcoeff{\msc}^{2}}{\usdim{\msc}}\shmtx_{\msc}\shmtx_{\msc}^{\T}\|\stodom1$,
we can get
\begin{equation}
\|\sum_{k\geq\msc}\frac{\tcoeff k^{2}}{\usdim k}\shmtx_{k}\shmtx_{k}^{\T}\|\stodom1.\label{eq:testerr_normsqterm_1_spectralnormbd_2}
\end{equation}
Combining (\ref{eq:testerr_normsqterm_1_spectralnormbd_1}) and (\ref{eq:testerr_normsqterm_1_spectralnormbd_2})
with (\ref{eq:testerr_normsqterm_1_Ephi2normsq}), we get:
\[
\|\vphi_{2}\|^{2}\stodom\frac{\usdim{<\msc}}{n}\lesssim\frac{1}{d}.
\]
After combining (\ref{eq:testerr_normsqterm_1_vecnorm_1}), (\ref{eq:testerr_normsqterm_1_vecnorm_2})
and (\ref{eq:testerr_normsqterm_1_vecnorm_3}) with (\ref{eq:testerr_normsqterm_1_phi})
and (\ref{eq:testerr_normsqterm_1_approx_2}), we get (\ref{eq:testerr_normsqterm_1_approx_1_bd_2}).
Finally, the proof is completed by combing (\ref{eq:testerr_normsqterm_1_approx_1_bd_1})
with (\ref{eq:testerr_normsqterm_1_approx_1_bd_2}).
\end{proof}

\section{Concentration of a Quadratic Form}
\begin{lem}
\label{lem:Fkl_sq_bound}For any finite integer $k,\ell\geq0$, let
\[
F_{k,\ell}(\vx)=\mY_{k,\symset}(\vx)^{\T}\mM\mY_{\ell,\symset}(\vx),
\]
where $\mY_{k,\symset}$ is defined in (\ref{eq:SH_symmetric}), $\vx\sim\spmeasure{d-1}$
and $\mM\in\R^{m\times p}$ is a deterministic matrix, with $m={d \choose k}$
and $p={d \choose \ell}$. It holds that
\begin{equation}
\E_{\vx}F_{k,\ell}(\vx)^{2}\leq C_{k,\ell}\min\{N_{k},N_{\ell}\}\|\mM\|_{F}^{2},\label{eq:Fkl_sq_bound}
\end{equation}
where $C_{k,\ell}>0$ is some constant that only depends on $k$ and $\ell$.
\end{lem}

\begin{proof}
Note that since $\vx$ can be represented as: $\vx=\frac{\vtheta}{\|\vtheta\|/\sqrt{d}}$,
where $\vtheta\sim\mathcal{N}(\boldsymbol{0},\mI_{d})$ and by concentration
of norm of Gaussian vector [\emph{e.g.}, Theorem 3.1.1 in \cite{vershynin2018high}],
there exists $c>0$, such that for any $d\in\mathbb{Z}^{+},$$\P(\|\vtheta\|/\sqrt{d}<1/2)\leq e^{-cd}$.
Hence, for $\vx=\frac{\vtheta}{\|\vtheta\|/\sqrt{d}}$, $\vtheta\sim\mathcal{N}(\boldsymbol{0},\mI_{d})$,
we can get
\[
\mY_{k,\symset}(\vx)^{\T}\mM\mY_{\ell,\symset}(\vx)\lesssim\mY_{k,\symset}(\vtheta)^{\T}\mM\mY_{\ell,\symset}(\vtheta).
\]
Therefore, it suffices to prove (\ref{eq:Fkl_sq_bound}) for $\vx\sim\mathcal{N}(\boldsymbol{0},\mI_{d})$.
In the following, we set $\vx\sim\mathcal{N}(\boldsymbol{0},\mI_{d})$. We
will prove (\ref{eq:Fkl_sq_bound}) by induction. When $k=\ell=0$,
we trivially have $\E[Y_{0,\symset}(\vx)^{\T}M Y_{0,\symset}(\vx)]^{2}=M^{2}$.
Now suppose we have shown for any $\mM_{1}\in\R^{m_{1}\times p_{1}}$
and $\mM_{2}\in\R^{m_{2}\times p_{2}}$, with $m_{1}={d \choose k-1}$,
$p_{1}={d \choose \ell}$ and $m_{2}={d \choose k}$ and $p_{2}={d \choose \ell-1}$,
it holds that
\begin{align}
\E[F_{k-1,\ell}(\vx)]^{2} & \leq C_{k-1,\ell}\min\{N_{k-1},N_{\ell}\}\|\mM_{1}\|_{F}^{2},\label{eq:Fkl_sq_bound_start1}\\
\E[F_{k,\ell-1}(\vx)]^{2} & \leq C_{k,\ell-1}\min\{N_{k},N_{\ell-1}\}\|\mM_{2}\|_{F}^{2}.\label{eq:Fkl_sq_bound_start2}
\end{align}
Based on these two bounds, we are going to show for any $\mM\in\R^{m\times p}$,
with $m={d \choose k}$, $p={d \choose \ell}$, 
\begin{equation}
\E[F_{k,\ell}(\vx)]^{2}\leq C_{k,\ell}\min\{N_{k},N_{\ell}\}\|\mM\|_{F}^{2}.\label{eq:Fkl_sq_bound_end}
\end{equation}
In particular, if $k=0$ (or $\ell=0$), we just use (\ref{eq:Fkl_sq_bound_start2})
or {[}(\ref{eq:Fkl_sq_bound_start1}){]} for the induction.

We will bound $|\E F_{k,\ell}(\vx)|$ and $\text{Var}F_{k,\ell}(\vx)$,
separately. The expectation $\E F_{k,\ell}(\vx)$ is easy to deal
with. When $k\neq\ell$, $\E F_{k,\ell}(\vx)=0$; when $k=\ell$,
\begin{align*}
\E F_{k,k}(\vx) & =\E\text{Tr}[\mM\mY_{k,\symset}(\vx)\mY_{k,\symset}(\vx)^{\T}]\\
 & =\upsilon_{k}^{2}\text{Tr}\mM
\end{align*}
where $\upsilon_{k}$ is the normalizing constant defined in (\ref{eq:SH_symmetric}).
Therefore,
\begin{align}
|\E F_{k,k}(\vx)|^{2} & \leq\upsilon_{k}^{4}\big(\sum_{u=1}^{m}|M_{uu}|\big)^{2}\nonumber \\
 & \leq C_k N_{k}(\sum_{u=1}^{m}M_{uu}^{2})\nonumber \\
 & \leq C_k N_{k}\|\mM\|_{F}^{2}\label{eq:abs_meanEF_bd_2}
\end{align}
where $C_k$ is a constant that only depends on $k$. Next, we compute the variance of $F_{k,\ell}(\vx)$. Taking derivative
of $F_{k,\ell}(\vx)$ with respect to each $x_{i}$, we have
\[
\frac{\partial F_{k,\ell}(\vx)}{\partial x_{i}}=\frac{\upsilon_{k}}{\upsilon_{k-1}}\mY_{k-1,\symset}(\vx)^{\T}\text{\ensuremath{\mM}}_{i}^{\text{row}}\mY_{\ell,\symset}(\vx)+\frac{\upsilon_{\ell}}{\upsilon_{\ell-1}}\mY_{k,\symset}(\vx)^{\T}\text{\ensuremath{\mM}}_{i}^{\text{col}}\mY_{\ell-1,\symset}(\vx)
\]
where $\text{\ensuremath{\mM}}_{i}^{\text{row}}$ and$\text{\ensuremath{\mM}}_{i}^{\text{col}}$
are formed by concatenating a subset of rows (columns) from $\mM$
and some zero rows (columns). As a result,
\begin{align*}
\|\nabla F_{k,\ell}(\vx)\|^{2}
&\leq2\Big(\frac{\upsilon_{k}}{\upsilon_{k-1}}\Big)^{2}\sum_{i=1}^{d}[\mY_{k-1,\symset}(\vx)^{\T}\text{\ensuremath{\mM}}_{i}^{\text{row}}\mY_{\ell,\symset}(\vx)]^{2}\\
&\hspace{6em}+2\Big(\frac{\upsilon_{\ell}}{\upsilon_{\ell-1}}\Big)^{2}\sum_{i=1}^{d}[\mY_{k,\symset}(\vx)^{\T}\text{\ensuremath{\mM}}_{i}^{\text{col}}\mY_{\ell-1,\symset}(\vx)]^{2}.
\end{align*}
Then by (\ref{eq:Fkl_sq_bound_start1}) and (\ref{eq:Fkl_sq_bound_start2}),
we have
\begin{align*}
\E\|\nabla F_{k,\ell}(\vx)\|^{2} & \leq C_{k,\ell}\Big[\min\{N_{k-1},N_{\ell}\}\sum_{i=1}^{d}\|\text{\ensuremath{\mM}}_{i}^{\text{row}}\|_{F}^{2}+\min\{N_{k},N_{\ell-1}\}\sum_{i=1}^{d}\|\text{\ensuremath{\mM}}_{i}^{\text{col}}\|_{F}^{2}\Big]\\
 & \leq C_{k,\ell}\Big[k\min\{N_{k-1},N_{\ell}\}\|\mM\|_{F}^{2}+\ell\min\{N_{k},N_{\ell-1}\}\|\mM\|_{F}^{2}\Big]\\
 & \leq C_{k,\ell} \min\{N_{k},N_{\ell}\}\|\mM\|_{F}^{2}.
\end{align*}
Using Gaussian Poincar\'{e} inequality, we can get
\begin{align}
\text{Var}F_{k,\ell}(\vx) & \leq\E\|\nabla F_{k,\ell}(\vx)\|^{2}\nonumber \\
 & \leq C_{k,\ell}\min\{N_{k},N_{\ell}\}\|\mM\|_{F}^{2}.\label{eq:abs_varEF_bd_2}
\end{align}
Finally, combining (\ref{eq:abs_varEF_bd_2}) and (\ref{eq:abs_meanEF_bd_2}),
we get
\begin{align*}
\E[F_{k,\ell}(\vx)^{2}] & =[\E F_{k,\ell}(\vx)]^{2}+\text{Var}F_{k,\ell}(\vx)\\
 & \leq C_{k,\ell} \big(\delta_{k\ell}N_{k}\|\mM\|_{F}^{2}+\min\{N_{k},N_{\ell}\}\|\mM\|_{F}^{2}\big)\\
 & \leq2 C_{k,\ell}\min\{N_{k},N_{\ell}\}\|\mM\|_{F}^{2}
\end{align*}
which is (\ref{eq:Fkl_sq_bound_end}).
\end{proof}


\end{document}